\documentclass[11pt,a4paper]{article}

\usepackage[margin=2.55cm]{geometry}
\usepackage{amsmath,amssymb,amsthm,mathtools,mathrsfs}
\usepackage{enumitem}
\usepackage{xcolor}
\usepackage{graphicx}
\usepackage{subcaption}
\usepackage[colorlinks=true,linkcolor=blue,citecolor=blue,urlcolor=blue]{hyperref}
\setlist{leftmargin=2em}

\newtheorem{theorem}{Theorem}[section]
\newtheorem{proposition}[theorem]{Proposition}
\newtheorem{corollary}[theorem]{Corollary}
\newtheorem{lemma}[theorem]{Lemma}
\newtheorem{assumption}[theorem]{Assumption}
\theoremstyle{definition}
\newtheorem{definition}[theorem]{Definition}
\theoremstyle{remark}
\newtheorem{remark}[theorem]{Remark}

\newcommand{\R}{\mathbb{R}}
\newcommand{\Id}{\mathrm{I}}
\newcommand{\norm}[1]{\left\lVert #1\right\rVert}
\newcommand{\abs}[1]{\left\lvert #1\right\rvert}
\newcommand{\cL}{\mathcal{L}}
\newcommand{\cK}{\mathcal{K}}
\newcommand{\cD}{\mathcal{D}}
\newcommand{\cH}{\mathcal{H}}
\newcommand{\Var}{\operatorname{Var}}
\newcommand{\PiP}{\Pi_p}
\newcommand{\ip}[2]{\left\langle #1,#2\right\rangle}
\newcommand{\dd}{\mathrm{d}}

\title{Improved Analysis for Hessian-free High-resolution Monte Carlo Sampling}
\author{
Wujun Lv\thanks{School of Mathematics and Statistics, Donghua University, Shanghai, People's Republic of China; \texttt{lvwujun@dhu.edu.cn}}
\and
Xiaoyu Wang\thanks{Hong Kong University of Science and Technology (Guangzhou), Guangzhou, Guangdong Province, People's Republic of China;
\texttt{xiaoyuwang@hkust-gz.edu.cn}}
\and
Yingli Wang\thanks{School of Mathematical Sciences, Fudan University, Shanghai, People's Republic of China; \texttt{yingliwang@fudan.edu.cn}}
\and
Lingjiong Zhu\thanks{Department of Mathematics, Florida State University, Tallahassee, Florida, United States of America; \texttt{zhu@math.fsu.edu}}}
\date{\today}

\begin{document}
\maketitle

\begin{abstract}
Hessian-free high-resolution (HFHR) dynamics augments underdamped Langevin dynamics (ULD) with reversible
position diffusion for sampling problems that arise in machine learning.
Under a position Poincar\'e inequality, weighted bounds on the Hessian
and Laplacian, and a compact Sobolev embedding, we give an explicit
decay rate for potentials that may be nonconvex and have
unbounded Hessian. A time-augmented Poincar\'e inequality and direct
control of position dissipation give a rate bound that adds the
position-diffusion contribution to the ULD rate up to friction saturation.
For convex potentials the bound depends only on the position
Poincar\'e constant, the friction and the position-diffusion coefficient.
We also give a weak-solution construction
and a self-contained spectral proof of the divergence lemma underlying
the argument. For the frozen-force HFHR Monte Carlo scheme with globally
Lipschitz gradient, a path-space
Girsanov argument gives non-asymptotic convergence and explicit
iteration complexity in total variation distance.
The bounds hold for every set of model parameters and remain
regular at the ULD endpoint. At sufficiently high accuracy, the
position-diffusion coefficient minimizing our iteration bound is
positive and tends to zero with the tolerance. The optimized bound
has the same leading order as at the optimized ULD endpoint.
Under common strong-convexity and warm-start conditions, the optimized TV
complexity bound improves the accuracy and dimension dependence of
the existing bound for the same discretization.
Gaussian calculations and Bayesian regression and neural-network
experiments illustrate the effect of positive position diffusion.
\end{abstract}

\paragraph{Keywords.}
Hessian-free high-resolution; Monte Carlo sampling; Fokker--Planck equation; Poincar\'e inequality; nonconvex potential; explicit
convergence rate.

\section{Introduction}

Consider a given probability measure on the Euclidean space $\R^d$:
\begin{align}\label{target:dist}
  \mu(\dd q):=Z_U^{-1}e^{-U(q)}\,\dd q,
\end{align}
where $U:\mathbb{R}^{d}\rightarrow\mathbb{R}$ is a potential function
and $Z_{U}>0$ is a normalizing constant.  
The problems of sampling a given target distribution in \eqref{target:dist}
arise routinely in Bayesian statistics, inverse problems and modern-day large-scale machine
learning problems \cite{gelman1995bayesian,stuart2010inverse,andrieu2003introduction,teh2016consistency,DistMCMC19,GIWZ2024,DIGing2025}.

Langevin algorithms are popular Markov chain Monte Carlo algorithms 
that are designed to solve the sampling problem \eqref{target:dist}. 
The most classical Langevin algorithm is based on the discretization
of the \textit{overdamped Langevin dynamics} (OLD):
\begin{equation}\label{eq:overdamped-2}
  \dd q_t = -\nabla U(q_t)\,\dd t + \sqrt{2}\,\dd W_t,
\end{equation}
where $(W_{t})_{t\geq 0}$ is a standard $d$-dimensional Brownian motion.
The diffusion \eqref{eq:overdamped-2} has invariant distribution
\eqref{target:dist} under mild conditions; see e.g.\
\cite{chiang1987diffusion,stroock-langevin-spectrum}.
There have been extensive studies on the non-asymptotic convergence theory 
for \eqref{eq:overdamped-2} and its various discretization schemes
\cite{Dalalyan,DM2017,DM2016,DK2017,Raginsky,Barkhagen2021,Chau2019,Zhang2019,CB2018,EHZ2022}.

To accelerate convergence, one can introduce a momentum variable and
consider the \textit{underdamped Langevin dynamics} (ULD) (also known as kinetic or second-order Langevin dynamics)
\cite{mattingly2002ergodicity,Villani2009,cheng2018underdamped,cheng-nonconvex,CLW2020,cao-lu-wang-2023,dalalyan-riou-durand-2020,GGZ2,Ma2019,GGZ}:
\begin{equation}\label{eqn:underdamped}
\begin{cases}
  \dd p_t = -\gamma p_t\,\dd t - \nabla U(q_t)\,\dd t + \sqrt{2\gamma}\,\dd W_t,
  \\
  \dd q_t = p_t\,\dd t,
\end{cases}
\end{equation}
where $(W_t)_{t\ge0}$ is a standard $d$--dimensional Brownian motion and
$\gamma>0$ is the friction parameter.
Under mild assumptions, \eqref{eqn:underdamped} admits a unique
invariant measure with density $\propto e^{-U(q)-\frac12|p|^2}$, whose
$q$--marginal coincides with $\mu$ as in \eqref{target:dist}.
It is by now well--understood that underdamped Langevin dynamics
\eqref{eqn:underdamped} can converge faster than its overdamped
counterpart, and related gains are available for certain discretizations,
target classes, and initialization regimes
\cite{eberle-guillin-zimmer-2019,cao-lu-wang-2023,cheng2018underdamped,GGZ}.

There is a close connection between the underdamped Langevin dynamics 
and \textit{Nesterov's accelerated gradient} (NAG) method in optimization literature 
\cite{Nesterov1983,Nesterov2013,Ma2019,GGZ}. Motivated by the high-resolution
ordinary differential equation (ODE) viewpoint on NAG, in a seminal paper, \cite{li-zha-tao-2022} proposed the \emph{Hessian-free
high-resolution} (HFHR) dynamics that satisfies the following stochastic differential equation (SDE):
\begin{equation}\label{eq:HFHR-SDE}
\begin{cases}
  \dd q_t=(p_t-\alpha\nabla U(q_t))\,\dd t
  +\sqrt{2\alpha}\,\dd W_t^q,\\
  \dd p_t=(-\nabla U(q_t)-\gamma p_t)\,\dd t
  +\sqrt{2\gamma}\,\dd W_t^p,
\end{cases}
\qquad \alpha\geq0,\quad\gamma>0,
\end{equation}
where $(W_t^q)_{t\geq 0}$ and $(W_t^p)_{t\geq 0}$ are independent standard
$d$-dimensional Brownian motions.  The case $\alpha=0$ is
underdamped Langevin dynamics (ULD) \eqref{eqn:underdamped}.  
The invariant law of \eqref{eq:HFHR-SDE} is given by
\cite[Theorem~4.1]{li-zha-tao-2022}:
\begin{equation}\label{eq:pi}
  \pi(\dd q\,\dd p):=\mu(\dd q)\kappa(\dd p),\qquad
  \kappa(\dd p):=(2\pi)^{-d/2}e^{-\abs p^2/2}\,\dd p.
\end{equation}
The drift in the SDE \eqref{eq:HFHR-SDE} only depends on $\nabla U$ and
is thus ``Hessian-free'', in contrast to other high-resolution
ODEs for NAG that contain $\nabla^2 U$; see e.g.\ \cite{Shi2022}.
The connection to NAG method in optimization has been further studied 
to design gradient-adjusted underdamped Langevin (GAUL) dynamics for accelerated sampling that includes HFHR dynamics as a special case \cite{zuo2025gradient}.
Recently, HFHR and GAUL dynamics have been shown to reduce
the asymptotic variance in Monte Carlo computations compared to the more classical OLD \cite{Ni2026}.

Under the $m$-strong convexity and $L$-smoothness assumptions:
$m\Id_{d}\preceq\nabla^2U\preceq L\Id_{d}$,
the HFHR analysis in \cite[Theorem~5.1 and Lemma~D.5]{li-zha-tao-2022} proves weighted synchronous-coupling
contraction and, in the parameter regime
$\gamma^2>L+m$ and
$\alpha\leq\frac{\gamma^2-L-m}{m\gamma}$
obtains the position-marginal $\mathcal W_2$ contraction rate 
$\frac{m}{\gamma}+\alpha m$.

A complementary continuous-time result was recently obtained in
\cite[Theorem~4.2]{cortild-delplancke-oudjane-peypouquet-2025}.  For an equivalent
fully diffusive parametrization of HFHR, they prove exponential decay
of relative entropy under a logarithmic Sobolev inequality.

Coupling method is a popular approach to explicit contraction in
transport metrics in the literature.  For overdamped Langevin dynamics with locally nonconvex
potentials, \cite{eberle-2016} develops reflection coupling with
concave transport costs.  Eberle--Guillin--Zimmer
\cite{eberle-guillin-zimmer-2019} combine reflection and synchronous
coupling with a Lyapunov weight for underdamped Langevin dynamics; see also
\cite{dalalyan-riou-durand-2020} for underdamped Langevin sampling in the
log-concave setting.  Building on this kinetic coupling
framework, Wang, Wang and Zhu \cite{wang-wang-zhu-2026} study the
same HFHR dynamics under global smoothness and dissipativity assumptions and prove
exponential contraction in a Lyapunov-weighted Wasserstein distance
for all sufficiently small $\alpha>0$.  Under an additional
asymptotically linear-gradient condition, they establish a
strict linear-in-$\alpha$ improvement over ULD and transfer the
result to $\mathcal W_2$. 
In this paper, we adopt a position-Poincar\'e framework and obtain
$L^2$ decay estimates for arbitrary $\alpha\geq0$ and $\gamma>0$.

Quantitative convergence for ULD is a central example of
hypocoercivity \cite{Villani2009}.  In $L^2$, one major approach
modifies the Hilbert norm to couple the macroscopic and microscopic
modes \cite{dolbeault-mouhot-schmeiser-2015,roussel-stoltz-2018}.
Other constructive approaches use Schur complements or
Poincar\'e--Lions inequalities and time averaging
\cite{bernard-fathi-levitt-stoltz-2022,brigati-stoltz-2025}, while
variational space--time Poincar\'e inequalities provide the route
closest to ours
\cite{albritton-armstrong-mourrat-novack-2024,cao-lu-wang-2023}.

Eberle and L\"orler study the same HFHR diffusion \eqref{eq:HFHR-SDE}
in \cite[Section~3.3]{eberle-loerler-2026-relaxation}. Their
Theorem~27 bounds the $L^2$ relaxation time of the transition semigroup
via the singular value gap of the generator of two independent copies.
For $\alpha>0$ their bound uses global lower and upper Hessian bounds.
Theorem~\ref{thm:main} below allows the Hessian to grow with
$|\nabla U|$ and, for convex potentials, gives a rate independent of
any Hessian upper bound. Remark~\ref{rem:recent-HFHR-comparison} compares the
two estimates on the same family of Gaussian targets. At the ULD
endpoint, Eberle and L\"orler's earlier work gives a square-root limit
on the speed-up of second-order lifts \cite{eberle-loerler-2026-lifts}
and a space--time divergence lemma for reflected Langevin dynamics and
randomized HMC \cite{eberle-loerler-2026-divergence}.

Our proof extends the ULD argument of Cao, Lu and Wang
\cite[Theorems~1 and~2]{cao-lu-wang-2023}, who combine a
time-augmented Poincar\'e inequality with the $L^2$ energy identity.
For HFHR, integration by parts expresses the extra position-diffusion
term in terms of the position gradients of the solution and of their
divergence test.
Two recent complementary ULD results further situate this endpoint.
A gap-shifted modified-$L^2$ approach
\cite[Theorem~1 and Corollary~2]{fan-li-lu-2026} recovers a
rate of order $\sqrt m$ for convex potentials, and more generally an
explicit rate under a global Hessian lower bound.  In the convex and
semiconvex regimes, the optimized estimate agrees up to universal
constants with the Cao--Lu--Wang rate, while requiring fewer technical
regularity assumptions; it does not, however, provide the full
arbitrary-friction and general-potential estimate used here.  In relative entropy, a recent
result \cite[Theorem~2.3]{lu-2026-entropy} obtains a sharp rate of order $\sqrt\rho$
for convex potentials whose position marginal satisfies a logarithmic
Sobolev inequality with constant $\rho$, using a nonlinear Wasserstein
entropy-current corrector.  Neither result addresses the HFHR position
diffusion $\alpha>0$ considered here.

Closer to the present setting,
\cite[Theorem~4.2, Algorithm~2, and Corollary~4.4]{cortild-delplancke-oudjane-peypouquet-2025} studies a more
general parametrization of the same fully diffusive high-resolution
Langevin system, together with its discretization and an application
to global optimization.  Under a logarithmic Sobolev inequality for
the joint Gibbs law, they obtain relative-entropy decay at rate
$2\rho\min\{\alpha,\gamma\}$ in our normalization.  This direct
elliptic estimate requires $\alpha>0$ and degenerates as
$\alpha \rightarrow 0$.  Our result instead works in a
position-Poincar\'e framework, supplemented by weighted derivative and
compactness assumptions, retains the hypocoercive contribution, and
covers the ULD endpoint $\alpha=0$.

Discrete-time ULD convergence bounds under a Poincar\'e inequality
require additional care because continuous-time $L^2$ contraction
does not by itself control the discretization error in relative
entropy.  For ULD, \cite[Lemma~8 and Theorem~9]{zhang-chewi-li-balasubramanian-erdogdu-2023}
combines a R\'enyi Girsanov estimate with continuous-time
$\chi^2$-decay, while \cite[Theorem~1]{lehec-2025} obtains a total-variation
bound by controlling the kinetic moments through relative entropy and
the exact diffusion.  We adapt the latter self-bounding mechanism to
the two-noise HFHR interpolation.  The resulting entropy--moment
closure is explicit, remains regular at $\alpha=0$, and avoids a
separate time-uniform moment-stability assumption.

For $K$-strongly convex and
$L$-smooth potentials, with condition number $\kappa=L/K$, Bou-Rabee
\cite{bou-rabee-2026-nonacceleration} proves that none of the six standard
Strang splittings admits a uniform cold-start total-variation mixing bound
on any scale $o(\kappa)$ when the step size and friction are chosen only
from $(K,L,d)$.  A Gaussian calculation gives a nonaccelerated obstruction
of the same order for the left-endpoint exponential integrator
\cite[Appendix~A]{bou-rabee-2026-nonacceleration}.  The $\alpha=0$
specialization of the HFHRMC scheme introduced below is exactly this
left-endpoint integrator.  These lower bounds concern deterministic cold
starts and fixed, target-independent ULD tuning. The bounds in
Section~\ref{sec:discrete} apply to $L^2(\pi)$ initial densities and
allow accuracy-dependent tuning.

In this paper, for the discretization scheme of HFHR dynamics \eqref{eq:HFHR-SDE}, we consider the \textit{Hessian-free high-resolution Monte Carlo} (HFHRMC) algorithm:
\begin{equation}\label{eqn:HFHRMC}
\begin{split}
  P_{k+1}&=aP_k-b\nabla U(Q_{k})+\eta_k^p,\\
  Q_{k+1}&=Q_k+bP_k-(c+\alpha h)\nabla U(Q_{k})+\eta_k^q,
\end{split}
\end{equation}
where $a,b,c$ are specified in \eqref{eqn:abc}
and $(\eta_{k}^{p},\eta_{k}^{q})$ are i.i.d. centered Gaussian random vectors
with covariance structure given in \eqref{eq:noise-pp}, \eqref{eq:noise-qp} and \eqref{eq:noise-qq}.
This is exactly the discretization scheme given in
\cite[Algorithm~2]{cortild-delplancke-oudjane-peypouquet-2025}.
We obtain non-asymptotic convergence bounds and iteration-complexity
estimates for the HFHRMC algorithm \eqref{eqn:HFHRMC} for every
$\alpha\geq0$ and $\gamma>0$.  This full-range guarantee allows
parameter optimization without imposing a priori parameter
restrictions.

\paragraph{Contributions.}

\begin{itemize}
\item
Under Assumptions~\ref{ass:PI}--\ref{ass:compact}, Theorem~\ref{thm:main}
gives the continuous-time $L^2$ decay rate
\[
  \nu_{\alpha,\gamma}
  \geq
  c\min\left\{
  \gamma,\,
  \alpha m+
  \frac{m\gamma}{(\sqrt m+R+\gamma)^2}
  \right\},
\]
where $m$ is the position Poincar\'e constant and $R$ is given in
\eqref{eq:R}. For $\alpha>0$, the kinetic contribution persists even
when the Hessian is unbounded above. For convex $U$, $R=0$, so the
bound depends only on $m,\alpha,\gamma$. On the Gaussian family of
Remark~\ref{rem:recent-HFHR-comparison}(i), where $m=\gamma=1$,
$\alpha=1/n$ and $L=n^3$ is the largest Hessian eigenvalue, we obtain
$\nu_{1/n,1}\geq c/4$ for every $n\geq1$. With the same global
parameters, the better of the two bounds of Eberle and L\"orler
\cite[equation~(27) and Theorem~27(ii)]{eberle-loerler-2026-relaxation}
gives a rate lower bound of order $1/n$.
At $\alpha=0$ we recover the Cao--Lu--Wang rate
\cite[Theorem~1]{cao-lu-wang-2023}. For convex $U$ and
$\gamma\gg\sqrt m$, the second term in the minimum is of order $\alpha m+m/\gamma$
\cite[Theorem~5.1]{li-zha-tao-2022}.
\item 
We develop the analytic framework needed to prove the rate.  This
includes conservative Markov semigroup constructions for both
$\alpha>0$ and the kinetic endpoint $\alpha=0$, an HFHR-adapted
space--time Poincar\'e inequality, and a variational decay estimate (Proposition~\ref{prop:weak-semigroup},
\ref{prop:kinetic-endpoint}, Proposition~\ref{prop:space-time}, and
Proposition~\ref{prop:variational-rate}).
The derivative bound \eqref{eq:Phi-1} on the divergence test
controls the extra term in \eqref{eq:macro-duality} for $\alpha>0$.
Combining the resulting variational estimate with the product
Poincar\'e inequality gives the rate in Theorem~\ref{thm:main}.
\item
For the HFHRMC algorithm, which is based on a discretization scheme of the continuous-time HFHR dynamics, we prove a
path-space Kullback-Leibler (KL) error bound whose Girsanov factor is
$\alpha+\gamma^{-1}$, and hence remains finite at the ULD endpoint
(Theorem~\ref{thm:girsanov-discretization}).
An entropy--moment argument leads to the state-dependent local error, without assuming a
time-uniform moment bound for the numerical chain.  This yields an
explicit non-asymptotic convergence bound in TV distance (Corollary~\ref{cor:discrete-TV})
and iteration complexity
(Corollary~\ref{cor:iteration-complexity}).  We then
optimize this iteration complexity over $\alpha,\gamma$ and obtain an accuracy- and
dimension-dependent $\alpha_\varepsilon^\star$, as well as the
high-accuracy friction choice
$\gamma_\varepsilon^\star\to2(\sqrt m+R)$
(Proposition~\ref{prop:optimal-alpha}).  The constrained optimizer
arising from the strongly convex $\mathcal W_2$ complexity bound of
\cite[Corollary~5.4 and Remark~5.5]{li-zha-tao-2022} is independent of
the requested accuracy, whereas our TV optimizer tends to zero with
the requested error.  The leading high-accuracy order of our bound
agrees with that of the optimized ULD endpoint.  Under common
strong-convexity and warm-start conditions, comparison with
\cite[Corollary~4.4]{cortild-delplancke-oudjane-peypouquet-2025}
gives improved accuracy and dimension dependence for the TV
complexity bound of the same discretization
(Remark~\ref{rem:cortild-complexity-comparison}).
\item
Section~\ref{sec:numerical} compares positive position diffusion with
the ULD endpoint in four experiments. An anisotropic Gaussian target
and a Bayesian linear regression posterior allow exact Gaussian
calculations. For Bayesian logistic regression and a neural network
with bounded transformed weights and biases, we compare the iterates
with numerical posterior references.
\end{itemize}

\paragraph{Organization.}
The rest of the paper is organized as follows. In Section~\ref{sec:setup}, we present
the preliminaries and assumptions for the model.
Section~\ref{sec:main} introduces the main results for the continuous-time analysis, 
including an explicit $L^2$ decay rate for HFHR dynamics (Section~\ref{sec:rate})
and a functional framework and the space-time estimate (Section~\ref{sec:functional}).
The proofs of the main results will be provided in Section~\ref{sec:energy}.
Section~\ref{sec:discrete} develops the discrete-time HFHRMC algorithm, its
non-asymptotic convergence bounds, iteration complexity, and parameter tuning.
Finally, we conclude in Section~\ref{sec:conclude}. 
The additional technical results, 
such as weighted elliptic estimates
and spectral construction of the divergence test, 
and weighted kinetic density, traces and Green's formula
will be provided in Appendix~\ref{app:elliptic}, \ref{app:divergence} and \ref{app:kinetic-trace} respectively.

\section{Preliminaries and Assumptions}
\label{sec:setup}

\paragraph{Notation.}
The probability measures $\mu$, $\kappa$, and
$\pi=\mu\otimes\kappa$ are those introduced in \eqref{target:dist} and
\eqref{eq:pi}.  For a probability measure $\nu$ on the Euclidean space $\mathbb{R}^{n}$,
we write
\[
  \nu(f):=\int_{\mathbb{R}^{n}} f\,\dd\nu,\qquad
  (f)_\nu:=\nu(f),\qquad
  L_0^2(\nu):=\{f\in L^2(\nu):\nu(f)=0\}.
\]
We set
$\Var_\nu(f):=\nu(|f-\nu(f)|^2)$.
For the interval $I=(0,T)$, the normalized time measure and the Gaussian velocity
projection are
\[
  \lambda_T(\dd t):=T^{-1}\mathbf 1_{(0,T)}(t)\,\dd t,
  \qquad
  \PiP f(t,q):=\int_{\R^d}f(t,q,p)\,\dd\kappa(p).
\]
The symbol $H^k(\nu)$ denotes the usual weighted Sobolev space.  We
write
\[
  H_\kappa^1
  :=\{g\in L^2(\kappa):\nabla_pg\in L^2(\kappa)\},
  \qquad H_\kappa^{-1}:=(H_\kappa^1)^*,
\]
with $L^2(\kappa)$ as pivot space.  Bochner spaces in the
$(t,q)$-variables are taken against $\lambda_T\otimes\mu$, and
$H_0^1(\lambda_T\otimes\mu)$ means zero traces only at
$t=0$ and $t=T$.  We also use
\[
  \nabla_{t,q}=\widetilde\nabla:=(\partial_t,\nabla_q),
  \qquad D_{t,q}^2:=D^2_{(t,q)}.
\]
Unlabelled norms and pairings are taken with respect to the measure
specified by the variables appearing in the corresponding formula.
The relations $a\lesssim b$
and $a\asymp b$ mean, respectively, one-sided and two-sided bounds
up to universal positive constants.

The adjoints of the gradients in $L^2(\pi)$ are
\[
  \nabla_q^*F=-\nabla_q\cdot F+\nabla U\cdot F,
  \qquad
  \nabla_p^*F=-\nabla_p\cdot F+p\cdot F.
\]
Define
\begin{equation}\label{eq:HFHR-generator}
  \cL_{\alpha,\gamma}
  :=
  \cL_{\mathrm{ham}}
  -\alpha\nabla_q^*\nabla_q
  -\gamma\nabla_p^*\nabla_p,
\end{equation}
where
\begin{equation}\label{eq:operators}
  \cL_{\mathrm{ham}}
  :=\nabla_p^*\nabla_q-\nabla_q^*\nabla_p
  =p\cdot\nabla_q-\nabla U\cdot\nabla_p.
\end{equation}
The Hamiltonian operator is antisymmetric in $L^2(\pi)$, while the
last two operators are self-adjoint and nonpositive.
We write $A_q=\nabla_q^*\nabla_q$ for the nonnegative Friedrichs
realization in $L^2(\mu)$, and
\[
  \cK_\alpha:=\partial_t-\cL_{\mathrm{ham}}+\alpha A_q,
  \qquad
  \mathscr A:=-\partial_{tt}+A_q
\]
for the kinetic graph operator and the time--position operator,
respectively; $\mathscr A$ carries Neumann conditions in time.

For $\alpha\geq0$ and $\gamma>0$, we denote by
$(S_t^{\alpha,\gamma})_{t\geq0}$ the $L^2(\pi)$ semigroup
generated by $\cL_{\alpha,\gamma}$.  Its construction for
$\alpha>0$ and $\alpha=0$ is given in Propositions
\ref{prop:weak-semigroup} and \ref{prop:kinetic-endpoint},
respectively.

We impose the following assumptions, with the same normalization as
\cite{cao-lu-wang-2023}.

\begin{assumption}[Position Poincar\'e inequality]\label{ass:PI}
There is $m>0$ such that
\begin{equation}\label{eq:position-PI}
  \Var_\mu(g)
  \leq\frac1m\int_{\R^d}\abs{\nabla_qg}^2\,\dd\mu
  \qquad\text{for every }g\in H^1(\mu).
\end{equation}
\end{assumption}

\begin{assumption}[Weighted derivative bounds]\label{ass:growth}
The potential $U\in C^2(\R^d)$, and there exist $M\geq1$ and
$\delta\in(0,1)$ such that, for every $q\in\R^d$,
\begin{align}
  \abs{\nabla^2U(q)}_{\mathrm F}^2
  &\leq M^2\left(d+\abs{\nabla U(q)}^2\right),
  \label{eq:Hessian-growth}\\
  \Delta U(q)
  &\leq Md+\frac{\delta}{2}\abs{\nabla U(q)}^2.
  \label{eq:Laplacian-growth}
\end{align}
\end{assumption}

\begin{assumption}[Compact embedding]\label{ass:compact}
The embedding $H^1(\mu)\hookrightarrow L^2(\mu)$ is compact.
\end{assumption}

\begin{remark}[Sufficient conditions for the position Poincar\'e inequality]
\label{rem:PI-conditions}
If $\nabla^2U\succeq m_0\Id_d$ for some $m_0>0$, the
Bakry--\'Emery criterion gives Assumption~\ref{ass:PI} with
$m\geq m_0$ \cite[Theorem~3.1]{menz-schlichting-2014}.  More generally,
every integrable log-concave measure
on $\R^d$ satisfies a Poincar\'e inequality, although the constant
need not be explicit
\cite[Corollary~1.9]{bakry-barthe-cattiaux-guillin-2008}.  If $U=V+W$, where
$\nabla^2V\succeq m_0\Id_d$ and
$\operatorname{osc}(W)<\infty$, the bounded perturbation principle
gives $m\geq m_0e^{-\operatorname{osc}(W)}$
\cite[Theorem~3.2]{menz-schlichting-2014}.  More precisely, put
$\cL_{\mathrm{OLD}}=\Delta_q-\nabla U\cdot\nabla_q$.  If there are
$\theta>0$, $b\geq0$, $R>0$, and a $C^2$ function
$\mathcal W\geq1$ such that
\[
  \cL_{\mathrm{OLD}}\mathcal W
  \leq-\theta\mathcal W+b\mathbf 1_{B(0,R)},
\]
then \cite[Theorem~1.4]{bakry-barthe-cattiaux-guillin-2008} yields a
Poincar\'e inequality, including for many nonconvex confining
potentials.
\end{remark}

\begin{remark}[Comparison with global smoothness]
\label{rem:growth-smoothness}
If $\nabla U$ is globally $L$-Lipschitz, then
$\abs{\nabla^2U}_{\mathrm F}^2\leq dL^2$ and
$\Delta U\leq dL$.  Hence Assumption~\ref{ass:growth} holds with
$M=\max\{1,L\}$ and any fixed $\delta\in(0,1)$.  The converse is
false. For example, $U(x)=x^2/2+x^4/4$ on $\R$ has
$U''(x)=1+3x^2$ and $U'(x)=x+x^3$. Since
$|U''(x)|^2/(1+|U'(x)|^2)\to0$ and
$U''(x)-|U'(x)|^2/4\to-\infty$ as $|x|\to\infty$,
\eqref{eq:Hessian-growth}--\eqref{eq:Laplacian-growth} hold with
$\delta=1/2$ and some finite $M$. The lower bound $U''\geq1$ gives
the position Poincar\'e inequality, and $|U'(x)|\to\infty$ gives
compactness by Remark~\ref{rem:compact}.
\end{remark}

\begin{remark}[Role of compactness]\label{rem:compact}
The compact embedding provides the discrete spectrum of $A_q$ used in
Appendix~\ref{app:divergence}. Under Assumption~\ref{ass:growth}, a sufficient condition for
compactness is $\abs{\nabla U(q)}\to\infty$ as $\abs q\to\infty$. Indeed,
applying Lemma~\ref{lem:weighted-multiplier} to a time-independent $g\in H^1(\mu)$ gives,
for all sufficiently large $r$,
\[
  \int_{\{\abs q>r\}}\abs g^2\,\dd\mu
  \leq
  \frac{16\norm{\nabla_qg}_{L^2(\mu)}^2+4Md\norm g_{L^2(\mu)}^2}
       {\inf_{\abs q>r}\abs{\nabla U(q)}^2}.
\]
The right-hand side tends to zero uniformly on the unit ball of $H^1(\mu)$.
On $Q_r=(-r,r)^d$, let $P_\ell g$ be the Lebesgue averages of $g$ on the
partition into cubes of side $\ell=2r/n$, $n\geq1$, and extend $P_\ell g$
by zero outside $Q_r$. Tensorizing the interval Poincar\'e inequality on
each cell and comparing the Gibbs density to Lebesgue measure on $Q_r$
gives $\norm{g-P_\ell g}_{L^2(Q_r,\mu)}^2\leq C_r\ell^2
\norm{\nabla_qg}_{L^2(Q_r,\mu)}^2$. Choosing $r$ large and then $\ell$
small gives uniform finite-rank approximation of the $H^1(\mu)$ unit
ball in $L^2(\mu)$, proving Assumption~\ref{ass:compact}.
\end{remark}

As in \cite{cao-lu-wang-2023}, we introduce the following quantity that measures the curvature contribution:
\begin{equation}\label{eq:R}
  R:=
  \begin{cases}
    0,
    & U\text{ is convex},\\[1mm]
    \sqrt K,
    & \nabla^2U(q)\succeq-K\Id_{d}\text{ for every }q,\\[1mm]
    M+M^{3/4}d^{1/4},
    & \text{without an additional Hessian lower bound}.
  \end{cases}
\end{equation}
In the second line $K\geq0$.  The first line is the case $K=0$.
For $C^2$ potentials, the condition
$\nabla^2U\succeq-K\Id_d$ is equivalent to convexity of
$q\mapsto U(q)+K\abs q^2/2$.  Such a function is commonly called
$K$-weakly convex in optimization and $K$-semiconvex in analysis;
conventions concerning the factor $1/2$ may vary
\cite{davis-drusvyatskiy-2019}.

\section{Continuous-Time Analysis}\label{sec:main}

\subsection{Explicit \texorpdfstring{$L^2$}{L2} decay rate for HFHR dynamics}\label{sec:rate}

For positive position diffusion, the product Poincar\'e inequality
alone gives the rate $\min\{\alpha m,\gamma\}$. The next theorem also
accounts for the transfer of momentum dissipation to the position
variable by Hamiltonian transport.

\begin{theorem}[Explicit $L^2$ decay for HFHR]\label{thm:main}
Suppose Assumptions~\ref{ass:PI}--\ref{ass:compact} hold.  Let
$\alpha\geq0$ and $\gamma>0$.  Recall that
$(S_t^{\alpha,\gamma})_{t\geq0}$ is the $L^2(\pi)$ semigroup
generated by $\cL_{\alpha,\gamma}$, introduced in
Section~\ref{sec:setup} and constructed in Propositions
\ref{prop:weak-semigroup} and \ref{prop:kinetic-endpoint}.  For
$f_0\in L^2(\pi)$ with $\pi(f_0)=0$, set
$f_t=S_t^{\alpha,\gamma}f_0$; equivalently,
\[
  \partial_tf_t=\cL_{\alpha,\gamma}f_t,
  \qquad \pi(f_0)=0.
\]
There are universal constants $c,C>0$, independent of
$m,M,d,R,\alpha,\gamma$, such that
\begin{equation}\label{eq:main-decay}
  \norm{f_t}_{L^2(\pi)}
  \leq
  C e^{-\nu_{\alpha,\gamma}t}
  \norm{f_0}_{L^2(\pi)},
\end{equation}
where
\begin{equation}\label{eq:main-rate}
  \nu_{\alpha,\gamma}
  \geq
  c\min\left\{
  \gamma,\,
  \alpha m+
  \frac{m\gamma}{(\sqrt m+R+\gamma)^2}
  \right\}.
\end{equation}
\end{theorem}

The proof in Section~\ref{sec:energy} combines the product estimate
\eqref{eq:direct-rate} with the space--time estimate
\eqref{eq:variational-rate}. The latter controls the decay as
$\alpha$ tends to zero.

\begin{remark}[Convex, weakly convex, and general potentials]
Theorem~\ref{thm:main} gives the following three explicit forms:
\begin{align}
  \nu_{\alpha,\gamma}
  &\geq
  c\min\left\{\gamma,\alpha m+
  \frac{m\gamma}{(\sqrt m+\gamma)^2}\right\},
  && U\text{ convex},
  \label{eq:convex-rate}\\
  \nu_{\alpha,\gamma}
  &\geq
  c\min\left\{\gamma,\alpha m+
  \frac{m\gamma}{(\sqrt m+\sqrt K+\gamma)^2}\right\},
  && \nabla^2U\succeq-K\Id,
  \label{eq:semiconvex-rate}\\
  \nu_{\alpha,\gamma}
  &\geq
  c\min\left\{\gamma,\alpha m+
  \frac{m\gamma}{
  (\sqrt m+M+M^{3/4}d^{1/4}+\gamma)^2}\right\},
  && \text{general case}.
  \label{eq:general-rate}
\end{align}
Thus position diffusion contributes through the Poincar\'e constant
even when the Hessian is not pointwise positive.
\end{remark}

\begin{remark}[Comparison with ULD]
The additional position diffusion in HFHR \eqref{eq:HFHR-SDE} compared to ULD \eqref{eqn:underdamped} has two analytically distinct
roles.  It provides direct coercivity through the product Poincar\'e
inequality, and it appears as an extra weak term in the space-time
argument of \cite[Theorem~2 and Lemma~2.6]{cao-lu-wang-2023}.  Since the divergence construction
already controls the derivative of the required test function, the second
role costs no new assumption on the potential.  Combining the two
estimates prevents this perturbative cost from degrading the ULD rate
and produces the additive HFHR scaling up to universal constants.
At $\alpha=0$, \eqref{eq:main-rate} gives the explicit
Cao--Lu--Wang rate \cite{cao-lu-wang-2023}:
\[
  \nu_{0,\gamma}
  \geq
  c\frac{m\gamma}{(\sqrt m+R+\gamma)^2}.
\]
Thus, our result extends the analysis for ULD in \cite{cao-lu-wang-2023}.
In the convex case, $R=0$.  At the ULD endpoint $\alpha=0$,
choosing $\gamma=\sqrt m$ gives
$\nu_{0,\gamma}\geq c\sqrt m/4\gtrsim\sqrt m$,
which gives the square-root improvement over the overdamped Langevin
$L^2$ rate $m$ discussed after \cite[Theorem~1]{cao-lu-wang-2023}.
This order is also the maximal possible acceleration for second-order
lifts in the non-asymptotic relaxation-time framework of
\cite[Theorem~11]{eberle-loerler-2026-lifts}.  If $\alpha$ remains of constant order as
$m \rightarrow 0$, the additional contribution $\alpha m$ is lower
order than $\sqrt m$.  Thus positive position diffusion does not
yield a further improvement in the asymptotic dependence on $m$ in
this regime, although it may improve constants at finite parameter
values.  Taking $\alpha$ large can improve the continuous-time rate, but it also might increase the
discretization error.  This motivates the
parameter optimization for the HFHRMC algorithm in
Section~\ref{sec:discrete}.
\end{remark}

\begin{remark}[Comparison with recent HFHR results]
\label{rem:recent-HFHR-comparison}
(i) Eberle and L\"orler \cite[Section~3.3]{eberle-loerler-2026-relaxation}
consider the same position Poincar\'e inequality and write the lower
Hessian bound as $\nabla^2U\succeq-cm\Id$. In our notation $c=K/m$.
Their relaxation time $t_{\rm rel}$ is the first time the operator norm of
$S_t^{\alpha,\gamma}$ on $L^2_0(\pi)$ falls to $e^{-1}$.
Semigroup contraction and iteration over intervals of length
$t_{\rm rel}$ give
$\|S_t^{\alpha,\gamma}\|_{L^2_0\to L^2_0}\leq e^{1-t/t_{\rm rel}}$.
At $\alpha=0$, their Theorem~27(i) and \eqref{eq:semiconvex-rate}
give the same rate up to universal constants, since
\[
  \frac{m\gamma}{(\sqrt m+\sqrt K+\gamma)^2}
  \asymp \min\!\left\{\frac{m}{\gamma},\frac{\gamma}{1+K/m}\right\}.
\]
For $\alpha>0$ and $-K\Id\preceq\nabla^2U\preceq L\Id$, the two
relaxation bounds in
\cite[equation~(27) and Theorem~27(ii)]{eberle-loerler-2026-relaxation} give
\begin{equation}\label{eq:el-relaxation-rate}
  t_{\rm rel}^{-1}\gtrsim
  \max\!\left\{\min\{\alpha m,\gamma\},\
  \min\!\left\{\frac{m+\alpha^2m^2}{2\gamma},
               \frac{\gamma}{5+4K/m+2L\alpha^2}\right\}\right\}.
\end{equation}
The dependence on $L$ matters even for Gaussian targets. Take
$U_n(q)=(q_1^2+n^3q_2^2)/2$ on $\mathbb R^2$, $n\geq1$. The position
spectral gap is $m=1$, and one can take $K=0$ and $L=n^3$.
With $\gamma=1$ and $\alpha=1/n$, the direct rate in
\eqref{eq:el-relaxation-rate} is $1/n$ and its second minimum is at
most $1/(5+2n)$. Thus inserting the global constants into
\eqref{eq:el-relaxation-rate} gives a lower bound of order $1/n$.
For the same targets and parameters, \eqref{eq:convex-rate} gives
$\nu_{1/n,1}\geq c\min\{1,1/n+1/4\}\geq c/4$ with the universal
constant of Theorem~\ref{thm:main}.

(ii) Under the $m$-strong convexity and $L$-smoothness assumptions
$m\Id_{d}\preceq\nabla^2U\preceq L\Id_{d}$,
the HFHR analysis in \cite[Theorem~5.1 and Lemma~D.5]{li-zha-tao-2022} proves weighted synchronous-coupling
contraction and, in the parameter regime
$\gamma^2>L+m$ and
$\alpha\leq\frac{\gamma^2-L-m}{m\gamma}$
obtains the position-marginal $\mathcal W_2$ contraction rate 
$\frac{m}{\gamma}+\alpha m$.
Our result gives $L^2$ and $\chi^2$ convergence under
the more general Assumptions~\ref{ass:PI}--\ref{ass:compact}.  
If $U$ is convex, then $R=0$.  If in addition
$\gamma\gg\sqrt m$, then
$\frac{m\gamma}{(\sqrt m+\gamma)^2}
  \asymp\frac{m}{\gamma}$,
and hence,
\[
  \nu_{\alpha,\gamma}
  \gtrsim
  \min\left\{\gamma,\alpha m+\frac m\gamma\right\}.
\]
Thus, in the
overlapping strongly convex high-friction regime, two
estimates have the same $\frac{m}{\gamma}+\alpha m$ scaling up to universal
constants.  The friction choice $\gamma=\sqrt m$ lies outside
their admissible range.  Our $L^2$ lower bound
\eqref{eq:convex-rate} permits this choice and, at $\alpha=0$,
gives $\nu_{0,\sqrt m}\gtrsim\sqrt m$.

(iii) In the normalization of the present paper,
\cite[Theorem~4.2]{cortild-delplancke-oudjane-peypouquet-2025}
gives relative-entropy decay at rate
$2\rho\min\{\alpha,\gamma\}$, assuming a logarithmic Sobolev
inequality with constant $\rho$ for the joint Gibbs measure.  That
estimate works directly in relative entropy but requires $\alpha>0$
and degenerates as $\alpha \rightarrow 0$.  By contrast, our $L^2$ rate
retains the positive hypocoercive contribution
$m\gamma/(\sqrt m+R+\gamma)^2$ at $\alpha=0$.

(iv) Wang, Wang and Zhu \cite[Corollaries~3.9, 4.13 and~4.14]{wang-wang-zhu-2026}
use global smoothness and dissipativity to control contraction between
initial laws in a Lyapunov-weighted Wasserstein distance. Their strict
rate improvement for small positive $\alpha$ uses an additional
asymptotically linear-gradient condition. Theorem~\ref{thm:main}
controls convergence to equilibrium in $L^2(\pi)$ for every
$\alpha\geq0$ under Assumptions~\ref{ass:PI}--\ref{ass:compact}.
\end{remark}

\begin{remark}[Dimension dependence of the Poincar\'e constant]
\label{rem:dimension-m}
The dimension dependence of the rate is partly encoded in the
Poincar\'e constant $m$.  It is dimension-free, for example, under
uniform strong convexity or for tensor products of one-dimensional
measures with a uniform spectral gap.  For genuinely nonconvex
multi-well targets, however, $m$ may be exponentially small in the
dimension or in the barrier height; see also
\cite[Corollary~2.15]{menz-schlichting-2014}.  Under smoothness and dissipativity
assumptions, \cite[Proposition~13 and Appendix~B]{Raginsky} derives a conservative spectral-gap bound
whose reciprocal can grow exponentially in $d$.  Thus
\eqref{eq:main-rate} is explicit in $m$, but it does not by itself
guarantee polynomial dimension dependence in general nonconvex
settings.
\end{remark}

\begin{remark}[Overdamped Langevin spectral gap]
\label{rem:OLD-gap}
The optimal value of $m$ in Assumption~\ref{ass:PI} is the spectral
gap of the overdamped Langevin generator
$\Delta_q-\nabla U\cdot\nabla_q$; any admissible $m$ is a lower
bound on this gap.  Consequently, for every centered
$g\in L^2(\mu)$,
\[
  \norm{P_t^{\mathrm{OLD}}g}_{L^2(\mu)}
  \leq e^{-mt}\norm g_{L^2(\mu)}.
\]
For forward densities this is equivalently a $\chi^2$-decay estimate
with exponent $2m$.  This statement is specific to $L^2$; it does
not automatically identify the Wasserstein contraction rate.
\end{remark}

The semigroup constructions in Propositions
\ref{prop:weak-semigroup} and \ref{prop:kinetic-endpoint} also
determine the forward equation, including at the kinetic endpoint, as
presented in the following proposition.

\begin{proposition}[Adjoint semigroup and forward equation]
\label{prop:forward-semigroup}
Let $\alpha\geq0$ and $\gamma>0$, and set
\[
  P_t^{\alpha,\gamma}:=(S_t^{\alpha,\gamma})^*
  \quad\text{on }L^2(\pi).
\]
Then $(P_t^{\alpha,\gamma})_{t\geq0}$ is a conservative Markov
contraction semigroup.  If $h_0\geq0$, $\pi(h_0)=1$, and
$h_0\in L^2(\pi)$, then
\[
  h_t=P_t^{\alpha,\gamma}h_0,
  \qquad
  \rho_t=h_t\pi,
\]
defines a probability measure for every $t\geq0$.  It is the weak
forward solution in the sense that, for every
$\varphi\in C_c^\infty(\R^{2d})$,
\begin{equation}\label{eq:weak-forward}
  \frac{\dd}{\dd t}\int_{\R^{2d}}\varphi h_t\,\dd\pi
  =
  \int_{\R^{2d}}(\cL_{\alpha,\gamma}\varphi)h_t\,\dd\pi
\end{equation}
in $\mathcal D'((0,\infty))$.  Moreover,
\begin{equation}\label{eq:adjoint-decay}
  \norm{P_t^{\alpha,\gamma}g}_{L^2(\pi)}
  \leq C e^{-\nu_{\alpha,\gamma}t}\norm g_{L^2(\pi)}
  \qquad
  \text{whenever }\pi(g)=0,
\end{equation}
with the constants and rate of Theorem~\ref{thm:main}.
\end{proposition}

\begin{proof}
Positivity of $S_t^{\alpha,\gamma}$ implies positivity of its adjoint,
because the positive cone of $L^2(\pi)$ is self-dual.  The two
identities
\[
  S_t^{\alpha,\gamma}1=1,
  \qquad
  \left\langle S_t^{\alpha,\gamma}f,1\right\rangle=\ip{f}{1}
\]
imply, respectively, preservation of mass by $P_t^{\alpha,\gamma}$
and $P_t^{\alpha,\gamma}1=1$.  Thus the adjoint is conservative and
Markovian.  It is strongly continuous because $L^2(\pi)$ is reflexive.

For $\varphi\in C_c^\infty$, Propositions
\ref{prop:weak-semigroup} and \ref{prop:kinetic-endpoint} identify the
backward infinitesimal with $\cL_{\alpha,\gamma}$ on $\varphi$.
The generator identity
\[
  S_t^{\alpha,\gamma}\varphi-\varphi
  =\int_0^t S_s^{\alpha,\gamma}
  \cL_{\alpha,\gamma}\varphi\,\dd s
\]
therefore implies, after pairing with $h_0$,
\[
  \ip{P_t^{\alpha,\gamma}h_0}{\varphi}
  -\ip{h_0}{\varphi}
  =\int_0^t
  \ip{P_s^{\alpha,\gamma}h_0}
  {\cL_{\alpha,\gamma}\varphi}\,\dd s.
\]
This is \eqref{eq:weak-forward} in distributions.  Finally, both
$S_t^{\alpha,\gamma}$ and
$P_t^{\alpha,\gamma}$ leave $L_0^2(\pi)$ invariant.  Their
restrictions to this closed subspace are adjoints of one another, and
hence have equal operator norms.  Theorem
\ref{thm:main} now yields \eqref{eq:adjoint-decay}.
\end{proof}

As an immediate corollary of 
Theorem~\ref{thm:main}, we obtain the following 
result for the finite-time convergence bound.

\begin{corollary}[$\chi^2$, KL and total-variation convergence]
\label{cor:chi-square-TV}
Under the assumptions of Theorem~\ref{thm:main}, let
$\rho_0=h_0\pi$, where $h_0\geq0$, $\pi(h_0)=1$, and
$h_0\in L^2(\pi)$.  Then
\begin{equation}\label{eq:chi-square}
 \mathrm{KL}(\rho_{t}\Vert\pi)\leq \chi^2(\rho_t\Vert\pi)
  \leq C^2e^{-2\nu_{\alpha,\gamma}t}
  \chi^2(\rho_0\Vert\pi),
\end{equation}
and
\[
  \norm{\rho_t-\pi}_{\mathrm{TV}}
  \leq
  \frac C2 e^{-\nu_{\alpha,\gamma}t}
  \sqrt{\chi^2(\rho_0\Vert\pi)}.
\]
\end{corollary}

\begin{proof}
Since
$h_t-1=P_t^{\alpha,\gamma}(h_0-1)$,
\eqref{eq:chi-square} follows from \eqref{eq:adjoint-decay}.  The
total-variation estimate follows from Cauchy-Schwarz inequality:
\[
  \norm{\rho-\pi}_{\mathrm{TV}}
  \leq\frac12
  \norm{\dd\rho/\dd\pi-1}_{L^2(\pi)}.
\]
Finally, the KL divergence is upper bounded by the $\chi^{2}$ divergence, which concludes the proof.
\end{proof}

\begin{remark}\label{remark:W:2}
If $\pi$ satisfies a log-Sobolev inequality with constant $\rho>0$,
Talagrand's inequality \cite[Theorem~22.15(i)]{villani-2009} implies
the following convergence in $2$-Wasserstein distance:
\begin{equation}
\mathcal{W}_{2}(\rho_{t},\pi)
\leq\sqrt{\frac{2}{\rho} \mathrm{KL}(\rho_{t}\Vert\pi)}
\leq\sqrt{\frac{2}{\rho}}
Ce^{-\nu_{\alpha,\gamma}t}
\sqrt{\chi^2(\rho_0\Vert\pi)}.
\end{equation}
\end{remark}

\subsection{Functional framework and the space--time estimate}
\label{sec:functional}

The additional diffusion in $q$ changes the natural energy space.
This section constructs the HFHR semigroup for $\alpha>0$ and records
the weak formulation used later. The endpoint $\alpha=0$ is
constructed below by vanishing position diffusion.

Let
\[
  V:=H^1(\pi)
  =\left\{f\in L^2(\pi):
  \nabla_qf,\nabla_pf\in L^2(\pi)\right\},
\]
that is equipped with its usual graph norm.  Smooth compactly supported
functions are dense in $V$.  Indeed, cutoff first reduces to a compact
set because $\pi$ is finite, and ordinary mollification applies there
because the density of $\pi$ is positive, continuous, and locally
bounded from above and below.

For smooth $f,g$, the weighted adjoint identities give
\begin{equation}\label{eq:hamiltonian-weak}
  \ip{\cL_{\mathrm{ham}}f}{g}_{L^2(\pi)}
  =
  \ip{\nabla_qf}{\nabla_pg}_{L^2(\pi)}
  -
  \ip{\nabla_pf}{\nabla_qg}_{L^2(\pi)}.
\end{equation}
The right-hand side extends continuously to $V\times V$.  For
$\alpha,\gamma>0$, define
\begin{align}
  \mathfrak a_{\alpha,\gamma}(f,g)
  :=
  \alpha\ip{\nabla_qf}{\nabla_qg}
  +\gamma\ip{\nabla_pf}{\nabla_pg}
  -\ip{\nabla_qf}{\nabla_pg}
  +\ip{\nabla_pf}{\nabla_qg}.
  \label{eq:HFHR-form}
\end{align}
All unlabelled inner products in this section are in $L^2(\pi)$.

We first turn the formal generator into a well-posed Markov evolution.
The coercive-form construction below also supplies the exact energy
identity needed in the decay argument.

\begin{proposition}[Weak HFHR semigroup]\label{prop:weak-semigroup}
Let $\alpha,\gamma>0$.  For every $f_0\in L^2(\pi)$ and $T>0$,
there exists a unique
\[
  f\in L^2(0,T;V)\cap C([0,T];L^2(\pi)),
  \qquad
  \partial_tf\in L^2(0,T;V^*),
\]
such that $f(0)=f_0$ and
\begin{equation}\label{eq:weak-HFHR}
  \ip{\partial_tf}{g}_{V^*,V}
  +\mathfrak a_{\alpha,\gamma}(f,g)=0
  \qquad\text{for a.e. }t\text{ and every }g\in V.
\end{equation}
The solutions define a strongly continuous contraction semigroup
$S_t^{\alpha,\gamma}$ on $L^2(\pi)$.  This semigroup is conservative
and Markov: $S_t^{\alpha,\gamma}1=1$, and
$0\leq S_t^{\alpha,\gamma}f_0\leq1$ for any $0\leq f_0\leq1$.

If $G_{\alpha,\gamma}$ denotes its infinitesimal generator, then
$C_c^\infty(\R^{2d})\subset D(G_{\alpha,\gamma})$ and
$G_{\alpha,\gamma}f=\cL_{\alpha,\gamma}f$ on this class.  The solution
preserves its $\pi$-mean and satisfies
\begin{equation}\label{eq:weak-energy}
  \frac12\norm{f_t}_{L^2(\pi)}^2
  +\int_0^t\left(
  \alpha\norm{\nabla_qf_s}_{L^2(\pi)}^2
  +\gamma\norm{\nabla_pf_s}_{L^2(\pi)}^2
  \right)\,\dd s
  =\frac12\norm{f_0}_{L^2(\pi)}^2.
\end{equation}
\end{proposition}

\begin{proof}
The form is continuous on $V\times V$.  Its antisymmetric terms
cancel on the diagonal, so
\[
  \mathfrak a_{\alpha,\gamma}(f,f)
  =\alpha\norm{\nabla_qf}^2+\gamma\norm{\nabla_pf}^2.
\]
Consequently,
$\mathfrak a_{\alpha,\gamma}+\ip{\cdot}{\cdot}$ is
coercive on $V$, with a coercivity constant that may depend on the
fixed pair $(\alpha,\gamma)$.  The representation and
semigroup-generation theorems for continuous coercive forms
\cite[Proposition~1.22 and Theorems~1.52 and~1.55]{ouhabaz-2005}
then give a unique solution in
$L^2(0,T;V)\cap C([0,T];L^2(\pi))$, with time derivative in
$L^2(0,T;V^*)$, and an associated strongly continuous contraction
semigroup.
For $f\in C_c^\infty$, integration by parts shows
$\mathfrak a_{\alpha,\gamma}(f,g)
=-\ip{\cL_{\alpha,\gamma}f}{g}$ for every $g\in V$.  Since
$\cL_{\alpha,\gamma}f\in L^2(\pi)$, the form representation gives
$f\in D(G_{\alpha,\gamma})$ and $G_{\alpha,\gamma}f=\cL_{\alpha,\gamma}f$.

Taking $g=1$ in \eqref{eq:weak-HFHR} proves mean preservation, while
the constant function is itself a solution, so
$S_t^{\alpha,\gamma}1=1$.  To prove positivity, use the negative part
$f_t^-=\max\{-f_t,0\}$ as a time-dependent test function.  The
Sobolev chain rule is valid in $V$, and the two antisymmetric terms
cancel on the set $\{f_t<0\}$.  Consequently,
\[
  \frac12\frac{\dd}{\dd t}\norm{f_t^-}_{L^2(\pi)}^2
  +\alpha\norm{\nabla_qf_t^-}_{L^2(\pi)}^2
  +\gamma\norm{\nabla_pf_t^-}_{L^2(\pi)}^2=0.
\]
Thus nonnegative initial data remain nonnegative.  Applying the same
argument to $(f_t-1)^+$ proves preservation of the interval
$[0,1]$.  This is the Markov property.  Finally, the Hilbert-space
chain rule applied with $g=f_t$ yields \eqref{eq:weak-energy}.
\end{proof}

\begin{remark}[Semigroup and diffusion interpretation]
\label{rem:diffusion-realization}
Because $C_c^\infty(\R^{2d})$ is dense both in $V$ and in
$C_0(\R^{2d})$, the form is regular.  Its transpose has the same
Markov contraction property. The representation theorem for regular
non-symmetric Dirichlet forms \cite[Theorems~IV.3.5 and V.1.5]{ma-rockner-1992}
realizes $S_t^{\alpha,\gamma}$ as the transition semigroup of a
conservative diffusion with martingale generator $\cL_{\alpha,\gamma}$
on $C_c^\infty$.
\end{remark}

\begin{definition}[HFHR kinetic energy space]\label{def:HFHR-space}
For $I=(0,T)$ and $\alpha\geq0$, let $\mathbb H_\alpha(I)$ be the
graph space of distributions $f$ such that
\[
  f\in L^2(\lambda_T\otimes\mu;H_\kappa^1),
  \qquad
  \sqrt\alpha\,\nabla_qf\in L^2(\lambda_T\otimes\pi),
\]
and
\[
  \cK_\alpha f
  \in L^2(\lambda_T\otimes\mu;H_\kappa^{-1}).
\]
For $\alpha=0$, this is the time-dependent kinetic space used in
\cite{cao-lu-wang-2023,albritton-armstrong-mourrat-novack-2024}.
\end{definition}

The preceding form construction degenerates when $\alpha=0$.
We therefore construct the kinetic endpoint by vanishing position
diffusion and retain the energy and Markov properties in the limit.

\begin{proposition}[Kinetic endpoint]\label{prop:kinetic-endpoint}
Let $\gamma>0$.  For every $f_0\in L^2(\pi)$ and $T>0$, there is
a unique
\[
  f\in C([0,T];L^2(\pi))\cap\mathbb H_0((0,T))
\]
with initial trace $f(0)=f_0$ such that
\begin{equation}\label{eq:kinetic-endpoint-equation}
  \cK_0f
  =\partial_tf-\cL_{\mathrm{ham}}f
  =-\gamma\nabla_p^*\nabla_pf
\end{equation}
in distributions.  These solutions define a conservative Markov
contraction semigroup $S_t^{0,\gamma}$, preserve the $\pi$-mean, and
satisfy, for its infinitesimal $G_{0,\gamma}$,
\[
  C_c^\infty(\R^{2d})\subset D(G_{0,\gamma}),
  \qquad
  G_{0,\gamma}\varphi=\cL_{0,\gamma}\varphi.
\]
They also satisfy
\begin{equation}\label{eq:kinetic-endpoint-energy}
  \frac12\norm{f_t}_{L^2(\pi)}^2
  +\gamma\int_0^t\norm{\nabla_pf_s}_{L^2(\pi)}^2\,\dd s
  =\frac12\norm{f_0}_{L^2(\pi)}^2.
\end{equation}
\end{proposition}

\begin{proof}
We give a vanishing-position-diffusion construction.  For
$\varepsilon>0$, let
$f^\varepsilon=S_t^{\varepsilon,\gamma}f_0$ be the solution from
Proposition~\ref{prop:weak-semigroup}.  These solutions are defined for
all positive times.  By \eqref{eq:weak-energy}, for every $t\geq0$,
\begin{equation}\label{eq:epsilon-energy}
  \norm{f_t^\varepsilon}_{L^2(\pi)}^2
  +2\gamma\int_0^t\norm{\nabla_pf_s^\varepsilon}_{L^2(\pi)}^2\,\dd s
  +2\varepsilon\int_0^t
  \norm{\nabla_qf_s^\varepsilon}_{L^2(\pi)}^2\,\dd s
  =\norm{f_0}_{L^2(\pi)}^2.
\end{equation}
In particular, for every finite $T$,
\[
  \sup_{0\leq t\leq T}\norm{f_t^\varepsilon}_{L^2(\pi)}^2
  \leq\norm{f_0}_{L^2(\pi)}^2.
\]
Weak compactness yields, along a sequence
$\varepsilon \rightarrow 0$,
\[
  f^\varepsilon\rightharpoonup^\ast f
  \quad\text{in }L^\infty(0,T;L^2(\pi)),
  \qquad
  \nabla_pf^\varepsilon\rightharpoonup\nabla_pf
  \quad\text{in }L^2(0,T;L^2(\pi)).
\]
A diagonal extraction over integer values of $T$ gives a single
limit on $(0,\infty)$.
For every compactly supported smooth test function $\chi$,
\[
  \left|\varepsilon\int_0^T
  \ip{\nabla_qf_s^\varepsilon}{\nabla_q\chi_s}\,\dd s\right|
  \leq
  \sqrt\varepsilon
  \left(\sqrt\varepsilon\norm{\nabla_qf^\varepsilon}_{L^2}\right)
  \norm{\nabla_q\chi}_{L^2}
  \longrightarrow0.
\]
Passing to the limit in the weak equation therefore gives
\eqref{eq:kinetic-endpoint-equation}; in particular
$f\in\mathbb H_0((0,T))$.

The same weak equations give compactness in weak time continuity.  More
precisely, for every $\chi\in C_c^\infty$, the Hamiltonian term is
uniformly integrable in time, the velocity-diffusion term is uniformly
one-half H\"older in duality with $\chi$, and the position-diffusion
term is $\mathcal{O}(\sqrt\varepsilon)$ by \eqref{eq:epsilon-energy}.  A diagonal
argument over a countable dense family of tests therefore gives, as
$\varepsilon \rightarrow 0$,
\[
  f^\varepsilon\longrightarrow f
  \quad\text{in }C([0,T];L^2(\pi)_{\mathrm w}).
\]
Since $f_0^\varepsilon=f_0$, the limit has weak $L^2$ initial trace
$f_0$.

Proposition~\ref{prop:kinetic-green} first gives
$f\in C((0,\infty);L^2(\pi))$.  Weak lower semicontinuity in
\eqref{eq:epsilon-energy}, using the weak convergence at each fixed
time supplied above, gives
\begin{equation}\label{eq:kinetic-limit-inequality}
  \frac12\norm{f_t}_{L^2(\pi)}^2
  +\gamma\int_0^t\norm{\nabla_pf_s}_{L^2(\pi)}^2\,\dd s
  \leq\frac12\norm{f_0}_{L^2(\pi)}^2
  \qquad(t>0).
\end{equation}
In particular,
$\limsup_{t \rightarrow 0}\norm{f_t}_{L^2(\pi)}\leq\norm{f_0}_{L^2(\pi)}$.
The endpoint criterion in Proposition~\ref{prop:kinetic-green} therefore
upgrades the weak initial trace to a strong one.  Corollary
\ref{cor:kinetic-chain-rule}, first on interior intervals and then with
the initial endpoint included, proves
\eqref{eq:kinetic-endpoint-energy}.  Thus
$f\in C([0,T];L^2(\pi))$ for every finite $T$.

Applying the same chain rule to the difference of two solutions proves
uniqueness. Hence, the whole family
$f^\varepsilon$, not merely a subsequence, converges weakly to $f$;
uniqueness also gives the semigroup property.  Contractivity follows from
\eqref{eq:kinetic-endpoint-energy}.  Mean preservation passes through
the weak limit.  Positivity and preservation of the interval $[0,1]$
pass as well, because these are weakly closed convex conditions in
$L^2(\pi)$ and hold for every $S_t^{\varepsilon,\gamma}$.  Thus
$S_t^{0,\gamma}$ is conservative and Markov.

It remains to record the infinitesimal identity used below.  Fix
$\varphi\in C_c^\infty$ and put
$g=\cL_{0,\gamma}\varphi\in L^2(\pi)$.  By linearity of the weak
equation, the variation-of-constants function
\[
  v_t=\varphi+\int_0^tS_s^{0,\gamma}g\,\dd s
\]
belongs locally to the kinetic energy space.  Since
$g=\cL_{0,\gamma}\varphi$, integration of the weak equation for
$S_s^{0,\gamma}g$ gives
\[
  \cK_0v
  =\cK_0\varphi+g
  +\int_0^t\cK_0\left(S_s^{0,\gamma}g\right)\,\dd s
  =-\gamma\nabla_p^*\nabla_pv.
\]
Thus, $v$ solves the same weak Cauchy problem as
$S_t^{0,\gamma}\varphi$.  Uniqueness gives
$v_t=S_t^{0,\gamma}\varphi$.  Dividing by $t$ and using strong
continuity yields
\[
  \lim_{t \rightarrow 0}\frac{S_t^{0,\gamma}\varphi-\varphi}{t}=g
  \quad\text{in }L^2(\pi),
\]
which proves the asserted infinitesimal inclusion.
\end{proof}

For every $\alpha\geq0$, the solutions constructed in Propositions
\ref{prop:weak-semigroup} and \ref{prop:kinetic-endpoint} satisfy in distributions,
by \eqref{eq:weak-HFHR} and \eqref{eq:kinetic-endpoint-equation},
\begin{equation}\label{eq:weak-residual}
  \cK_\alpha f=-\gamma\nabla_p^*\nabla_pf.
\end{equation}
In particular, the right-hand side belongs to
$L^2(\lambda_T\otimes\mu;H_\kappa^{-1})$, with
\begin{equation}\label{eq:weak-residual-bound}
  \norm{\cK_\alpha f}_{L^2(\lambda_T\otimes\mu;H_\kappa^{-1})}
  \leq\gamma\norm{\nabla_pf}_{L^2(\lambda_T\otimes\pi)}.
\end{equation}
Together with \eqref{eq:weak-energy} and
\eqref{eq:kinetic-endpoint-energy}, this proves that the solutions
belong to $\mathbb H_\alpha(I)$ on every finite time interval.

\paragraph{The modified time-augmented Poincar\'e inequality.}

The proof uses the divergence construction of
\cite[Lemma~2.6]{cao-lu-wang-2023}, itself motivated by the variational
kinetic framework of
\cite{albritton-armstrong-mourrat-novack-2024}.  Define
\begin{align}
  B_0(T)
  &:=
  \frac{1}{\sqrt m(1-e^{-\sqrt mT})}+T,
  \label{eq:B0}\\
  B_1(T)
  &:=
  1+RT+
  \frac{1}{(1-e^{-\sqrt mT})^2}
  +
  \frac{R}{\sqrt m(1-e^{-\sqrt mT})^2}.
  \label{eq:B1}
\end{align}

The macroscopic component of $f$ will be estimated by duality
against a time--position divergence equation.  The following lemma
provides a test field with the first- and second-order bounds required
for that argument.

\begin{lemma}[Cao--Lu--Wang divergence test]\label{lem:divergence}
For every centered $h\in L^2(\lambda_T\otimes\mu)$, there are
$\phi_0\in H_0^1(\lambda_T\otimes\mu)$ and
$\psi\in L^2(\lambda_T;H^2(\mu))$, with
$\nabla_q\psi\in H_0^1(\lambda_T\otimes\mu)^d$, such that
\begin{equation}\label{eq:divergence}
  -\partial_t\phi_0+\nabla_q^*\nabla_q\psi=h.
\end{equation}
Moreover,
\begin{align}
  \norm{\phi_0}_{L^2}+\norm{\nabla_q\psi}_{L^2}
  &\leq C B_0(T)\norm{h}_{L^2},
  \label{eq:div-bound-0}\\
  \norm{\nabla_q\phi_0}_{L^2}
  +\norm{\nabla_{t,q}\nabla_q\psi}_{L^2}
  &\leq C B_1(T)\norm{h}_{L^2}.
  \label{eq:div-bound-1}
\end{align}
\end{lemma}

\begin{proof}
Appendix~\ref{app:divergence} gives the spectral construction.
\end{proof}

To use the divergence field with functions in the kinetic graph
space, we need an integration-by-parts identity that remains valid at
the available weak regularity.  It also isolates the additional term
created by position diffusion.

\begin{lemma}[Graph-space integration by parts]
\label{lem:graph-integration}
Let $f\in\mathbb H_\alpha(I)$, let
$\phi_0\in H_0^1(\lambda_T\otimes\mu)$, and let
$\zeta\in H_0^1(\lambda_T\otimes\mu)^d$.  Set
\[
  \Phi(t,q,p):=\phi_0(t,q)-p\cdot\zeta(t,q),
\]
and
\begin{align}
  \mathscr T\Phi
  :=-\partial_t\phi_0+p\cdot\nabla_q\phi_0
  +p\cdot\partial_t\zeta-p^\top\nabla_q\zeta\,p
  +\nabla U\cdot\zeta.
  \label{eq:test-transport}
\end{align}
Then $\Phi\in L^2(\lambda_T\otimes\mu;H_\kappa^1)$,
$\nabla_q\Phi\in L^2(\lambda_T\otimes\pi)$,
$\mathscr T\Phi\in L^2(\lambda_T\otimes\pi)$, and
\begin{equation}\label{eq:graph-integration}
  \int_{I\times\mathbb{R}^{2d}}f\,\mathscr T\Phi\,\dd\lambda_T\,\dd\pi
  =
  \ip{\cK_\alpha f}{\Phi}_{H_\kappa^{-1},H_\kappa^1}
  -\alpha\int_{I\times\mathbb{R}^{2d}}\nabla_qf\cdot\nabla_q\Phi\,\dd\lambda_T\,\dd\pi.
\end{equation}
\end{lemma}

\begin{proof}
The only coefficient requiring attention is
$\nabla U\cdot\zeta$.  Lemma~\ref{lem:weighted-multiplier}, applied
componentwise, shows that it belongs to $L^2$; the remaining assertions
follow from the Gaussian moments of $1,p_i,p_ip_j$.

For $\phi_0,\zeta$ smooth, compactly supported in $q$, and vanishing
near $t=0,T$, weighted integration by parts in $t,q,p$ gives
\[
  \int_{I\times\mathbb{R}^{2d}} f\,\mathscr T\Phi\,\dd\lambda_T\,\dd\pi
  =\ip{\partial_tf-\cL_{\mathrm{ham}}f}{\Phi}
  =\ip{\cK_\alpha f}{\Phi}
   -\alpha\ip{\nabla_qf}{\nabla_q\Phi}.
\]
To pass to the stated Sobolev class, approximate $\phi_0$ and
$\zeta$ by trace-preserving smooth functions in $H_0^1$, and insert
a Gaussian velocity cutoff in the corresponding $\Phi$.  The usual
Sobolev convergence handles all first-order terms, while Lemma
\ref{lem:weighted-multiplier} gives
\[
  \norm{(\zeta_n-\zeta)\nabla U}_{L^2}
  \leq C(M,d)\norm{\zeta_n-\zeta}_{H^1}.
\]
Thus $\Phi_n\to\Phi$ in $L^2(H_\kappa^1)$,
$\nabla_q\Phi_n\to\nabla_q\Phi$ in $L^2$, and
$\mathscr T\Phi_n\to\mathscr T\Phi$ in $L^2$.  Each term in the
identity therefore converges, proving \eqref{eq:graph-integration}.
\end{proof}

We now combine the divergence construction with the graph-space
identity.  This yields the HFHR analogue of the time-augmented
Poincar\'e inequality used for ULD.

\begin{proposition}[HFHR space-time inequality]\label{prop:space-time}
Under Assumptions~\ref{ass:PI}--\ref{ass:compact}, let
$f\in\mathbb H_\alpha(I)$.  When $\alpha=0$, the last term below
is understood to vanish.
Then
\begin{align}
  \norm{f-(f)_{\lambda_T\otimes\pi}}_{L^2(\lambda_T\otimes\pi)}
  \leq C\bigg[&
  B_1(T)\norm{(\Id-\PiP)f}_{L^2(\lambda_T\otimes\pi)}
  \notag\\
  &+B_0(T)
  \norm{\cK_\alpha f}_{
  L^2(\lambda_T\otimes\mu;H_\kappa^{-1})}
+\alpha B_1(T)
  \norm{\nabla_qf}_{L^2(\lambda_T\otimes\pi)}
  \bigg].
  \label{eq:HFHR-space-time}
\end{align}
For $\alpha=0$, this is the time-augmented Poincar\'e inequality of
\cite[Theorem~2]{cao-lu-wang-2023}.
\end{proposition}

\begin{proof}
It suffices to prove the estimate for functions with zero global
mean.  Set
\[
  h:=\PiP f,\qquad g:=(\Id-\PiP)f,
\]
so that $f=h+g$, $\PiP g=0$, and
$(h)_{\lambda_T\otimes\mu}=0$.  Apply Lemma
\ref{lem:divergence} to $h$, and define
\begin{equation}\label{eq:test-function}
  \Phi(t,q,p)
  :=
  \phi_0(t,q)-p\cdot\nabla_q\psi(t,q).
\end{equation}
Gaussian moment identities and
\eqref{eq:div-bound-0}--\eqref{eq:div-bound-1} imply
\begin{align}
  \norm{\Phi}_{L^2(\lambda_T\otimes\mu;H_\kappa^1)}
  &\leq C B_0(T)\norm{\PiP f}_{L^2},
  \label{eq:Phi-0}\\
  \norm{\nabla_q\Phi}_{L^2(\lambda_T\otimes\pi)}
  &\leq C B_1(T)\norm{\PiP f}_{L^2}.
  \label{eq:Phi-1}
\end{align}

Take $\zeta=\nabla_q\psi$ in \eqref{eq:test-transport}.
The Gaussian identities
\begin{equation}\label{eq:Gaussian-moments}
  \int_{\mathbb{R}^{d}} p_i\,\dd\kappa=0,\qquad
  \int_{\mathbb{R}^{d}} p_ip_j\,\dd\kappa=\delta_{ij},\qquad
  \int_{\mathbb{R}^{d}} p_ip_jp_k\,\dd\kappa=0,
\end{equation}
and \eqref{eq:divergence} give
\[
  \PiP(\mathscr T\Phi)
  =-\partial_t\phi_0-\Delta_q\psi+\nabla U\cdot\nabla_q\psi=h.
\]
Since $f=h+g$, it follows that
\begin{equation}\label{eq:macro-micro-identity}
  \norm h_{L^2}^2
  =\int_{I\times\mathbb{R}^{2d}} f\,\mathscr T\Phi\,\dd\lambda_T\,\dd\pi
  -\int_{I\times\mathbb{R}^{2d}} g\,\mathscr T\Phi\,\dd\lambda_T\,\dd\pi.
\end{equation}

Because $\phi_0$ and $\nabla_q\psi$ have zero traces at both time
endpoints, Lemma~\ref{lem:graph-integration}, with
$\zeta=\nabla_q\psi$, gives
\begin{equation}\label{eq:macro-duality}
  \int_{I\times\mathbb{R}^{2d}} f\,\mathscr T\Phi\,\dd\lambda_T\,\dd\pi
  =\ip{\cK_\alpha f}{\Phi}
  -\alpha\ip{\nabla_qf}{\nabla_q\Phi}.
\end{equation}
Duality, the Cauchy--Schwarz inequality, and
\eqref{eq:Phi-0}--\eqref{eq:Phi-1} yield
\begin{align}
  \abs{\int_{I\times\mathbb{R}^{2d}} f\,\mathscr T\Phi\,\dd\lambda_T\,\dd\pi}
  &\leq
  \norm{\cK_\alpha f}_{L^2(H_\kappa^{-1})}
  \norm\Phi_{L^2(H_\kappa^1)}
  +
  \alpha\norm{\nabla_qf}_{L^2}\norm{\nabla_q\Phi}_{L^2}
  \notag\\
  &\leq C\left[
  B_0(T)\norm{\cK_\alpha f}_{L^2(H_\kappa^{-1})}
  +\alpha B_1(T)\norm{\nabla_qf}_{L^2}
  \right]\norm h_{L^2}.
  \label{eq:new-position-term}
\end{align}

It remains to estimate the second term of
\eqref{eq:macro-micro-identity}.  By \eqref{eq:test-transport} and
\eqref{eq:divergence},
\begin{equation}\label{eq:centered-test-transport}
  \mathscr T\Phi-h
  =p\cdot(\nabla_q\phi_0+\partial_t\nabla_q\psi)
  -p^\top\nabla_q^2\psi\,p+\Delta_q\psi.
\end{equation}
Expanding the square in \eqref{eq:centered-test-transport}, using
\eqref{eq:Gaussian-moments} and the fourth-moment identity
\[
  \int_{\mathbb{R}^{d}} p_ip_jp_kp_\ell\,\dd\kappa
  =\delta_{ij}\delta_{k\ell}
  +\delta_{ik}\delta_{j\ell}
  +\delta_{i\ell}\delta_{jk},
\]
and then applying \eqref{eq:div-bound-1} gives
\begin{align}
  \norm{\mathscr T\Phi-h}_{L^2(\lambda_T\otimes\pi)}^2
  &=\norm{\nabla_q\phi_0+\partial_t\nabla_q\psi}_{L^2(\lambda_T\otimes\mu)}^2
  +2\norm{\nabla_q^2\psi}_{L^2(\lambda_T\otimes\mu)}^2
  \notag\\
  &\leq C B_1(T)^2\norm h_{L^2(\lambda_T\otimes\mu)}^2.
  \label{eq:centered-transport-bound}
\end{align}
Since $\PiP g=0$, the Cauchy--Schwarz inequality and
\eqref{eq:centered-transport-bound} imply
\begin{align}
  \abs{\int_{I\times\mathbb{R}^{2d}} g\,\mathscr T\Phi\,\dd\lambda_T\,\dd\pi}
  &=\abs{\int_{I\times\mathbb{R}^{2d}}g(\mathscr T\Phi-h)\,\dd\lambda_T\,\dd\pi}
  \notag\\
  &\leq C B_1(T)\norm g_{L^2}\norm h_{L^2}.
  \label{eq:micro-duality-bound}
\end{align}

If $h=0$, the desired estimate follows immediately from
$f=g$ and $B_1(T)\geq1$.  Otherwise, insert
\eqref{eq:new-position-term} and \eqref{eq:micro-duality-bound} into
\eqref{eq:macro-micro-identity} and divide by $\norm h_{L^2}$.  This
controls the macroscopic part.  Finally,
\[
  \norm f_{L^2}
  \leq
  \norm h_{L^2}+\norm g_{L^2},
\]
and the microscopic term is absorbed because $B_1(T)\geq1$.
\end{proof}

\subsection{Proofs of the main results}
\label{sec:energy}

For $\alpha>0$, Proposition~\ref{prop:weak-semigroup} gives the
energy identity below directly for the weak solution.  For
$\alpha=0$, it is provided by Proposition
\ref{prop:kinetic-endpoint}:
\begin{equation}\label{eq:energy}
  \frac{\dd}{\dd t}\norm{f_t}_{L^2(\pi)}^2
  =
  -2\alpha\norm{\nabla_qf_t}_{L^2(\pi)}^2
  -2\gamma\norm{\nabla_pf_t}_{L^2(\pi)}^2.
\end{equation}
Write
\begin{equation}\label{eq:D}
  \cD_{\alpha,\gamma}(f)
  :=
  \alpha\norm{\nabla_qf}_{L^2(\pi)}^2
  +\gamma\norm{\nabla_pf}_{L^2(\pi)}^2.
\end{equation}

\paragraph{Direct coercivity.}

Position diffusion also gives a direct coercive mechanism that is
independent of hypocoercivity.  Tensorization of the position and
Gaussian Poincar\'e inequalities makes this contribution explicit.

\begin{lemma}[Product Poincar\'e bound]\label{lem:direct}
For every centered $f\in H^1(\pi)$,
\begin{equation}\label{eq:product-PI}
  \norm f_{L^2(\pi)}^2
  \leq
  \frac1m\norm{\nabla_qf}_{L^2(\pi)}^2
  +\norm{\nabla_pf}_{L^2(\pi)}^2.
\end{equation}
Consequently,
\begin{equation}\label{eq:direct-rate}
  \norm{S_t^{\alpha,\gamma}f}_{L^2(\pi)}
  \leq
  e^{-\nu_{\mathrm{dir}}t}\norm f_{L^2(\pi)},
  \qquad
  \nu_{\mathrm{dir}}=\min\{\alpha m,\gamma\}.
\end{equation}
\end{lemma}

\begin{proof}
Tensorize Assumption~\ref{ass:PI} with the Gaussian Poincar\'e
inequality.  Then
\[
  \cD_{\alpha,\gamma}(f)
  \geq
  \min\{\alpha m,\gamma\}\norm f_{L^2(\pi)}^2.
\]
For $\alpha>0$, the conclusion follows from the weak energy identity.
When $\alpha=0$, $\nu_{\mathrm{dir}}=0$, and the asserted estimate
is simply the $L^2$-contractivity following from the kinetic energy
identity; no $q$-regularity is needed at that endpoint.
\end{proof}

\paragraph{The time-averaged hypocoercive bound.}

Fix an interval $I=(s,s+T)$, translated to $(0,T)$, and use the
normalized measure $\lambda_T$.  The HFHR solution satisfies
\eqref{eq:weak-residual-bound} on this interval.  The Gaussian
Poincar\'e inequality also gives
\begin{equation}\label{eq:micro-bound}
  \norm{(\Id-\PiP)f}_{L^2(\lambda_T\otimes\pi)}
  \leq
  \norm{\nabla_pf}_{L^2(\lambda_T\otimes\pi)}.
\end{equation}

Combining Proposition~\ref{prop:space-time} with
\eqref{eq:weak-residual-bound} and \eqref{eq:micro-bound}, and then applying
weighted Cauchy--Schwarz inequality, yields
\begin{align}
  \norm f_{L^2(\lambda_T\otimes\pi)}
  \leq C\left[
  (B_1+\gamma B_0)
  \norm{\nabla_pf}_{L^2(\lambda_T\otimes\pi)}
  +\alpha B_1
  \norm{\nabla_qf}_{L^2(\lambda_T\otimes\pi)}
  \right],
  \label{eq:interval-first}
\end{align}
and hence
\begin{equation}\label{eq:interval-coercivity}
  \norm f_{L^2(\lambda_T\otimes\pi)}^2
  \leq C K_{\alpha,\gamma}(T)
  \int_I\cD_{\alpha,\gamma}(f_t)\,\dd\lambda_T(t),
\end{equation}
where
\begin{equation}\label{eq:K}
  K_{\alpha,\gamma}(T)
  :=
  \frac{(B_1(T)+\gamma B_0(T))^2}{\gamma}
  +\alpha B_1(T)^2.
\end{equation}

The following elementary bookkeeping lemma records both sides of the
estimate needed in the iteration.  Its lower bound is important: it
makes the prefactor in the resulting exponential estimate universal,
rather than merely finite for each fixed set of parameters.

\begin{lemma}[Constants at the natural time scale]
\label{lem:natural-scale}
Let
\begin{equation}\label{eq:Q-alpha-gamma}
T_*:=m^{-1/2},\qquad A:=\sqrt m+R,
\qquad
  Q_{\alpha,\gamma}
  :=
  \frac{(A+\gamma)^2}{m\gamma}
  +\frac{\alpha A^2}{m}.
\end{equation}
Then, with a universal constant $C>0$,
\begin{equation}\label{eq:natural-scale-ledger}
  B_0(T_*)\leq\frac{C}{\sqrt m},\qquad
  B_1(T_*)\leq C\frac{A}{\sqrt m},\qquad
  4T_*\leq K_{\alpha,\gamma}(T_*)
  \leq C Q_{\alpha,\gamma}.
\end{equation}
\end{lemma}

\begin{proof}
At $T=T_*$, one has
$\eta_*:=1-e^{-\sqrt mT_*}=1-e^{-1}$, and in fact
\[
  B_0(T_*)=\frac{1+\eta_*^{-1}}{\sqrt m},
  \qquad
  B_1(T_*)=(1+\eta_*^{-2})\frac{A}{\sqrt m}.
\]
This proves the first two estimates and implies
\[
  B_1(T_*)+\gamma B_0(T_*)
  \leq C\frac{A+\gamma}{\sqrt m},
\]
which gives the upper bound for $K_{\alpha,\gamma}(T_*)$.
For the lower bound, the arithmetic--geometric mean inequality gives,
for every $T>0$,
\[
  \frac{(B_1(T)+\gamma B_0(T))^2}{\gamma}
  \geq4B_1(T)B_0(T)\geq4T,
\]
because $B_1(T)\geq1$ and $B_0(T)\geq T$.  This proves
\eqref{eq:natural-scale-ledger}.
\end{proof}

The interval coercivity estimate can now be combined with energy
dissipation and iterated over intervals of the natural length
$m^{-1/2}$.  The resulting estimate records the purely variational
decay rate.

\begin{proposition}[Variational rate]\label{prop:variational-rate}
There are universal constants $c,C>0$ such that, for every
$f\in L_0^2(\pi)$,
\begin{equation}\label{eq:variational-decay}
  \norm{S_t^{\alpha,\gamma}f}_{L^2(\pi)}
  \leq
  C e^{-\nu_{\mathrm{var}}t}\norm f_{L^2(\pi)},
\end{equation}
with
\begin{equation}\label{eq:variational-rate}
  \nu_{\mathrm{var}}
  \geq
  c\left[
  \frac{(\sqrt m+R+\gamma)^2}{m\gamma}
  +
  \frac{\alpha(\sqrt m+R)^2}{m}
  \right]^{-1}.
\end{equation}
\end{proposition}

\begin{proof}
Let $E(t):=\norm{f_t}_{L^2(\pi)}^2$.  Integrating
\eqref{eq:energy} over $I=(s,s+T)$ gives
\begin{equation}\label{eq:interval-energy}
  E(s)-E(s+T)
  =
  2T\int_I\cD_{\alpha,\gamma}(f_t)\,\dd\lambda_T(t).
\end{equation}
Since $E$ is nonincreasing,
\[
  E(s+T)
  \leq\norm f_{L^2(\lambda_T\otimes\pi)}^2.
\]
Using \eqref{eq:interval-coercivity} and
\eqref{eq:interval-energy},
\[
  E(s+T)
  \leq
  \frac{C_{\mathrm{st}}K_{\alpha,\gamma}(T)}{2T}
  \left(E(s)-E(s+T)\right).
\]
Here $C_{\mathrm{st}}$ denotes one fixed universal constant for which
\eqref{eq:interval-coercivity} holds.  Set
\[
  x_T:=\frac{2T}{C_{\mathrm{st}}K_{\alpha,\gamma}(T)}.
\]
After moving the last occurrence of $E(s+T)$ to the left, the
previous inequality becomes the exact one-step contraction
\begin{equation}\label{eq:one-step-contraction}
  E(s+T)\leq(1+x_T)^{-1}E(s).
\end{equation}
Consequently, if $k=\lfloor t/T\rfloor$, monotonicity of $E$ and
iteration of \eqref{eq:one-step-contraction} give
\begin{align}
  E(t)
 \leq(1+x_T)^{-k}E(0)
\leq(1+x_T)
  \exp\left[-\frac{t}{T}\log(1+x_T)\right]E(0).
  \label{eq:continuous-from-discrete}
\end{align}

Choose $T=T_*=m^{-1/2}$.  Lemma~\ref{lem:natural-scale} implies
$0<x_{T_*}\leq C_0$ for a universal $C_0$.  Hence both the
prefactor $(1+x_{T_*})^{1/2}$ in the norm estimate and the constant
in
\[
  \log(1+x_{T_*})
  \geq\frac{x_{T_*}}{1+C_0}
\]
are universal.  Taking square roots in
\eqref{eq:continuous-from-discrete} therefore gives a decay exponent
\[
  \nu_{\mathrm{var}}
  =\frac1{2T_*}\log(1+x_{T_*})
  \geq\frac{c}{K_{\alpha,\gamma}(T_*)}
  \geq\frac{c}{Q_{\alpha,\gamma}}.
\]
Inserting \eqref{eq:Q-alpha-gamma} proves
\eqref{eq:variational-rate} and shows at the same time that the
constant $C$ in \eqref{eq:variational-decay} is independent of all
parameters.
\end{proof}

\paragraph{Combining the two mechanisms.}

The variational estimate retains the ULD contribution but deteriorates
when $\alpha$ is large, whereas direct coercivity improves with
$\alpha$.  The following elementary lemma combines the two
complementary bounds.

\begin{lemma}[Rate algebra]\label{lem:algebra}
Let $A\geq\sqrt m$, $\alpha\geq0$, and $\gamma>0$.  Set
\[
  \nu_0:=\frac{m\gamma}{(A+\gamma)^2},\qquad
  \nu_1:=
  \left(\nu_0^{-1}+\frac{\alpha A^2}{m}\right)^{-1},
  \qquad
  \nu_2:=\min\{\alpha m,\gamma\}.
\]
Then, for a universal $c>0$,
\begin{equation}\label{eq:algebra}
  \max\{\nu_1,\nu_2\}
  \geq
  c\min\{\gamma,\alpha m+\nu_0\}.
\end{equation}
\end{lemma}

\begin{proof}
First, $\nu_0\leq\gamma$.  If
$\alpha A^2/m\leq\nu_0^{-1}$, then
$\nu_1\geq\nu_0/2$.  Otherwise,
\[
  \frac{\alpha m}{\nu_0}
  >
  \frac{m^2}{A^2\nu_0^2}
  =
  \frac{(A+\gamma)^4}{A^2\gamma^2}
  \geq16.
\]
The last inequality is simply
$(A+\gamma)^2\geq4A\gamma$.  Hence in the second case
$\alpha m\geq16\nu_0$, and therefore
$\nu_2=\min\{\alpha m,\gamma\}\geq\nu_0$.  In both cases it follows
that
\[
  \max\{\nu_1,\nu_2\}
  \geq\frac12
  \max\{\nu_0,\min(\alpha m,\gamma)\}.
\]
Finally, for any $a\leq g$ and $b\geq0$,
\[
  \max\{a,\min(b,g)\}
  \geq\frac12\min\{g,a+b\}.
\]
Apply this elementary inequality with
$a=\nu_0$, $b=\alpha m$, and $g=\gamma$ to obtain
\eqref{eq:algebra}, for instance with $c=1/4$.
\end{proof}

Finally, we are ready to prove Theorem~\ref{thm:main}.

\begin{proof}[Proof of Theorem~\ref{thm:main}]
Set
\[
  A:=\sqrt m+R,
  \qquad
  \nu_0:=\frac{m\gamma}{(A+\gamma)^2},
\]
and define
\begin{equation}\label{eq:main-proof-rates}
  \nu_1
  :=\left(\nu_0^{-1}+\frac{\alpha A^2}{m}\right)^{-1},
  \qquad
  \nu_2:=\min\{\alpha m,\gamma\}.
\end{equation}

By density and $L^2(\pi)$-contractivity, the estimate
\eqref{eq:direct-rate} extends to every $f\in L_0^2(\pi)$.
Proposition~\ref{prop:variational-rate} and Lemma~\ref{lem:direct}
therefore give, for every $\alpha\geq0$, universal constants
$c_1,C_1>0$ such that
\begin{equation}\label{eq:main-proof-variational}
  \norm{S_t^{\alpha,\gamma}f}_{L^2(\pi)}
  \leq C_1e^{-c_1\nu_1t}\norm f_{L^2(\pi)},
\end{equation}
and
\begin{equation}\label{eq:main-proof-direct}
  \norm{S_t^{\alpha,\gamma}f}_{L^2(\pi)}
  \leq e^{-\nu_2t}\norm f_{L^2(\pi)}.
\end{equation}
After decreasing $c_1$ so that $c_1\leq1$ and increasing
$C_1$ so that $C_1\geq1$, selecting
\eqref{eq:main-proof-variational} when $\nu_1\geq\nu_2$ and
\eqref{eq:main-proof-direct} otherwise yields
\begin{equation}\label{eq:main-proof-best-rate}
  \norm{S_t^{\alpha,\gamma}f}_{L^2(\pi)}
  \leq
  C_1e^{-c_1\max\{\nu_1,\nu_2\}t}\norm f_{L^2(\pi)}.
\end{equation}

Finally, Lemma~\ref{lem:algebra}, applied with
$A=\sqrt m+R$, gives
\[
  \max\{\nu_1,\nu_2\}
  \geq
  c_2\min\left\{
  \gamma,\,
  \alpha m+
  \frac{m\gamma}{(\sqrt m+R+\gamma)^2}
  \right\}
\]
for a universal $c_2>0$.  Inserting this estimate into
\eqref{eq:main-proof-best-rate} proves
\eqref{eq:main-decay}--\eqref{eq:main-rate}.  
This completes the proof.
\end{proof}

\section{Discrete-Time Analysis}\label{sec:discrete}

To implement HFHR dynamics \eqref{eq:HFHR-SDE} in practice, one relies
on a discrete-time algorithm that is based on a discretization
scheme of approximating the continuous-time dynamics \eqref{eq:HFHR-SDE}.

Let $h>0$, $t_k=kh$, and set
$g_k:=\nabla U(\bar q_{t_k})$.  On
$[t_k,t_{k+1})$, define the continuous interpolation
\begin{equation}\label{eq:discrete-interpolation}
\begin{cases}
  \dd\bar q_t=(\bar p_t-\alpha g_k)\,\dd t
  +\sqrt{2\alpha}\,\dd W_t^q,\\
  \dd\bar p_t=(-g_k-\gamma\bar p_t)\,\dd t
  +\sqrt{2\gamma}\,\dd W_t^p.
\end{cases}
\end{equation}
Thus the momentum appearing in the position equation is not frozen.
Writing
\begin{align}\label{eqn:abc}
  a:=e^{-\gamma h},\qquad
  b:=\frac{1-a}{\gamma},\qquad
  c:=\frac{h-b}{\gamma},
\end{align}
the grid-point chain generated by \eqref{eq:discrete-interpolation} is exactly the HFHRMC algorithm given in \eqref{eqn:HFHRMC}:
\begin{equation}\label{eq:frozen-force-scheme}
\begin{split}
  P_{k+1}&=aP_k-b\nabla U(Q_{k})+\eta_k^p,\\
  Q_{k+1}&=Q_k+bP_k-(c+\alpha h)\nabla U(Q_k)+\eta_k^q,
\end{split}
\end{equation}
such that $(\bar{p}_{kh},\bar{q}_{kh})=(P_{k},Q_{k})$ in distribution for every $k$, where $(\eta_k^q,\eta_k^p)$ are independent across steps, centered
joint Gaussian vectors with block covariances
\begin{align}
  \mathbb E\left[\eta_k^p\left(\eta_k^p\right)^\top\right]
  &=(1-a^2)\Id_d,
  \label{eq:noise-pp}\\
  \mathbb E\left[\eta_k^q\left(\eta_k^p\right)^\top\right]
  &=\frac{(1-a)^2}{\gamma}\Id_d,
  \label{eq:noise-qp}\\
  \mathbb E\left[\eta_k^q\left(\eta_k^q\right)^\top\right]
  &=\left[
    2\alpha h+\frac2\gamma
    \left(h-\frac{2(1-a)}\gamma
    +\frac{1-a^2}{2\gamma}\right)
  \right]\Id_d.
  \label{eq:noise-qq}
\end{align}
Consequently, \eqref{eq:frozen-force-scheme} is directly implementable
with one new evaluation of $\nabla U$ per iteration. This is the exact
linear frozen-force step; the corresponding frozen-gradient HFHR scheme
appears as \cite[Algorithm~2]{cortild-delplancke-oudjane-peypouquet-2025}.
At $\alpha=0$, it coincides with the left-endpoint exponential integrator
studied in \cite[Appendix~A]{bou-rabee-2026-nonacceleration}.

Throughout this section we impose the following additional regularity,
which is separate from the weighted derivative bounds used in the
continuous-time analysis.

\begin{assumption}[Smoothness for discretization]
\label{ass:discrete-smoothness}
The gradient $\nabla U$ is globally $L$-Lipschitz.
\end{assumption}

\subsection{Path-space KL discretization error}

The path-space argument below uses the SDE realization of the exact
dynamics, whereas the continuous-time decay theorem was formulated for
the semigroups constructed in Section~\ref{sec:functional}.  The next
lemma records that these are the same objects, including at the
degenerate endpoint $\alpha=0$.

\begin{lemma}[Identification of the exact SDE semigroup]
\label{lem:SDE-semigroup-identification}
Under Assumption~\ref{ass:discrete-smoothness}, let
$(\mathsf T_t^{\alpha,\gamma})_{t\geq0}$ be the backward transition
semigroup of the strong solution of \eqref{eq:HFHR-SDE}.  Then
\begin{equation}\label{eq:SDE-semigroup-identification}
  \mathsf T_t^{\alpha,\gamma}=S_t^{\alpha,\gamma}
  \qquad\text{on }L^2(\pi)
\end{equation}
for every $\alpha\geq0$ and $\gamma>0$.  Consequently, if
$\rho_0=h_0\pi$ with $h_0\in L^2(\pi)$, the exact SDE law at time
$t$ is $(P_t^{\alpha,\gamma}h_0)\pi$.
\end{lemma}

\begin{proof}
For $\alpha>0$, the identification follows from
Remark~\ref{rem:diffusion-realization}; global Lipschitz continuity of
the drift also gives uniqueness of the corresponding martingale
problem.

It remains to identify the endpoint obtained by vanishing position
diffusion.  Fix $x=(q,p)$, and construct on the same probability
space the solutions $X^\varepsilon$ and $X^0$ of
\eqref{eq:HFHR-SDE}, starting from $x$, with parameters
$\alpha=\varepsilon\in(0,1]$ and $\alpha=0$, and driven by the same
Brownian motions.  Global Lipschitz continuity and linear growth first
give
\[
  \mathbb E_x\sup_{0\leq s\leq T}\abs{X_s^\varepsilon}^2
  \leq C_{T,L,\gamma}
  \left(d+\abs x^2+\abs{\nabla U(0)}^2\right).
\]
Subtracting the equations in \eqref{eq:HFHR-SDE} for
$X^\varepsilon$ and $X^0$, corresponding respectively to
$\alpha=\varepsilon$ and $\alpha=0$, and then applying the
Cauchy--Schwarz and Burkholder--Davis--Gundy inequalities, gives
\begin{align*}
  \mathbb E_x\sup_{0\leq s\leq t}
  \abs{X_s^\varepsilon-X_s^0}^2
  \leq C_{T,L,\gamma}\int_0^t
  \mathbb E_x\sup_{0\leq r\leq s}
  \abs{X_r^\varepsilon-X_r^0}^2\,\dd s
  +C_{T,L,\gamma}\varepsilon
  \left(d+\abs x^2+\abs{\nabla U(0)}^2\right).
\end{align*}
Gr\"{o}nwall's inequality therefore gives, for every finite $T$,
\begin{equation}\label{eq:SDE-vanishing-alpha}
  \mathbb E_x\sup_{0\leq s\leq T}
  \abs{X_s^\varepsilon-X_s^0}^2
  \leq
  C_{T,L,\gamma}\varepsilon
  \left(d+\abs x^2+\abs{\nabla U(0)}^2\right).
\end{equation}
It follows from \eqref{eq:SDE-vanishing-alpha} that, for every fixed
$t\geq0$, $x\in\R^{2d}$, and
$\varphi\in C_c^\infty(\R^{2d})$, as $\varepsilon \rightarrow 0$,
\[
  \mathsf T_t^{\varepsilon,\gamma}\varphi(x)
  \longrightarrow
  \mathsf T_t^{0,\gamma}\varphi(x).
\]
The functions are uniformly bounded by $\norm\varphi_\infty$, so
the convergence also holds strongly in $L^2(\pi)$.

For $\varepsilon>0$, the first part of the proof gives
$\mathsf T_t^{\varepsilon,\gamma}=S_t^{\varepsilon,\gamma}$.
The vanishing-diffusion construction in Proposition
\ref{prop:kinetic-endpoint} gives
$S_t^{\varepsilon,\gamma}\varphi\rightharpoonup
S_t^{0,\gamma}\varphi$ in $L^2(\pi)$.  Comparing this weak limit
with the strong limit above yields
$\mathsf T_t^{0,\gamma}\varphi=S_t^{0,\gamma}\varphi$.
Since $C_c^\infty(\R^{2d})$ is dense in $L^2(\pi)$ and both
semigroups are contractions, the equality extends to all of
$L^2(\pi)$.  Taking adjoints proves the final assertion.
\end{proof}

The following entropy estimate will replace a time-uniform moment
assumption.  It uses only the Gaussian momentum law and global
smoothness of the potential.

\begin{lemma}[Entropy control of the kinetic moments]
\label{lem:entropy-moments}
Suppose $\nabla U$ is globally $L$-Lipschitz.  There is a universal
constant $C_0>0$ such that every probability measure
$\nu\ll\pi$ satisfies
\begin{equation}\label{eq:entropy-moment-bound}
  \int_{\R^{2d}}
  \left(\abs p^2+\abs{\nabla U(q)}^2\right)\,\dd\nu
  \leq
  C_0(1+L)\left[d+\mathrm{KL}(\nu\Vert\pi)\right].
\end{equation}
\end{lemma}

\begin{proof}
The Gaussian factor $\kappa$ satisfies
\[
  \log\int_{\R^d}e^{\abs p^2/4}\,\dd\kappa(p)
  =\frac d2\log2.
\]
We next record the analogous estimate for the force.  We may take
$L>0$, enlarging the Lipschitz constant if necessary.
Since $\nabla U$ is $L$-Lipschitz, by the descent lemma
\cite[Lemma~1.2.3]{Nesterov2013}, for every
$q,v\in\R^d$ and $\lambda\geq0$,
\[
  U(q-\lambda v)
  \leq U(q)-\lambda v\cdot\nabla U(q)
  +\frac{L\lambda^2}{2}\abs v^2.
\]
After exponentiating, integrating, and making the translation
$q\mapsto q-\lambda v$, we obtain
\[
  \int_{\mathbb{R}^{d}} e^{\lambda v\cdot\nabla U}\,\dd\mu
  \leq e^{\lambda^2L\abs v^2/2}.
\]
Gaussian randomization, with an independent vector
$G\sim\mathcal N(0,\Id_d)$, now yields
\begin{equation}\label{eq:force-exponential-moment}
  \int_{\R^d}
  \exp\!\left(\frac{\abs{\nabla U(q)}^2}{4L}\right)\,\dd\mu(q)
  \leq
  \mathbb E e^{\abs G^2/4}
  =2^{d/2}.
\end{equation}
The variational formula for relative entropy
\cite[Proposition~1.4.2]{dupuis-ellis-1997}, applied to
$\theta F$ and extended to nonnegative $F$ by truncation, gives,
for every $\theta>0$,
\[
  \int_{\mathbb{R}^{2d}} F\,\dd\nu
  \leq\frac1\theta\left[
    \mathrm{KL}(\nu\Vert\pi)
    +\log\int_{\mathbb{R}^{2d}} e^{\theta F}\,\dd\pi
  \right].
\]
Applying this inequality with
\[
  (F,\theta)=(\abs p^2,1/4)
  \quad\text{and}\quad
  (F,\theta)=(\abs{\nabla U(q)}^2,1/(4L)),
\]
respectively, and using the preceding exponential-moment bounds proves
\eqref{eq:entropy-moment-bound}.
\end{proof}

Let $\mathbb P_T$ be the path law on $[0,T]$ of the HFHR SDE
\eqref{eq:HFHR-SDE} and let $\bar{\mathbb P}_{T,h}$ be the path law
of \eqref{eq:discrete-interpolation}, with the same initial law and
$T=Nh$.  Denote their time-$T$ laws by $\rho_T$ and
$\bar\rho_N$, respectively.

\begin{theorem}[Girsanov discretization bound]
\label{thm:girsanov-discretization}
Suppose Assumption~\ref{ass:discrete-smoothness} holds and
$\rho_0=h_0\pi$, where $h_0\in L^2(\pi)$.  Define
\begin{equation}\label{eq:initial-Renyi-two}
  \mathcal R_0
  :=\log\!\left(1+\chi^2(\rho_0\Vert\pi)\right),
  \qquad
  G_0:=d+\mathcal R_0,
\end{equation}
and set
\[
  \Lambda_{\alpha,\gamma}
  :=\alpha+\frac1\gamma,
  \qquad
  \Xi_{\alpha,L}:=(1+\alpha^2)(1+L).
\]
There is a universal constant $C_1>0$ such that, for every
$\alpha\geq0$, $\gamma>0$, and $T=Nh$,
\begin{equation}\label{eq:path-KL-bound}
  \mathrm{KL}\left(\bar{\mathbb P}_{T,h}\Vert\mathbb P_T\right)
  \leq C_1L^2\Lambda_{\alpha,\gamma}T
  \left[
    \alpha d h+\Xi_{\alpha,L}G_0h^2
  \right]\cdot
  \exp\!\left\{
    C_1L^2\Lambda_{\alpha,\gamma}
    \Xi_{\alpha,L}Th^2
  \right\}.
\end{equation}
\end{theorem}

\begin{proof}
For $R>0$, let $\bar X^R=(\bar q^R,\bar p^R)$ follow the frozen-force
interpolation until
\[
  \tau_R:=\inf\left\{t\geq0:\abs{\bar X_t^R}\geq R\right\},
\]
and follow the exact HFHR drift after $\tau_R$, with the same
diffusion coefficients throughout.  Thus, with
$g_k^R=\nabla U(\bar q_{t_k}^R)$, the force used in both drift
coordinates is
\[
  F_t^R
  =\mathbf 1_{\{t<\tau_R\}}g_k^R
  +\mathbf 1_{\{t\geq\tau_R\}}\nabla U(\bar q_t^R),
  \qquad t\in[t_k,t_{k+1}).
\]
Write $\bar{\mathbb P}_{t,h}^R$ for its path law on $[0,t]$,
$\mathbb E^R$ for expectation on the probability space carrying
$\bar X^R$ and its driving Brownian motions,
$\bar\rho_t^R$ for its time-$t$ law, and
\[
  \mathscr D_R(t)
  :=\mathrm{KL}\left(\bar{\mathbb P}_{t,h}^R\Vert\mathbb P_t\right).
\]
Before $\tau_R$, the drift discrepancy is
$(\alpha\Delta_t^R,\Delta_t^R)$, where
\[
  \Delta_t^R
  :=\nabla U\left(\bar q_t^R\right)
  -\nabla U\left(\bar q_{\lfloor t/h\rfloor h}^R\right),
\]
and it is zero afterwards.  This discrepancy is bounded on
$[0,T\wedge\tau_R]$, so the usual Novikov condition holds.  For
$\alpha>0$, its squared norm in the inverse diffusion covariance is
\[
  \frac\alpha2\abs{\Delta_t^R}^2
  +\frac1{2\gamma}\abs{\Delta_t^R}^2.
\]
The entropy identity in Girsanov's theorem
\cite[Theorem~15]{zhang-chewi-li-balasubramanian-erdogdu-2023}
therefore gives
\begin{equation}\label{eq:localized-path-KL}
  \mathscr D_R(t)
  =\frac14\Lambda_{\alpha,\gamma}
  \mathbb E^R\int_0^{t\wedge\tau_R}
  \abs{\Delta_s^R}^2\,\dd s.
\end{equation}
When $\alpha=0$, the position drifts of the two processes are both
the current momentum.  The discrepancy lies entirely in the noisy
momentum coordinate, so the same identity holds with
$\Lambda_{0,\gamma}=\gamma^{-1}$.  This also makes explicit why no
inverse power of $\alpha$ occurs.

Restriction of paths shows that $\mathscr D_R(s)\leq\mathscr D_R(t)$
when $s\leq t$.  Data processing inequality and the entropy variational
inequality imply
\begin{align}
  \mathrm{KL}\left(\bar\rho_t^R\Vert\pi\right)
  \leq
  2\mathrm{KL}\left(\bar\rho_t^R\Vert\rho_t\right)
  +\log\left(1+\chi^2(\rho_t\Vert\pi)\right)
\leq2\mathscr D_R(t)+\mathcal R_0.
  \label{eq:approximate-entropy-closure}
\end{align}
Indeed, with $r_t=\dd\rho_t/\dd\pi$, the entropy variational
inequality on $\R^{2d}$ gives
\[
  \int_{\R^{2d}}\log r_t(q,p)\,
  \bar\rho_t^R(\dd q\,\dd p)
  \leq\mathrm{KL}\left(\bar\rho_t^R\Vert\rho_t\right)
  +\log\int_{\R^{2d}}r_t(q,p)\,
  \rho_t(\dd q\,\dd p),
\]
and the last integral is
$\int_{\R^{2d}}r_t^2\,\dd\pi=1+\chi^2(\rho_t\Vert\pi)$.  The exact
SDE law is the forward semigroup law by Lemma
\ref{lem:SDE-semigroup-identification}.  Its $L^2(\pi)$
contractivity therefore makes this $\chi^2$ divergence no larger
than its initial value.

Lemma~\ref{lem:entropy-moments} and
\eqref{eq:approximate-entropy-closure} therefore give
\begin{equation}\label{eq:self-bounded-moments}
  \mathbb E^R\left[
    \abs{\bar p_t^R}^2+\abs{\nabla U(\bar q_t^R)}^2
  \right]
  \leq C_1(1+L)\left[G_0+\mathscr D_R(t)\right].
\end{equation}
For $t\in[t_k,t_{k+1}]$, put $s=t-t_k$.  The integral form of
the frozen-force equation on $\{t<\tau_R\}$, Cauchy--Schwarz inequality, the
bound
\[
  \mathbb E^R\!\left[
  \mathbf 1_{\{t<\tau_R\}}
  \abs{W_t^q-W_{t_k}^q}^2\right]
  \leq
  \mathbb E^R\abs{W_t^q-W_{t_k}^q}^2=d\,s,
\]
and \eqref{eq:self-bounded-moments} imply
\begin{equation}\label{eq:q-increment}
  \mathbb E^R\left[
  \mathbf 1_{\{t<\tau_R\}}
  \abs{\bar q_t^R-\bar q_{t_k}^R}^2\right]
  \leq C_1\left[
    \alpha d s
    +\Xi_{\alpha,L}
    \left(G_0+\mathscr D_R(t)\right)s^2
  \right].
\end{equation}
Here moments at times before $t$ are controlled by
$\mathscr D_R(t)$ because $\mathscr D_R$ is nondecreasing.
By the $L$-Lipschitz continuity of $\nabla U$, for any
$u\in[0,T]$ we may integrate
$\abs{\Delta_t^R}^2\leq
L^2\abs{\bar q_t^R-\bar q_{t_k}^R}^2$ on $\{t<\tau_R\}$ over
the complete time steps in $[0,u]$ and the last, possibly incomplete,
step.  Since $0\leq t-\lfloor t/h\rfloor h\leq h$,
\eqref{eq:localized-path-KL} and \eqref{eq:q-increment} give
\begin{equation}\label{eq:localized-entropy-integral}
  \mathscr D_R(u)
  \leq C_1L^2\Lambda_{\alpha,\gamma}
  \left[
    \alpha d u h+\Xi_{\alpha,L}G_0uh^2
    +\Xi_{\alpha,L}h^2\int_0^u\mathscr D_R(t)\,\dd t
  \right].
\end{equation}
Applying Gr\"{o}nwall's inequality to
\eqref{eq:localized-entropy-integral} and evaluating at $u=T$
gives the right-hand side of \eqref{eq:path-KL-bound}, uniformly in $R$.

It remains to remove the localization.  Couple $\bar X^R$ with the
original frozen-force interpolation $\bar X$ by using the same
initial state and Brownian motions.  Pathwise uniqueness makes them
identical up to $\tau_R$.  Global Lipschitz continuity gives
nonexplosion of the exact SDE and of the linear frozen-force evolution
on each time step.  Thus $\mathbb P(\tau_R\leq T)\to0$, which implies
that $\bar{\mathbb P}_{T,h}^R$ converges to
$\bar{\mathbb P}_{T,h}$ on path space.  Lower semicontinuity of
relative entropy gives
\[
  \mathrm{KL}\left(\bar{\mathbb P}_{T,h}\Vert\mathbb P_T\right)
  \leq\liminf_{R\to\infty}\mathscr D_R(T).
\]
The uniform bound on $\mathscr D_R(T)$ obtained from Gr\"{o}nwall's
inequality therefore proves \eqref{eq:path-KL-bound}.  This self-bounding
step is analogous to the entropy closure used in the proof of
\cite[Theorem~1]{lehec-2025}; the higher-order R\'enyi Girsanov approach of
\cite[Theorem~9 and Corollary~16]{zhang-chewi-li-balasubramanian-erdogdu-2023} provides a related
route under a Poincar\'e inequality.
\end{proof}

\begin{remark}[Scaling of the discretization error]
\label{rem:discretization-error-scaling}
If
$C_1L^2\Lambda_{\alpha,\gamma}\Xi_{\alpha,L}Th^2\leq1$,
the exponential factor in \eqref{eq:path-KL-bound} is bounded by a
universal constant, and hence
\[
  \mathrm{KL}\left(\bar{\mathbb P}_{T,h}\Vert\mathbb P_T\right)
  \lesssim
  L^2\Lambda_{\alpha,\gamma}
  \left(\alpha dTh+\Xi_{\alpha,L}G_0Th^2\right).
\]
Thus, for $\alpha>0$, the leading path-space KL error has the usual
$Th$ scaling and gives a $\sqrt{Th}$ total-variation error.  At the
kinetic endpoint $\alpha=0$, the $Th$ term vanishes and the bound
improves to $Th^2$ in KL, or $\sqrt T\,h$ in total variation.
For TV accuracy $\varepsilon\leq1$, controlling the $Th^2$ error
term also bounds the exponential factor, as shown in the proof of
Corollary~\ref{cor:iteration-complexity}.
\end{remark}

\subsection{Non-asymptotic convergence bound and iteration complexity}

The next result combines the path-space estimate with the continuous
mixing theorem.  It controls the joint law and therefore, by projection,
also the position marginal sampled by the algorithm.

\begin{corollary}[TV error bound of HFHRMC]
\label{cor:discrete-TV}
Suppose Assumptions~\ref{ass:PI}--\ref{ass:compact} and
\ref{ass:discrete-smoothness} hold, and let
$\rho_0=h_0\pi$ with $h_0\in L^2(\pi)$.  Then
\begin{equation}\label{eq:discrete-TV-bound}
\begin{split}
  \norm{\bar\rho_N-\pi}_{\mathrm{TV}}
  \le&\frac C2e^{-\nu_{\alpha,\gamma}T}
  \sqrt{\chi^2(\rho_0\Vert\pi)}\\
  &+C_1L\sqrt{
    \Lambda_{\alpha,\gamma}T
    \left[
      \alpha d h+\Xi_{\alpha,L}G_0h^2
    \right]}
    \exp\!\left\{
      C_1L^2\Lambda_{\alpha,\gamma}
      \Xi_{\alpha,L}Th^2
    \right\},
  \qquad T=Nh,
\end{split}
\end{equation}
where $C$ is as in Theorem~\ref{thm:main} and $C_1$ is universal.
\end{corollary}

\begin{proof}
By the triangle inequality,
\[
  \norm{\bar\rho_N-\pi}_{\mathrm{TV}}
  \leq\norm{\bar\rho_N-\rho_T}_{\mathrm{TV}}
  +\norm{\rho_T-\pi}_{\mathrm{TV}}.
\]
Pinsker's inequality \cite[Lemma~2.5]{tsybakov-2009}, together with
Theorem~\ref{thm:girsanov-discretization} bound the first term.  The
second term is controlled by Corollary~\ref{cor:chi-square-TV}.
\end{proof}

Set
\begin{equation}\label{eq:rate-proxy}
  A:=\sqrt m+R,\qquad
  r_0(\gamma):=\frac{m\gamma}{(A+\gamma)^2},\qquad
  r(\alpha,\gamma):=\min\{\gamma,m\alpha+r_0(\gamma)\}.
\end{equation}
Theorem~\ref{thm:main} supplies a universal $c_0>0$ such that
$\nu_{\alpha,\gamma}\geq c_0r(\alpha,\gamma)$.  Write
\begin{align*}
  W_0&:=\sqrt{\chi^2(\rho_0\Vert\pi)},
  &\mathfrak M_\chi
  &:=(1+L)\left[d+\log(1+W_0^2)\right],\\
  \ell_\varepsilon&:=\log\frac{CW_0}{\varepsilon},
  &T_\varepsilon&:=\frac{\ell_\varepsilon}
  {c_0r(\alpha,\gamma)}.
\end{align*}
The resulting iteration complexity is recorded as a corollary.

\begin{corollary}[Iteration complexity of HFHRMC]
\label{cor:iteration-complexity}
Suppose the assumptions of Corollary~\ref{cor:discrete-TV} hold and
$0<\varepsilon\leq1$ with $\varepsilon<CW_0$.
Set $h=T_\varepsilon/N$.
There is a universal constant $C_2>0$ such that
$\norm{\bar\rho_N-\pi}_{\mathrm{TV}}\leq\varepsilon$,
provided that
\begin{equation}\label{eq:iteration-complexity-sufficient}
  N\geq C_2\mathcal N_\varepsilon(\alpha,\gamma),
  \end{equation}
where
\begin{equation}\label{eq:iteration-complexity}
  \mathcal N_\varepsilon(\alpha,\gamma)
  :=
  \max\left\{
    \frac{L^2d\,\alpha(\alpha+\gamma^{-1})
      T_\varepsilon^2}{\varepsilon^2},
    \frac{L\sqrt{\mathfrak M_\chi
      (\alpha+\gamma^{-1})(1+\alpha^2)}
      T_\varepsilon^{3/2}}{\varepsilon}
  \right\}.
\end{equation}
When $\alpha=0$, the first term vanishes.  This reflects the fact
that the position increment is then of order $h$, whereas for
$\alpha>0$ its Brownian component is of order $\sqrt{\alpha h}$.
\end{corollary}

\begin{proof}
Since $\nu_{\alpha,\gamma}\geq c_0r(\alpha,\gamma)$, the definition
of $T_\varepsilon$ makes the mixing term in
\eqref{eq:discrete-TV-bound} at most $\varepsilon/2$.
Since $\mathfrak M_\chi=(1+L)G_0$ and $G_0\geq d\geq1$, the second term in
\eqref{eq:iteration-complexity}, together with $h=T_\varepsilon/N$,
gives
\begin{align*}
  L^2\Lambda_{\alpha,\gamma}\Xi_{\alpha,L}
  T_\varepsilon h^2
  &=\frac{L^2\Lambda_{\alpha,\gamma}\Xi_{\alpha,L}
  T_\varepsilon^3}{N^2}
  \leq\frac{\varepsilon^2}{C_2^2G_0}
  \leq C_2^{-2}.
\end{align*}
Thus the exponential factor in \eqref{eq:discrete-TV-bound} is
bounded by a universal constant for $C_2\geq1$.
The two terms in \eqref{eq:iteration-complexity} also give
\begin{align*}
  L^2\Lambda_{\alpha,\gamma}\alpha dT_\varepsilon h
  &\leq\frac{\varepsilon^2}{C_2},\\
  L^2\Lambda_{\alpha,\gamma}\Xi_{\alpha,L}G_0
  T_\varepsilon h^2
  &\leq\frac{\varepsilon^2}{C_2^2}.
\end{align*}
Inserting these bounds into \eqref{eq:discrete-TV-bound} and choosing
$C_2$ sufficiently large makes the discretization error at most
$\varepsilon/2$, proving \eqref{eq:iteration-complexity-sufficient}.
\end{proof}

\begin{remark}[2-Wasserstein error bound and iteration complexity of HFHRMC]
Under the additional assumption
that $\pi$ satisfies a log-Sobolev inequality with constant $\rho>0$,
Talagrand's inequality \cite[Theorem~22.15(i)]{villani-2009}
(see Remark~\ref{remark:W:2}) can be combined with the same entropy
closure used in \eqref{eq:approximate-entropy-closure}.  Indeed, data
processing, the entropy variational inequality, and
Corollary~\ref{cor:chi-square-TV} give
\begin{align}
  \mathrm{KL}(\bar\rho_N\Vert\pi)
  &\leq
  2\mathrm{KL}(\bar\rho_N\Vert\rho_T)
  +\log\!\left(1+\chi^2(\rho_T\Vert\pi)\right)
  \notag\\
  &\leq
  2\mathrm{KL}(\bar{\mathbb P}_{T,h}\Vert\mathbb P_T)
  +C^2e^{-2\nu_{\alpha,\gamma}T}
  \chi^2(\rho_0\Vert\pi).
  \label{eq:discrete-entropy-closure}
\end{align}
Talagrand's inequality, \eqref{eq:discrete-entropy-closure}, and
Theorem~\ref{thm:girsanov-discretization} therefore yield, after
enlarging the universal constants if necessary,
\begin{equation}\label{eq:discrete-W2-bound}
\begin{split}
  \mathcal{W}_{2}(\bar\rho_N,\pi)
  &\leq\sqrt{\frac{2}{\rho}}Ce^{-\nu_{\alpha,\gamma}T}
  \sqrt{\chi^2(\rho_0\Vert\pi)}\\
  &\qquad+\sqrt{\frac{2}{\rho}}C_1L\sqrt{
    \Lambda_{\alpha,\gamma}T
    \left[
      \alpha d h+\Xi_{\alpha,L}G_0h^2
    \right]}
    \exp\!\left\{
      C_1L^2\Lambda_{\alpha,\gamma}
      \Xi_{\alpha,L}Th^2
    \right\},
\end{split}
\end{equation}
where $T=Nh$.  To state the resulting iteration complexity, set
\[
  \varepsilon_\rho:=\min\left\{1,
  \frac{\sqrt\rho\,\varepsilon}{2\sqrt2}\right\}
\]
and suppose $0<\varepsilon_\rho<CW_0$.  Taking
$T=T_{\varepsilon_\rho}$ and $h=T_{\varepsilon_\rho}/N$ gives
$\mathcal{W}_{2}(\bar\rho_N,\pi)\leq\varepsilon$, provided that
$N\geq\widetilde{C}_2\widetilde{\mathcal N}_\varepsilon(\alpha,\gamma)$,
where $\widetilde{C}_2$ is a universal constant and
\begin{align}
  \widetilde{\mathcal N}_\varepsilon(\alpha,\gamma)
  &:=
  \mathcal{N}_{\varepsilon_\rho}(\alpha,\gamma)
  \nonumber\\
  &=
  \max\left\{
    \frac{L^2d\,\alpha(\alpha+\gamma^{-1})
      T_{\varepsilon_\rho}^2}{\varepsilon_\rho^2},
    \frac{L\sqrt{\mathfrak M_\chi
      (\alpha+\gamma^{-1})(1+\alpha^2)}
      T_{\varepsilon_\rho}^{3/2}}{\varepsilon_\rho}
  \right\}.
\end{align}
\end{remark}

\subsection{Optimal choices of
\texorpdfstring{$\alpha$ and $\gamma$}{alpha and gamma}}

We finally optimize the explicit proxy
$\mathcal N_\varepsilon$ in \eqref{eq:iteration-complexity}, rather
than the continuous-time rate
alone.  This distinction is important because increasing $\alpha$
accelerates continuous mixing but also increases the local
discretization error.
Define the saturation threshold
\begin{equation}\label{eq:alpha-saturation}
  \alpha_{\rm sat}(\gamma)
  :=\frac{\gamma-r_0(\gamma)}m>0.
\end{equation}
For $\alpha\geq\alpha_{\rm sat}(\gamma)$, the rate proxy in
\eqref{eq:rate-proxy} has already saturated at $\gamma$, while both
$\alpha(\alpha+\gamma^{-1})$ and
$(\alpha+\gamma^{-1})(1+\alpha^2)$ increase.  Hence no minimizer of
the complexity proxy needs to lie above
$\alpha_{\rm sat}(\gamma)$.

\begin{proposition}[High-accuracy parameter tuning]
\label{prop:optimal-alpha}
Assume the setting of Corollary~\ref{cor:iteration-complexity}, in
particular $0<\varepsilon\leq1$ and $\varepsilon<CW_0$. For fixed
$\gamma$, let $\alpha_\varepsilon^\star(\gamma)$
minimize $\mathcal N_\varepsilon(\alpha,\gamma)$ over
$\alpha\geq0$. 
With all problem parameters other than
$\varepsilon$ fixed, as
$\varepsilon \rightarrow 0$,
\begin{equation}\label{eq:our-optimal-alpha}
  \alpha_\varepsilon^\star(\gamma)
  \asymp
  \frac{\varepsilon
  \sqrt{\mathfrak M_\chi\,\gamma r_0(\gamma)}}
  {Ld\sqrt{\ell_\varepsilon}},
  \qquad
  \alpha_\varepsilon^\star(\gamma)
  <\alpha_{\rm sat}(\gamma).
\end{equation}
If $\gamma$ is optimized as well, let
$(\alpha_\varepsilon^\star,\gamma_\varepsilon^\star)$ denote a
joint minimizer.  Then
\begin{equation}\label{eq:our-optimal-gamma}
  \gamma_\varepsilon^\star\longrightarrow2A
  =2(\sqrt m+R),
  \qquad
  \alpha_\varepsilon^\star
  \asymp
  \frac{2\varepsilon\sqrt{\mathfrak M_\chi m}}
  {3Ld\sqrt{\ell_\varepsilon}}.
\end{equation}
\end{proposition}

\begin{proof}
Below saturation,
$r(\alpha,\gamma)=r_0(\gamma)+m\alpha$.  Denote the two terms in
\eqref{eq:iteration-complexity} by $B_1,B_2$, respectively.
For fixed $\gamma$, these terms, after omitting fixed universal
constants, are
\begin{align*}
  B_1(\alpha)
  =\frac{L^2d\,\alpha(\alpha+\gamma^{-1})
  \ell_\varepsilon^2}
  {\varepsilon^2(r_0+m\alpha)^2},
  \qquad
  B_2(\alpha)
  =\frac{L\sqrt{\mathfrak M_\chi
  (\alpha+\gamma^{-1})(1+\alpha^2)}
  \ell_\varepsilon^{3/2}}
  {\varepsilon(r_0+m\alpha)^{3/2}}.
\end{align*}
We first justify that it is enough to optimize these expressions near
zero.  At $\alpha=0$, the full proxy is of order
$\ell_\varepsilon^{3/2}/\varepsilon$.  On the other hand, for every
fixed $\delta\in(0,\alpha_{\rm sat})$, compactness of
$[\delta,\alpha_{\rm sat}]$ gives
\[
  \inf_{\delta\leq\alpha\leq\alpha_{\rm sat}}B_1(\alpha)
  \geq
  c_\delta\frac{\ell_\varepsilon^2}{\varepsilon^2}.
\]
Together with the exclusion of $\alpha>\alpha_{\rm sat}$ preceding
the proposition, this proves
$\alpha_\varepsilon^\star(\gamma)\to0$.

The logarithmic derivatives are
\begin{align*}
  \frac{\dd}{\dd\alpha}\log B_1(\alpha)
  &=\frac1\alpha+\frac1{\alpha+\gamma^{-1}}
    -\frac{2m}{r_0+m\alpha},\\
  \frac{\dd}{\dd\alpha}\log B_2(\alpha)
  &=\frac1{2(\alpha+\gamma^{-1})}
    +\frac\alpha{1+\alpha^2}
    -\frac{3m}{2(r_0+m\alpha)}.
\end{align*}
Hence $B_1$ is strictly increasing on a fixed neighborhood of zero,
whereas $B_2$ is strictly decreasing there because
\[
  \left.\frac{\dd}{\dd\alpha}\log B_2(\alpha)\right|_{\alpha=0}
  =\frac\gamma2-\frac{3m}{2r_0(\gamma)}<0.
\]
At zero one has $B_1<B_2$, whereas
at the fixed right endpoint of this neighborhood one has $B_1>B_2$
for all sufficiently small $\varepsilon$.  There is therefore a
unique crossover.  Before it the maximum equals the decreasing
function $B_2$, and after it the maximum equals the increasing
function $B_1$.  The crossover is consequently the global
high-accuracy minimizer.

The crossover tends to zero, so replacing
$r_0+m\alpha$ by $r_0$,
$\alpha+\gamma^{-1}$ by $\gamma^{-1}$, and
$1+\alpha^2$ by one in $B_1=B_2$ gives
\[
  \alpha_\varepsilon^\star(\gamma)
  \asymp
  \frac{\varepsilon
  \sqrt{\mathfrak M_\chi\,\gamma r_0(\gamma)}}
  {Ld\sqrt{\ell_\varepsilon}}.
\]
Since this quantity tends to zero, it is strictly below
$\alpha_{\rm sat}(\gamma)$ for all sufficiently small
$\varepsilon$.  This proves \eqref{eq:our-optimal-alpha}.
At this optimized scale, the leading $\gamma$-dependent factor is
\[
  \gamma^{-1/2}(r_0(\gamma))^{-3/2}
  =\frac{(A+\gamma)^3}{m^{3/2}\gamma^2}.
\]
For completeness, joint minimizers cannot escape to zero or infinity.
Indeed, the benchmark $(\alpha,\gamma)=(0,2A)$ has complexity of
order $\ell_\varepsilon^{3/2}/\varepsilon$.  If $\gamma \rightarrow 0$,
then $r\leq\gamma$ and
$\Lambda_{\alpha,\gamma}\geq\gamma^{-1}$, so $B_2$ is at
least a fixed multiple of
$\ell_\varepsilon^{3/2}/(\varepsilon\gamma^2)$.  If
$\gamma \rightarrow \infty$, let $x=\alpha+\gamma^{-1}$.  Since
$r_0(\gamma)\leq m/\gamma$, one has $r\leq mx$.  The terms $B_1$ and
$B_2$ are consequently bounded below, up to fixed factors, by
$\frac{\alpha}{x}
  \frac{\ell_\varepsilon^2}{\varepsilon^2}$
and  $\frac1x\frac{\ell_\varepsilon^{3/2}}{\varepsilon}$,
respectively.  Matching the benchmark forces
$\alpha/x=\mathcal{O}(\varepsilon/\sqrt{\ell_\varepsilon})$, hence
$x\sim\gamma^{-1}$; the lower bound for $B_2$ then diverges relative to
the benchmark.  Thus $\gamma_\varepsilon^\star$ remains in a compact
subset of $(0,\infty)$.

On compact $\gamma$-sets the preceding small-$\alpha$ expansion is
uniform.  The leading profile is therefore
$\frac{(A+\gamma)^3}{m^{3/2}\gamma^2}$, which
tends to infinity at both endpoints, and its logarithmic derivative
is $3/(A+\gamma)-2/\gamma$, whose unique zero is
$\gamma=2A$.  Hence
$\gamma_\varepsilon^\star\to2A$, proving
the first term in \eqref{eq:our-optimal-gamma}.  Finally,
$r_0(2A)=2m/(9A)$; substituting this value in
\eqref{eq:our-optimal-alpha} proves
the second term in \eqref{eq:our-optimal-gamma}.
\end{proof}

\begin{remark}[Optimized iteration complexity]
\label{rem:complexity-comparison-ULD}
Recall that $A=\sqrt m+R$.  For fixed $\gamma$, substitution of
\eqref{eq:our-optimal-alpha} into the complexity proxy gives
\begin{equation}\label{eq:optimized-complexity-fixed-gamma}
  \mathcal N_\varepsilon
  (\alpha_\varepsilon^\star(\gamma),\gamma)
  =
  (1+o(1))\mathcal N_\varepsilon(0,\gamma)
  =
  \widetilde{\mathcal O}\left(
    \frac{L\sqrt{\mathfrak M_\chi}(A+\gamma)^3}
    {\varepsilon m^{3/2}\gamma^2}
  \right),
  \qquad \varepsilon \rightarrow 0.
\end{equation}
Here $\widetilde{\mathcal O}$ suppresses powers of
$\ell_\varepsilon$ and other logarithmic factors.  
For $\alpha=0$, one has
$\Lambda_{0,\gamma}=\gamma^{-1}$,
$\Xi_{0,L}=1+L$, and
$r(0,\gamma)=r_0(\gamma)$.  Therefore,
\begin{equation}\label{eq:ULD-complexity-fixed-gamma}
  \mathcal N_\varepsilon(0,\gamma)
  =
  \mathcal O\left(
    \frac{L\sqrt{\mathfrak M_\chi}(A+\gamma)^3
    \ell_\varepsilon^{3/2}}
    {\varepsilon m^{3/2}\gamma^2}
  \right).
\end{equation}
For fixed $\gamma$, \eqref{eq:optimized-complexity-fixed-gamma} gives
the same leading high-accuracy dependence on $\varepsilon$, $d$, and $m$
for the optimized HFHR and ULD proxies. At sufficiently small positive
$\varepsilon$, the positive minimizer gives a lower-order reduction of
$\mathcal N_\varepsilon$.

If $L\geq1$ and
$\log(1+W_0^2)=\widetilde{\mathcal O}(d)$, then
$\mathfrak M_\chi=\widetilde{\mathcal O}(Ld)$
and it follows from \eqref{eq:optimized-complexity-fixed-gamma} that 
\begin{equation}
  \mathcal N_\varepsilon
  (\alpha_\varepsilon^\star(\gamma),\gamma)
  =
  \widetilde{\mathcal O}\left(
    \frac{L^{3/2}\sqrt{d}(A+\gamma)^3}
    {\varepsilon m^{3/2}\gamma^2}
  \right),
  \qquad \varepsilon \rightarrow 0.
\end{equation}

The ULD contribution has the parameter regimes
\[
  r_0(\gamma)\asymp
  \begin{cases}
    m\gamma/A^2, & \gamma\ll A,\\
    m/A,         & \gamma\asymp A,\\
    m/\gamma,    & \gamma\gg A.
  \end{cases}
\]
After optimizing also over $\gamma$, Proposition
\ref{prop:optimal-alpha} gives
\begin{equation}\label{eq:optimized-complexity-summary}
   \mathcal N_\varepsilon
  (\alpha_\varepsilon^\star,\gamma_\varepsilon^\star)
  =
  \widetilde{\mathcal O}\left(
    \frac{L\sqrt{\mathfrak M_\chi}\,A}
    {\varepsilon m^{3/2}}
  \right)  
  =
  \widetilde{\mathcal O}\left(
    \frac{L^{3/2}\sqrt d\,(\sqrt m+R)}
    {\varepsilon m^{3/2}}
  \right).
\end{equation}
In particular, the convex case $R=0$ gives
\[
  \mathcal N_\varepsilon
  (\alpha_\varepsilon^\star,\gamma_\varepsilon^\star)
  =\widetilde{\mathcal O}\left(
    \frac{L^{3/2}\sqrt d}{\varepsilon m}
  \right),
\]
whereas in the general case
$R=M+M^{3/4}d^{1/4}$ in
\eqref{eq:optimized-complexity-summary}.
\end{remark}

\begin{remark}[Complexity comparison with ULD]
For the ULD endpoint, the only $\gamma$-dependent factor in
\eqref{eq:ULD-complexity-fixed-gamma} is
$(A+\gamma)^3/\gamma^2$.  Its unique minimizer is
\begin{equation}\label{eq:ULD-optimal-gamma}
  \gamma_{\rm ULD}^\star=2A=2(\sqrt m+R).
\end{equation}
Consequently,
\begin{align}
  \mathcal N_\varepsilon(0,\gamma_{\rm ULD}^\star)
  =
  \mathcal O\left(
    \frac{L\sqrt{\mathfrak M_\chi}A\ell_\varepsilon^{3/2}}
    {\varepsilon m^{3/2}}
  \right)
=
  \widetilde{\mathcal O}\left(
    \frac{L\sqrt{\mathfrak M_\chi}A}
    {\varepsilon m^{3/2}}
  \right)
  \qquad(\varepsilon \rightarrow 0).
\end{align}
The jointly optimized HFHR proxy and the optimized ULD proxy have the
same leading order at high accuracy. A positive $\alpha$ can reduce the
proxy at finite accuracy.
\end{remark}

\begin{remark}[Comparison of ULD with \cite{dalalyan-riou-durand-2020} in 2-Wasserstein distance]
\label{rem:scope-discrete-comparisons}
For this comparison, suppose $U$ is $m$-strongly convex and
$L$-smooth, with $0<m\leq1$ and $L\geq1$.  Take the common initial
law $\rho_0=\nu_0\otimes\kappa$, where
$\dd\nu_0/\dd\mu\in L^2(\mu)$ and
$\log(1+W_0^2)=\widetilde{\mathcal O}(d)$, so that
$\mathfrak M_\chi=\widetilde{\mathcal O}(Ld)$.
The bound of
\cite{dalalyan-riou-durand-2020} is stated in $\mathcal W_2$, and its
friction restriction is already optimized at the boundary of its
admissible range, that is $\gamma^*=\sqrt{m+L}$ such that 
they obtain the iteration complexity as $\varepsilon \rightarrow  0$
\cite[Theorem~2 and the discussion following it]{dalalyan-riou-durand-2020}:
\begin{align}\label{compare:stx:1}
\mathcal N_{\varepsilon}^{\rm DR}(0,\gamma^*)=\widetilde{\mathcal O}\left(\frac{L^{3/2}\sqrt{d}}
    {\varepsilon m^{2}}\right).
\end{align} 
In our setting, applying the
Bakry--\'Emery criterion \cite[Theorem~3.1]{menz-schlichting-2014}
to the joint potential gives $\pi=\mu\otimes\kappa$ a log-Sobolev constant
$\rho=\min\{m,1\}=m$.  Since $A=\sqrt m$, as
$\varepsilon \rightarrow 0$ our $\mathcal{W}_{2}$ iteration complexity satisfies:
\begin{align}\label{compare:stx:2}
  \widetilde{\mathcal N}_\varepsilon(0,\gamma_{\rm ULD}^\star)
=
  \widetilde{\mathcal O}\left(
    \frac{L\sqrt{\mathfrak M_\chi}A}
    {\sqrt{\rho}\varepsilon m^{3/2}}
  \right)
  =\widetilde{\mathcal O}\left(
    \frac{L^{3/2}\sqrt{d}}
    {\varepsilon m^{3/2}}
  \right).
\end{align}
For these warm starts and parameter ranges, the upper bound in
\eqref{compare:stx:2} improves the $m$-dependence in
\eqref{compare:stx:1} by a factor $\sqrt m$.

Theorem~1 of \cite{cheng2018underdamped} allows deterministic
initialization at $(q_0,0)$.  Let $q_\star$ be the minimizer of $U$
and let $D^2=\abs{q_0-q_\star}^2$.  Their sufficient iteration
count in the high-accuracy regime has order
\[
  \mathcal N_\varepsilon^{\rm Cheng}
  =\widetilde{\mathcal O}\left(
    \frac{(L/m)^2}{\varepsilon}\sqrt{d/m+D^2}
  \right).
\]
If $D^2=\mathcal O(d/m)$, this becomes
$\widetilde{\mathcal O}(L^2\sqrt d/(\varepsilon m^{5/2}))$.
Comparison with \eqref{compare:stx:2} therefore concerns the
parameter dependence under these respective initialization
conditions: their theorem permits a point mass, whereas our bound
uses an $L^2$ warm start.
\end{remark}

\begin{remark}[Comparison of ULD with \cite{lehec-2025} in TV distance]
A closer ULD benchmark in total variation is
\cite[Theorem~1]{lehec-2025}.  Suppose $\nabla U$ is globally
$L$-Lipschitz.  Take $\rho_0=\nu_0\otimes\kappa$
with $\dd\nu_0/\dd\mu\in L^2(\mu)$ and
$\log(1+W_0^2)=\widetilde{\mathcal O}(d)$, where
$W_0^2=\chi^2(\nu_0\Vert\mu)$.  Writing its Poincar\'e constant
as $C_P=1/m$, its
warm-start iteration complexity is
$\mathcal O^\ast\left(
    \frac{(L/m)^{3/2}\sqrt d}{\varepsilon}
  \right)$
under a Poincar\'e inequality.
For our result, suppose also that $L\geq1$ and
Assumptions~\ref{ass:PI}--\ref{ass:compact} hold.  If $U$ is
$K$-weakly convex, then $A=\sqrt m+\sqrt K$ and
\[
\mathcal N_\varepsilon(0,\gamma_{\rm ULD}^\star)
  =
  \widetilde{\mathcal O}\left(
    \frac{L\sqrt{\mathfrak M_\chi}\,A}
    {\varepsilon m^{3/2}}
  \right)  
  =
  \widetilde{\mathcal O}\left(
    \frac{L^{3/2}\sqrt d\,(\sqrt m+\sqrt K)}
    {\varepsilon m^{3/2}}
  \right).
\]
In particular, if $m\leq1$ and $K=\mathcal O(1)$, then
$A=\mathcal O(1)$, and our bound is also
$\widetilde{\mathcal O}(L^{3/2}\sqrt d/(\varepsilon m^{3/2}))$
in this regime.
On the other hand, in the log-concave case, the TV iteration complexity in
\cite[Theorem~1]{lehec-2025} improves to
$\mathcal O^\ast\left(
    \frac{(L/m)\sqrt d}{\varepsilon}
  \right)$.
Our convex ULD specialization is
$\widetilde{\mathcal O}(L^{3/2}\sqrt d/(\varepsilon m))$,
which contains an additional factor $\sqrt L$ compared with Lehec's
convex bound under these conditions.
\end{remark}

\begin{remark}[Comparison of HFHRMC with \cite{cortild-delplancke-oudjane-peypouquet-2025} in the strongly convex setting]
\label{rem:cortild-complexity-comparison}
We next compare with the iteration complexity in
\cite[Algorithm~2 and Corollary~4.4]{cortild-delplancke-oudjane-peypouquet-2025}, which is based on the
same HFHRMC algorithm.  In their notation, our normalization
corresponds to
$a=b=1$,
$\beta_{\rm CDOP}=\sigma_{x,{\rm CDOP}}^2=\alpha$,
$\alpha_{\rm CDOP}=\sigma_{y,{\rm CDOP}}^2=\gamma$,
$\gamma_{\rm CDOP}=1$.
If $U$ is $m$-strongly convex, applying the Bakry--\'Emery
criterion \cite[Theorem~3.1]{menz-schlichting-2014} to the joint
potential gives $\pi$ a
log-Sobolev constant $\rho=\min\{m,1\}$.
Since our iteration complexity is in TV distance whereas that in
\cite[Corollary~4.4]{cortild-delplancke-oudjane-peypouquet-2025} is in KL,
Pinsker's inequality \cite[Lemma~2.5]{tsybakov-2009} suggests using KL
accuracy of order $\varepsilon^{2}$ for a fair comparison. Substitution
in the explicit constants of
\cite[Appendix~C.3]{cortild-delplancke-oudjane-peypouquet-2025} gives in the high-accuracy regime
\begin{equation}\label{eq:cortild-TV-complexity}
  \mathcal N_{\varepsilon}^{\rm CDOP}(\alpha,\gamma)
  =\widetilde{\mathcal O}\left(
    \frac{L^2(1+\alpha^2)B_{\rm CDOP}(\alpha,\gamma)d}
    {\rho^2s^3\varepsilon^2}
  \right),
\end{equation}
where, using the notation of
\cite[Corollary~4.4]{cortild-delplancke-oudjane-peypouquet-2025},
\[
  s:=\min\{\alpha,\gamma\},\qquad
  B_{\rm CDOP}(\alpha,\gamma)
  :=12+4L+2\alpha+4L\alpha^2+3\gamma.
\]
For $L\geq1$, minimizing $\mathcal N_{\varepsilon}^{\rm CDOP}(\alpha,\gamma)$ over both $\alpha$ and
$\gamma$ yields
\[
  \inf_{\alpha,\gamma>0}
  \mathcal N_{\varepsilon}^{\rm CDOP}(\alpha,\gamma)
  =\widetilde{\mathcal O}\left(
    \frac{L^3d}{\rho^2\varepsilon^2}
  \right).
\]
By comparison, for $L\geq1$ and $L^2$ warm starts satisfying
$\log(1+W_0^2)=\widetilde{\mathcal O}(d)$, our iteration complexity
in \eqref{eq:optimized-complexity-summary} with $R=0$ gives
\[
  \mathcal N_\varepsilon
  (\alpha_\varepsilon^\star,\gamma_\varepsilon^\star)
  =\widetilde{\mathcal O}\left(
    \frac{L^{3/2}\sqrt d}{m\varepsilon}
  \right).
\]
Thus, under these common strong-convexity and warm-start conditions,
the present analysis improves the certified TV dependence on
accuracy and dimension, as well as the dependence on $m$ when
$m\leq1$, for this common discretization.
\end{remark}

\begin{remark}[Comparison of HFHRMC with \cite{li-zha-tao-2022} in the strongly-convex setting]
\label{rem:alpha-star-comparison}
For the Strang-splitting scheme in
\cite[Algorithm~1]{li-zha-tao-2022}, Corollary~5.4 of that paper
gives a position-marginal $\mathcal W_2$ complexity bound.
Besides $m$-strong convexity and $L$-smoothness, its assumptions
include the third-order growth condition from their Theorem~5.2,
\[
  \abs{\nabla\Delta U(q)}\leq G\sqrt{1+\abs q^2},
  \qquad q\in\R^d,
\]
for some $G>0$, together with $\gamma^2>L+m$ and
$0\leq\alpha\leq(\gamma^2-L-m)/(m\gamma)$.
For $\mathcal W_2$ accuracy $\varepsilon$, their sufficient
iteration count has order
\begin{align}
\mathcal N_\varepsilon^{\rm LZT}
(\alpha,\gamma)
=\widetilde{\mathcal O}\left(\frac{b'(\alpha^{2}-\frac{\alpha}{\gamma}+\frac{1}{\gamma^{2}})}{m(\frac{m}{\gamma}+m\alpha)}\frac{\sqrt{d}}{\varepsilon}\right),
\end{align}
where $b'$ is independent of $\alpha$, with its dependence on
$\gamma$ and the other problem parameters left implicit in this
bound.  Minimizing the $\alpha$-dependent upper bound in
\cite[Remark~5.5]{li-zha-tao-2022} over the admissible interval gives
\[
\alpha_{\rm HFHR}^*
  =\min\left\{
    \frac{\sqrt3-1}{\gamma},
    \frac{\gamma^2-L-m}{m\gamma}
  \right\}.
\]
For fixed target parameters and $\gamma$, this choice is independent
of $\varepsilon$.  For the frozen-force scheme
\eqref{eq:frozen-force-scheme}, Proposition~\ref{prop:optimal-alpha}
minimizes the TV complexity proxy \eqref{eq:iteration-complexity}
and gives $\alpha_\varepsilon^\star\to0$ and
$\gamma_\varepsilon^\star\to2\sqrt m$ in the strongly convex case.
These tuning conclusions concern different discretizations and
error criteria.
\end{remark}

\section{Numerical Experiments}\label{sec:numerical}

This section illustrates our theory by studying a toy example of anisotropic Gaussian target on synthetic data and three Bayesian learning problems on real data. All discrete-time experiments use the HFHRMC algorithm in \eqref{eq:frozen-force-scheme}, with one full data gradient evaluation per iteration. Throughout this section, HFHRMC refers to $\alpha>0$ and $\alpha=0$ gives KLMC. The first experiment computes the continuous-time Gaussian law, and the second computes a Gaussian upper bound on the total variation error of the discrete chain. The last two experiments compare HFHRMC and KLMC with numerical posterior references for a Bayesian logistic regression and a Bayesian neural network. In the Bayesian learning experiments (Sections~\ref{subsec:numerical-linear-regression}, \ref{subsec:numerical-logistic-regression}, and~\ref{subsec:numerical-bnn}), $K$ is the number of completed iterations and also the number of full data gradient evaluations along each particle trajectory. The initial state is indexed by $K=0$. The number of particles and independent ensembles is not included in this count.

\subsection{Anisotropic Gaussian target}

We first consider a toy example with the four-dimensional Gaussian position target $\mu=\mathcal N(0,H^{-1})$ with
$H:=\operatorname{diag}(0.25,1,4,9)$.
Thus, $m=0.25$, $L=9$, $R=0$, and we consider $\gamma=\sqrt m=0.5$ and $\gamma=2\sqrt m=1$. The initial joint law is Gaussian with the invariant covariance $\operatorname{diag}(H^{-1},\Id_4)$ and a position mean shift $2e_1$ in the slowest eigendirection. Its initial divergence is therefore $\chi^2(\rho_0\Vert\pi)=e-1$. Since the covariance remains invariant, the mean $z_t$ is obtained from the exact matrix exponential of the linear HFHR dynamics and
$$
\chi^2(\rho_t\Vert\pi)=\exp\!\left(z_t^{\top}\operatorname{diag}(H,\Id_4)z_t\right)-1.
$$
Consequently, this toy experiment contains no time discretization or Monte Carlo error.

Figure~\ref{fig:gaussian-chi2} compares $\alpha\in\{0,0.1,0.3,0.7\}$ at both friction values. At $t=18$ and $\gamma=0.5$, the exact $\chi^2$ divergences are $2.1391\times10^{-4}$, $1.1877\times10^{-4}$, $3.6623\times10^{-5}$, and $4.3128\times10^{-6}$, respectively. At $\gamma=1$, they are $2.7566\times10^{-6}$, $2.8670\times10^{-7}$, $1.2550\times10^{-8}$, and $2.4157\times10^{-9}$. Increasing $\alpha$ accelerates convergence in this Gaussian slow mode for both choices of $\gamma$, which is consistent with the additional $\alpha m$ contribution in Theorem~\ref{thm:main} and the $\chi^2$-divergence estimate in Corollary~\ref{cor:chi-square-TV}.

\begin{figure}[htbp]
\centering
\begin{subfigure}{0.48\textwidth}
\centering
\includegraphics[width=\linewidth]{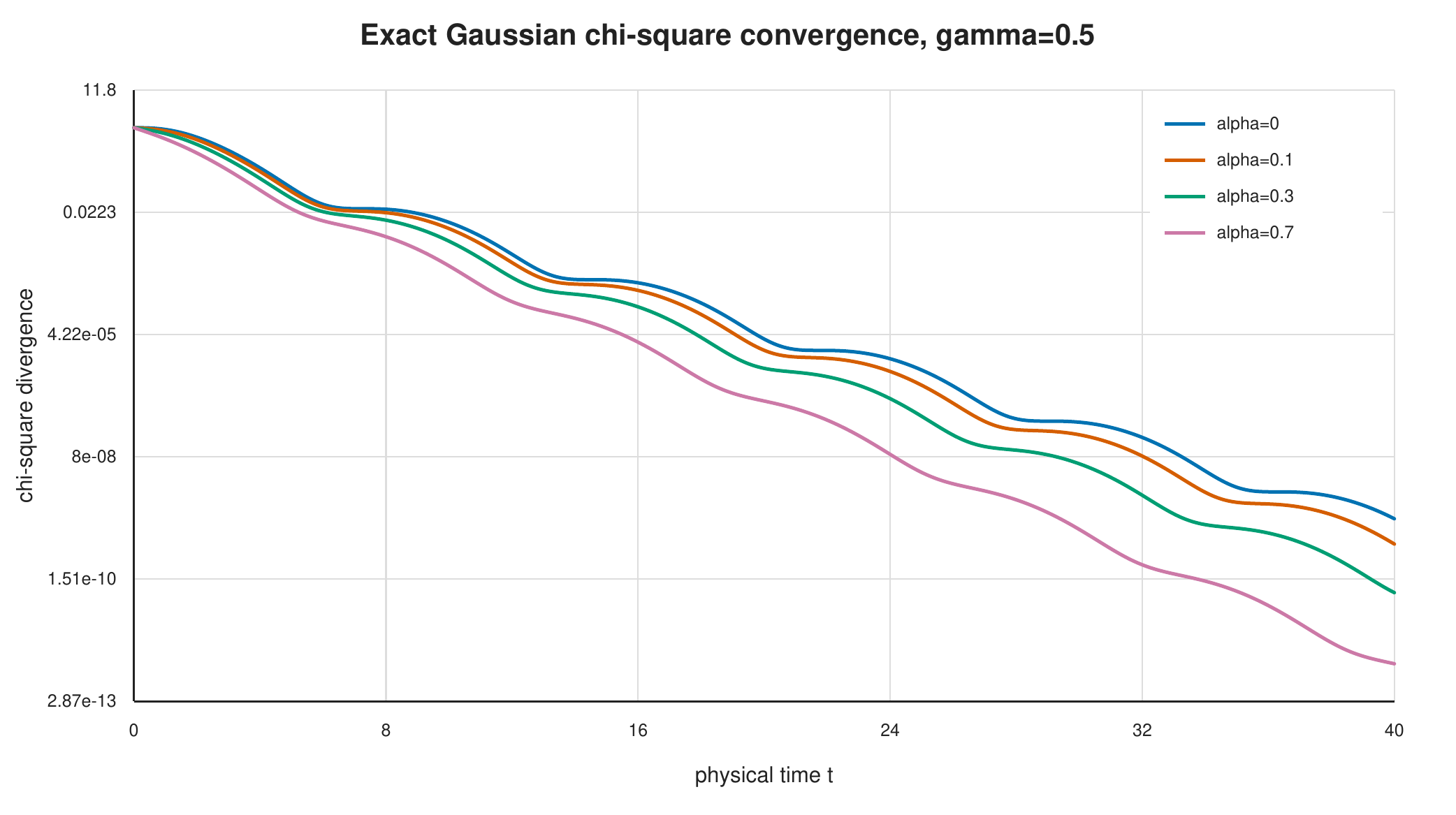}
\caption{$\gamma=\sqrt m=0.5$}
\end{subfigure}\hfill
\begin{subfigure}{0.48\textwidth}
\centering
\includegraphics[width=\linewidth]{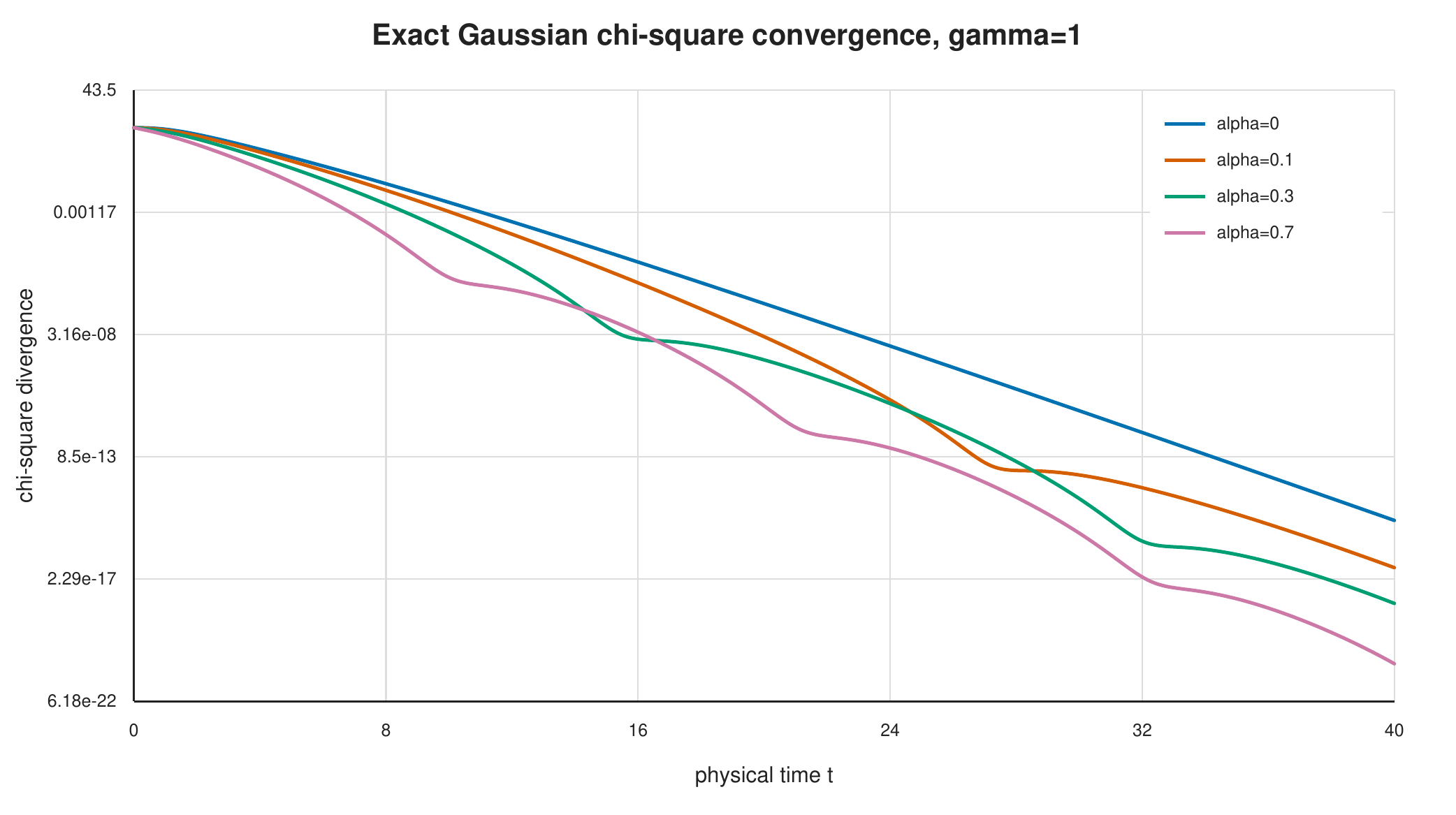}
\caption{$\gamma=2\sqrt m=1$}
\end{subfigure}
\caption{Exact convergence for the anisotropic Gaussian target in $\chi^2$-divergence. Every curve is evaluated analytically from the continuous-time Gaussian law.}
\label{fig:gaussian-chi2}
\end{figure}

\subsection{Bayesian linear regression on real data}\label{subsec:numerical-linear-regression}

We next conduct a Bayesian linear regresion experiment on real data. We use all $1599$ records and $11$ features of the UCI Wine Quality red wine data set.\footnote{The data set is available from \href{https://archive.ics.uci.edu/dataset/186/wine+quality}{the UCI Machine Learning Repository}, DOI 10.24432/C56S3T.} A fixed random split contains $1279$ training records and $320$ test records. Features and responses are centered and scaled using training statistics only. For the standardized response, we use the conjugate model
$$
y\mid\beta\sim\mathcal N(X\beta,\Id_{1279}),\qquad \beta\sim\mathcal N(0,4\Id_{11}).
$$
The posterior is $\mathcal N(\widehat\beta,H^{-1})$, where
$$
H:=X^{\top}X+0.25\Id_{11},\qquad \widehat\beta:=H^{-1}X^{\top}y.
$$
The eigenvalues of $H$ determine the geometry of this Gaussian posterior. Their smallest and largest values are $m=74.3373$ and $L=3976.8834$. Since the posterior potential is quadratic, its Hessian is constant, so $R=0$ and $A=\sqrt m+R=8.6219$. These values determine the friction scales used below. The response vector $y$ determines the posterior center $\widehat\beta$, but it does not change the precision matrix $H$. After transforming predictions back to the original quality scale, the posterior mean predictor has a test root mean square error (RMSE) of $0.6582$ quality units, compared with $0.8198$ for the mean training response. The convergence calculation below is instead based on the exact posterior law.

Since the posterior is Gaussian, its evolution can be calculated without simulating particles. Write the eigendecomposition of $H$ as
$$
H=V\operatorname{diag}(\lambda_1,\ldots,\lambda_{11})V^{\top},\qquad \lambda_1=m,
$$
and transform both the centered position and the momentum according to
$$
\widetilde Q=V^{\top}\left(Q-\widehat\beta\right),\qquad \widetilde P=V^{\top}P.
$$
In these coordinates, the joint dynamics separates into $11$ independent two-dimensional Gaussian systems, one for each eigenvalue of $H$. The continuous time mean and covariance are obtained from the corresponding matrix exponentials, and the discrete mean and covariance are propagated exactly under the linear Gaussian recurrence induced by the frozen force scheme. Let $v_1$ be the unit eigenvector corresponding to $\lambda_1=m$. We initialize the joint law by shifting its mean in this direction while keeping the target covariance,
$$
\mathbb E\left(Q_0-\widehat\beta\right)=\frac1{\sqrt m}v_1,\qquad \mathbb E P_0=0,\qquad \operatorname{Cov}\left((Q_0,P_0)\right)=\operatorname{diag}(H^{-1},\Id_{11}).
$$
The initial mean is therefore one posterior standard deviation away from the target mean in the slowest direction, since
$
\left(\frac1{\sqrt m}v_1\right)^{\top}H\left(\frac1{\sqrt m}v_1\right)=1,
$
and hence $\chi^2(\rho_0\Vert\pi)=e-1$. All distributions in this experiment are calculated exactly from their Gaussian means and covariance matrices.

We next specify how the iteration count is determined. For fixed $(\alpha,\gamma,\varepsilon)$, let $\rho_t^{\alpha,\gamma}$ be the exact law of the continuous-time dynamics.
Since $\norm{\rho_t^{\alpha,\gamma}-\pi}_{\mathrm{TV}}\leq B_{\rm cont}(t):=\frac12\sqrt{\chi^2(\rho_t^{\alpha,\gamma}\Vert\pi)}$, we choose $T$ as the earliest time for which $B_{\rm cont}(T)\leq\varepsilon/2$. Thus the total variation upper bound for the continuous law is at most half of the target tolerance at time $T$. For every integer $K\geq1$, we set $h_K=T/K$ and let $\left(Q_j^{(K)},P_j^{(K)}\right)_{j=0}^{K}$ denote the discrete chain with step size $h_K$. The exact Gaussian recurrence gives its law at the end of the same interval $[0,T]$. 
Then,
$$
\norm{\operatorname{Law}\left(Q_K^{(K)},P_K^{(K)}\right)-\pi}_{\mathrm{TV}}\leq B_K:=\frac12\sqrt{\chi^2\!\left(\operatorname{Law}\left(Q_K^{(K)},P_K^{(K)}\right)\Big\Vert\pi\right)}.
$$
The bound $B_K$ also controls the total variation error of the
coefficient distribution $Q_K^{(K)}$. For each reported count $K>1$, we
verify $B_K\leq\varepsilon<B_{K-1}$ by computing the two Gaussian laws
at the common terminal time $T$. They use step sizes $T/K$ and
$T/(K-1)$, respectively. The inequality certifies accuracy at $K$ and
locates a crossing between the two adjacent discretizations. Along one
trajectory of the corresponding sampler, $K$ is the number of updates
and full data gradient evaluations. At each accuracy, we select the
$\alpha$ with the smallest certified count among those obtained on the grid
$$
\{0.003,0.01,0.03,0.05,0.1,0.15,0.2,0.3,0.4,0.5,0.7,1,1.5,2\}.
$$
We report the results in Figure~\ref{fig:linear-regression-complexity}. 
The first panel compares HFHRMC and KLMC at the common friction $\gamma=2A$ and also shows KLMC at $\gamma=2\sqrt L$. The second panel compares HFHRMC and KLMC at the common friction $\gamma=A$.
In the first panel, the grid selection is $\alpha=0.03$ for the
reported accuracies from $0.3$ to $10^{-3}$ and $\alpha=0.01$ for
$3\times10^{-4}$ and $10^{-4}$. At $\varepsilon=10^{-4}$, the certified
counts are $548652$ for HFHRMC, $1318015$ for KLMC at $\gamma=2A$, and
$2000214$ for KLMC at $\gamma=2\sqrt L$. In the second panel, the grid
selects $\alpha=0.03$ at every reported accuracy. At
$\varepsilon=10^{-4}$, the certified counts are $1034914$ for HFHRMC and
$4159467$ for KLMC. The decrease in the selected $\alpha$ in the first
panel agrees qualitatively with Proposition~\ref{prop:optimal-alpha}, which gives
$\alpha_\varepsilon^\star\to0$ as $\varepsilon\rightarrow0$.

\begin{figure}[htbp]
\centering
\begin{subfigure}{0.48\textwidth}
\centering
\includegraphics[width=\linewidth]{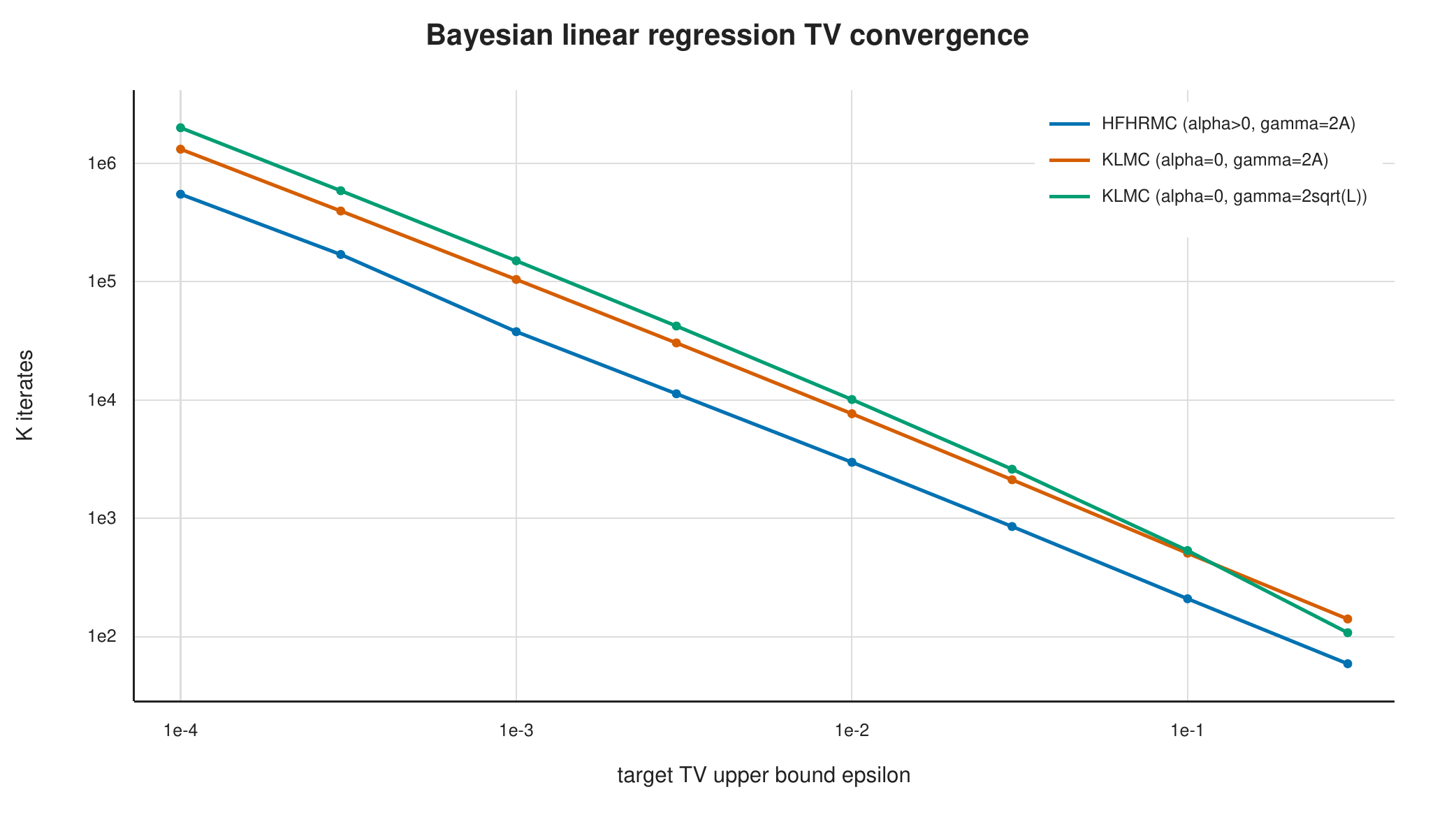}
\caption{The friction choices $\gamma=2A$ and $\gamma=2\sqrt L$}
\end{subfigure}\hfill
\begin{subfigure}{0.48\textwidth}
\centering
\includegraphics[width=\linewidth]{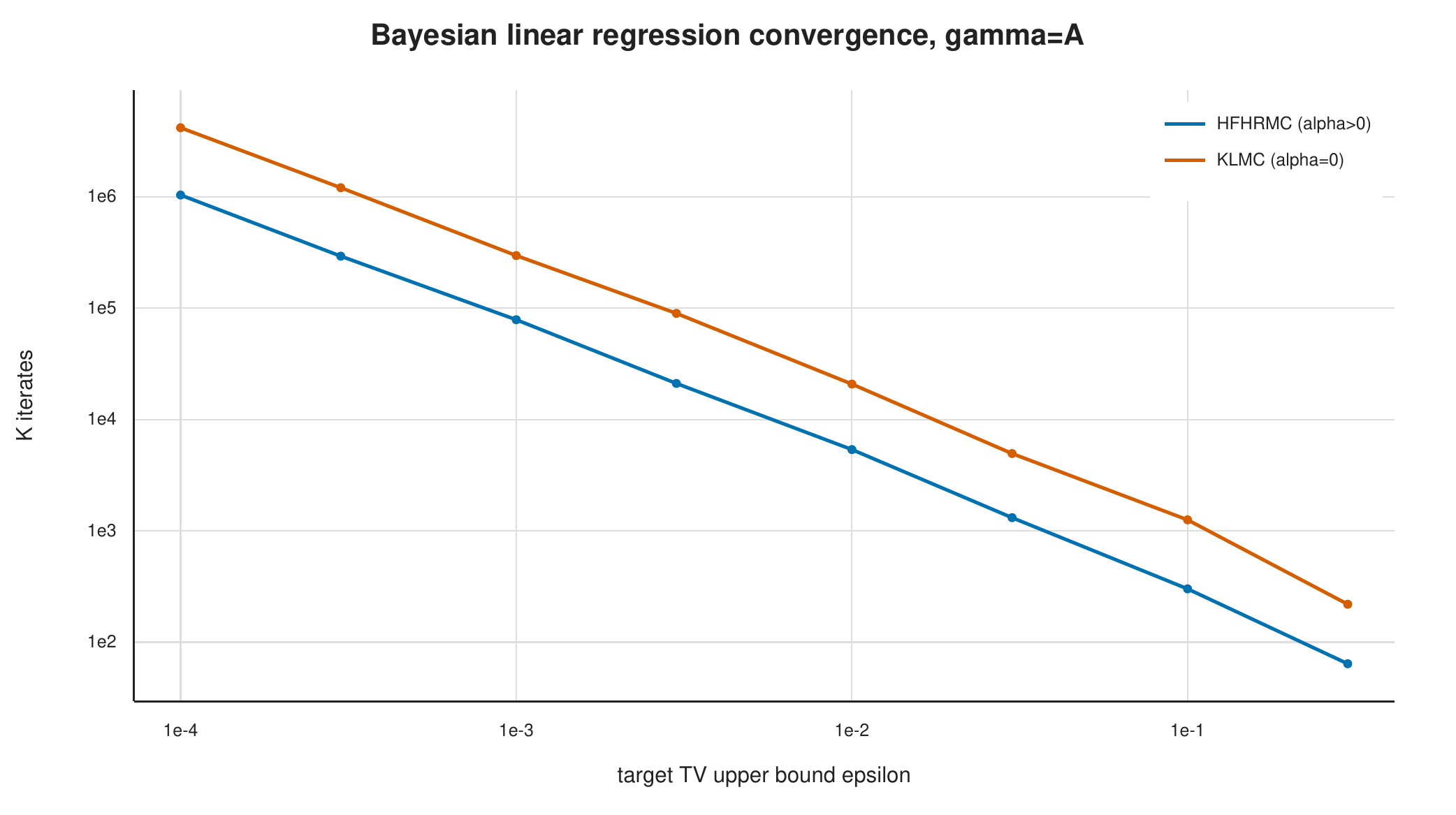}
\caption{The common friction $\gamma=A$}
\end{subfigure}
\caption{Certified iteration counts $K$ with $B_K\leq\varepsilon$ for the Bayesian linear regression posterior on wine quality real data. The quantity $B_K$ is the analytically computed $\chi^2$-divergence upper bound on joint total variation. Positive $\alpha$ is selected from the stated grid separately at each accuracy. For the corresponding sampler, $K$ is also the number of full data gradient evaluations along one trajectory. Both axes are logarithmic. We note $\gamma = 2\sqrt{L}$ corresponds to~\cite{cheng2018underdamped}.}
\label{fig:linear-regression-complexity}
\end{figure}

\subsection{Bayesian logistic regression on real data}\label{subsec:numerical-logistic-regression}

We next conduct a Bayesian logistic regression experiment using the UCI Breast Cancer Wisconsin Diagnostic data set.\footnote{The data set is available from \href{https://archive.ics.uci.edu/dataset/17/breast+cancer+wisconsin+diagnostic}{the UCI Machine Learning Repository}, DOI 10.24432/C5DW2B.} We use a fixed stratified subset with $96$ training and $64$ test observations. The six features are radius, texture, smoothness, compactness, concavity, and number of concave points, all measured from the original image. Standardization uses training statistics only, after which the design is multiplied by $0.65$. With a standard Gaussian prior, the posterior potential is
$$
U(\beta)=\frac12\abs\beta^2+\sum_{i=1}^{96}\left[\log\!\left(1+e^{x_i^{\top}\beta}\right)-y_i x_i^{\top}\beta\right].
$$
It satisfies
$$
\Id_6\preceq\nabla^2U(\beta)\preceq\left(1+\frac14\norm{X}_{\mathrm{op}}^2\right)\Id_6,
$$
so that we use the lower bound $m=1$ and $R=0$, and the scaled design gives $L\leq39.7118$. We use a weighted importance sample with $60000$ proposal draws as the numerical posterior reference. Its effective sample size is $45100.5$, or $75.17\%$ of the proposal size.

With $m=1$, we have $A=1$ and choose $\gamma=2A=2$. Since $r_0(2)=2/9$, we set $\alpha=2/9$, which increases $r(\alpha,\gamma)$ in \eqref{eq:rate-proxy} from $2/9$ to $4/9$, and use $h=0.30/\sqrt L=0.047606$. This choice follows the continuous time rate proxy, whereas $\alpha_\varepsilon^\star$ in Proposition~\ref{prop:optimal-alpha} minimizes the discrete complexity bound. The comparisons are KLMC with the same $\gamma$ and $h$, and KLMC with $\gamma=2\sqrt L=12.6035$ and $h=0.50/\sqrt L=0.079343$. For every ensemble, $Q_0$ and $P_0$ are independent standard Gaussian arrays, and all configurations share these arrays and the random number streams. Each method performs $K_{\max}=500$ iterations using $24$ independent ensembles of $384$ particles. The curves report ensemble means and pointwise $95\%$ Student $t$ confidence intervals.

To define posterior convergence, consider one ensemble after $K$ iterations with $M=384$ coefficient particles $\left\{\beta_K^{(j)}\right\}_{j=1}^M$. Let $\left\{\widetilde\beta^{(s)},\omega_s\right\}_{s=1}^{S}$ denote the weighted numerical reference, where $S=60000$, $\omega_s\geq0$, and $\sum_{s=1}^{S}\omega_s=1$. The two empirical posterior laws are
$\widehat\nu_K=\frac1M\sum_{j=1}^{M}\delta_{\beta_K^{(j)}}$
and $\widehat\nu_{\rm ref}=\sum_{s=1}^{S}\omega_s\delta_{\widetilde\beta^{(s)}}$.
For each of the $96$ fixed unit directions $v_\ell$, let $\widehat F_{K,\ell}^{-1}$ be the empirical quantile function of $(v_\ell^{\top}\cdot)_{\#}\widehat\nu_K$. For the weighted reference, let $\widehat F_{{\rm ref},\ell}^{-1}$ be the interpolated projected quantile curve obtained from the cumulative weight midpoints used in the computation. With $J=512$ and $u_r=(r-1/2)/J$, the plotted posterior convergence diagnostic is
$$
D_{\rm post}(K)=\left[\frac1{96J}\sum_{\ell=1}^{96}\sum_{r=1}^{J}\left|\widehat F_{K,\ell}^{-1}(u_r)-\widehat F_{{\rm ref},\ell}^{-1}(u_r)\right|^2\right]^{1/2}.
$$
Thus $D_{\rm post}$ is the sliced $\mathcal W_2$ approximation computed with $96$ directions and $512$ quantile levels. It compares projected coefficient distributions rather than a single posterior moment. 

To define posterior mean probability convergence, write $p_\beta(x_i)=\sigma(x_i^{\top}\beta)$, where $\sigma(z)=(1+e^{-z})^{-1}$, for the positive class probability at the $i$th test input. The current and reference posterior mean probabilities and their discrepancy are
\begin{align*}
&\overline p_K(x_i)=\frac1M\sum_{j=1}^{M}p_{\beta_K^{(j)}}(x_i),\qquad\overline p_{\rm ref}(x_i)=\sum_{s=1}^{S}\omega_s p_{\widetilde\beta^{(s)}}(x_i), 
\\
&D_{\rm mean}(K)=\left[\frac1{64}\sum_{i=1}^{64}\left|\overline p_K(x_i)-\overline p_{\rm ref}(x_i)\right|^2\right]^{1/2}.
\end{align*}
Thus, $D_{\rm mean}$ is the RMSE of the posterior expected class probabilities over the $64$ test inputs. It measures convergence of the posterior predictive mean, not convergence of the entire posterior distribution or classification accuracy. Figure~\ref{fig:logistic-convergence} reports both diagnostics. Over the fixed transient window $0\leq K\leq240$, the sliced $\mathcal W_2$ areas under the curve are $28.4475$, $46.3072$, and $61.2117$ for HFHRMC, KLMC at $\gamma=2$, and KLMC at $\gamma=2\sqrt L$, respectively. The corresponding posterior mean probability RMSE areas are $2.9340$, $5.3162$, and $7.3604$. For sliced $\mathcal W_2$, the first recorded values after which the upper confidence limit remains below $0.10$ are $K=53$, $77$, and $164$, respectively. For probability RMSE, the corresponding values for the threshold $0.02$ are $K=30$, $53$, and $93$. Over the late window $280\leq K\leq500$, the coefficients of variation remain below $6.5\%$ for both diagnostics and all configurations. Since each iteration uses one full data gradient evaluation, HFHRMC reaches both thresholds with fewer gradient evaluations.

\begin{figure}[htbp]
\centering
\begin{subfigure}{0.48\textwidth}
\centering
\includegraphics[width=\linewidth]{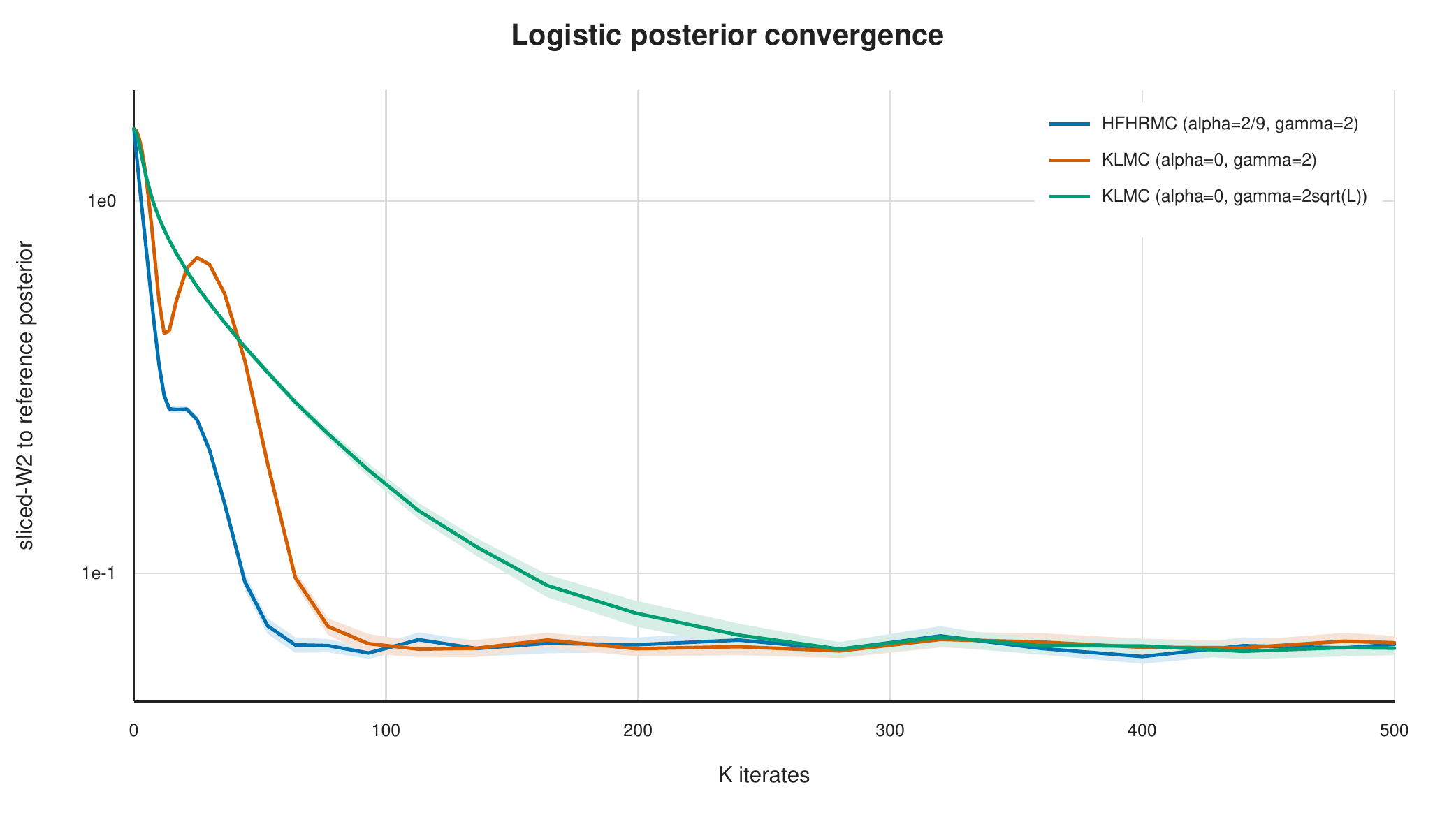}
\caption{Posterior convergence}
\end{subfigure}\hfill
\begin{subfigure}{0.48\textwidth}
\centering
\includegraphics[width=\linewidth]{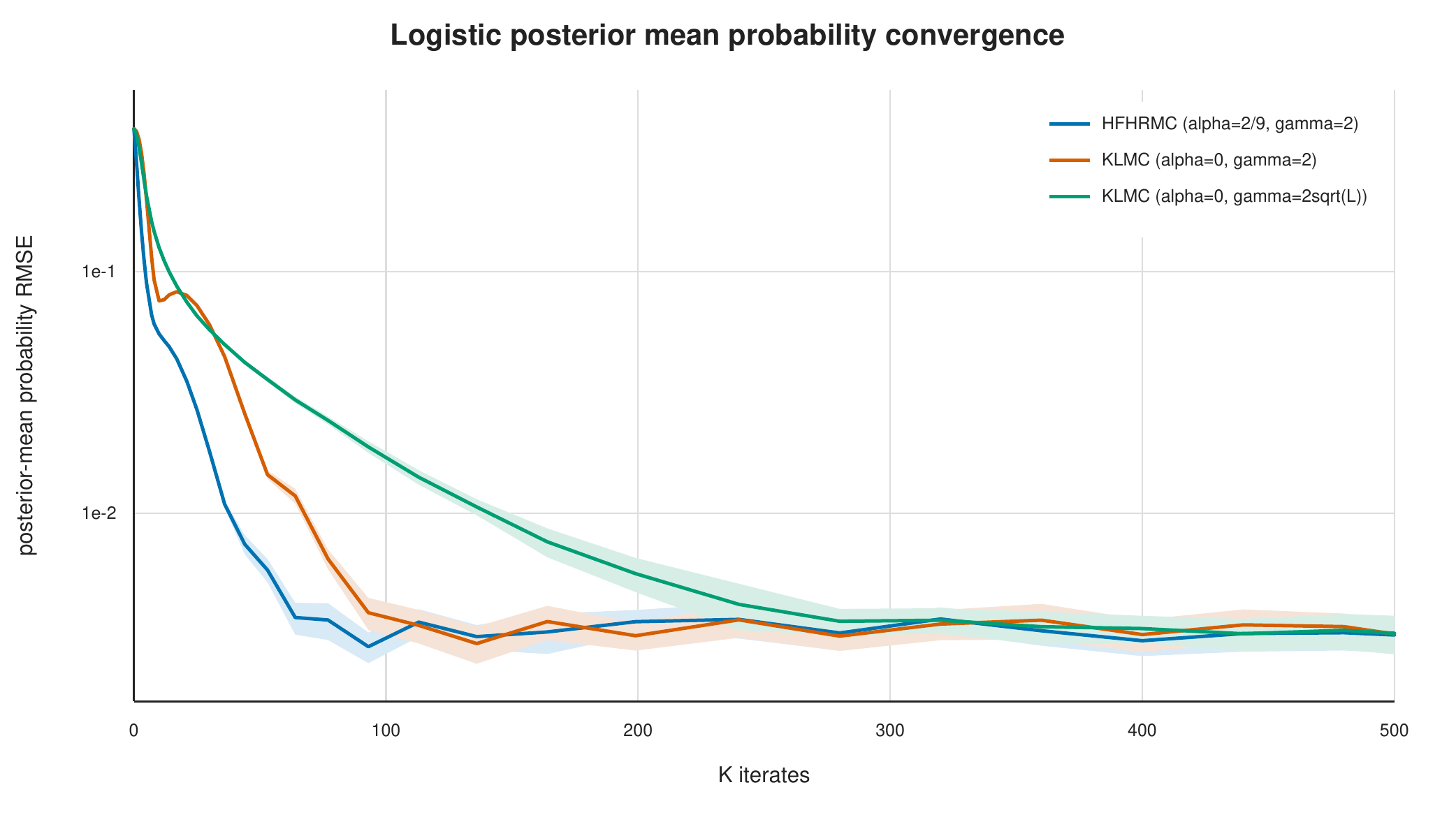}
\caption{Posterior mean probability convergence}
\end{subfigure}
\caption{Convergence for the real data Bayesian logistic posterior. Solid curves are independent ensemble means and shaded regions are pointwise $95\%$ confidence intervals. The horizontal axis reports $K$ iterates. The state shown at $K$ has used exactly $K$ full data gradient evaluations along each particle trajectory. We note $\gamma = 2\sqrt{L}$ corresponds to~\cite{cheng2018underdamped}.}
\label{fig:logistic-convergence}
\end{figure}

\subsection{Bayesian neural network with bounded weights and biases on real data}\label{subsec:numerical-bnn}

The final experiment is conducted for a Bayesian neural network with bounded weights and biases on real data; it uses the same training and test split with the standardized radius and texture features. Let $\theta\in\mathbb R^{13}$ be the latent parameter of a tanh network with two inputs, three hidden units, and one output. Its $13$ physical weights and biases are obtained coordinatewise from
$$
\vartheta_j=1.5\tanh(\theta_j)\in(-1.5,1.5),\qquad j=1,\ldots,13,
$$ and the latent prior is $\theta\sim\mathcal N(0,\Id_{13})$. Thus boundedness refers to the transformed weights and biases used by the network. The posterior of the latent variable $\theta$ has support $\mathbb R^{13}$, and its pushforward under the coordinatewise tanh map is supported on $(-1.5,1.5)^{13}$. The transformation pushes the Gaussian latent prior forward to a nonuniform prior supported on $(-1.5,1.5)^{13}$. Bounded parameter classes have also been used in Bayesian neural network theory. In particular, \cite{polson-rockova-2018} uses a uniform slab on $[-1,1]$ for active network parameters, \cite{kong-kim-2025} proves posterior concentration using approximation by fully connected networks with bounded parameters. Unlike the slab construction in \cite{polson-rockova-2018}, the tanh map gives a smooth unconstrained parametrization for the Langevin dynamics.

Since the output layer has three weights and one bias, the network logit satisfies $\abs{a_\theta(x)}\leq6$. If $a_\theta(x)$ denotes this logit, the Gibbs posterior potential is
$$
U(\theta)=\frac12\abs\theta^2+\frac6{96}\sum_{i=1}^{96}\left[\log\!\left(1+e^{a_\theta(x_i)}\right)-y_i a_\theta(x_i)\right].
$$
For $y\in\{0,1\}$, the oscillation of $\log(1+e^a)-ya$ over $a\in[-6,6]$ equals $6$. Consequently, if $W$ denotes the data term in $U$, then $\operatorname{osc}(W)\leq36$. The bounded perturbation principle \cite{holley-stroock-1987} applied to
the Gaussian latent prior gives a Poincar\'e inequality for the posterior.
The finite data set and the bounded first two derivatives of tanh imply
that $\nabla W$ and $\nabla^2W$ are bounded. Thus $\nabla U$ is globally
Lipschitz, and $\abs{\nabla U(\theta)}\geq\abs\theta-\norm{\nabla W}_\infty$
gives compactness by Remark~\ref{rem:compact}. These conditions allow a nonconvex
posterior. The sampler parameters are specified below. A weighted prior importance sample with $50000$ draws defines the numerical reference and has effective sample size $14145.1$, or $28.29\%$.

We compare HFHRMC with $(\alpha,\gamma,h)=(0.4,1,0.03)$ and KLMC with $(\alpha,\gamma,h)=(0,1,0.03)$. These parameters compare the methods at common $\gamma$ and $h$ rather than optimize over sampler parameters. We initialize $Q_0\sim\mathcal N(s,\Id_{13})$, where the entries of $s$ are equally spaced from $-0.7$ to $0.7$, and $P_0\sim\mathcal N(0,\Id_{13})$. Both methods use the same initial arrays, random number streams, friction, and step size. Each method performs $K_{\max}=260$ iterations and produces states indexed by $K=0,\ldots,260$. The state at $K$ has therefore used exactly $K$ full data gradient evaluations along each particle trajectory. We run $12$ independent ensembles of $384$ particles. For a latent parameter $\theta$, define its vector of test probabilities by
$$
f_\theta=\left(p_\theta(x_1),\ldots,p_\theta(x_{64})\right),\qquad p_\theta(x_i)=\sigma\!\left(a_\theta(x_i)\right).
$$
To avoid hidden unit permutation ambiguity, we use the same finite direction formula and projected quantile interpolation as $D_{\rm post}$ for the empirical and weighted reference laws of $f_\theta$ in $[0,1]^{64}$, again using $96$ fixed directions and $512$ quantile levels. The first Bayesian neural network (BNN) panel compares the posterior laws of predictive probability vectors, which are invariant under hidden unit permutations, rather than distances between latent parameter vectors. We use the formula for $D_{\rm mean}$ with $p_\beta$ replaced by $p_\theta$, so the second panel measures the RMSE of the posterior mean probability vector relative to the numerical reference. It measures a posterior mean rather than a second distributional distance.

Figure~\ref{fig:bnn-convergence} reports both diagnostics. Over the recorded checkpoints spanning $K=0$ to $K=260$, the areas under the sliced $\mathcal W_2$ curve for the predictive probability distribution are $7.3413$ for HFHRMC and $9.8377$ for KLMC. The posterior mean probability RMSE areas are $6.1325$ and $8.9156$, respectively. Over the recorded checkpoints with $K\geq151$, the sliced $\mathcal W_2$ coefficients of variation are $1.69\%$ and $3.51\%$, and the probability RMSE coefficients of variation are $5.63\%$ and $7.27\%$. Thus, HFHRMC has lower empirical areas under both reported error curves over the $260$ iteration budget and lower late window variability.

\begin{figure}[htbp]
\centering
\begin{subfigure}{0.48\textwidth}
\centering
\includegraphics[width=\linewidth]{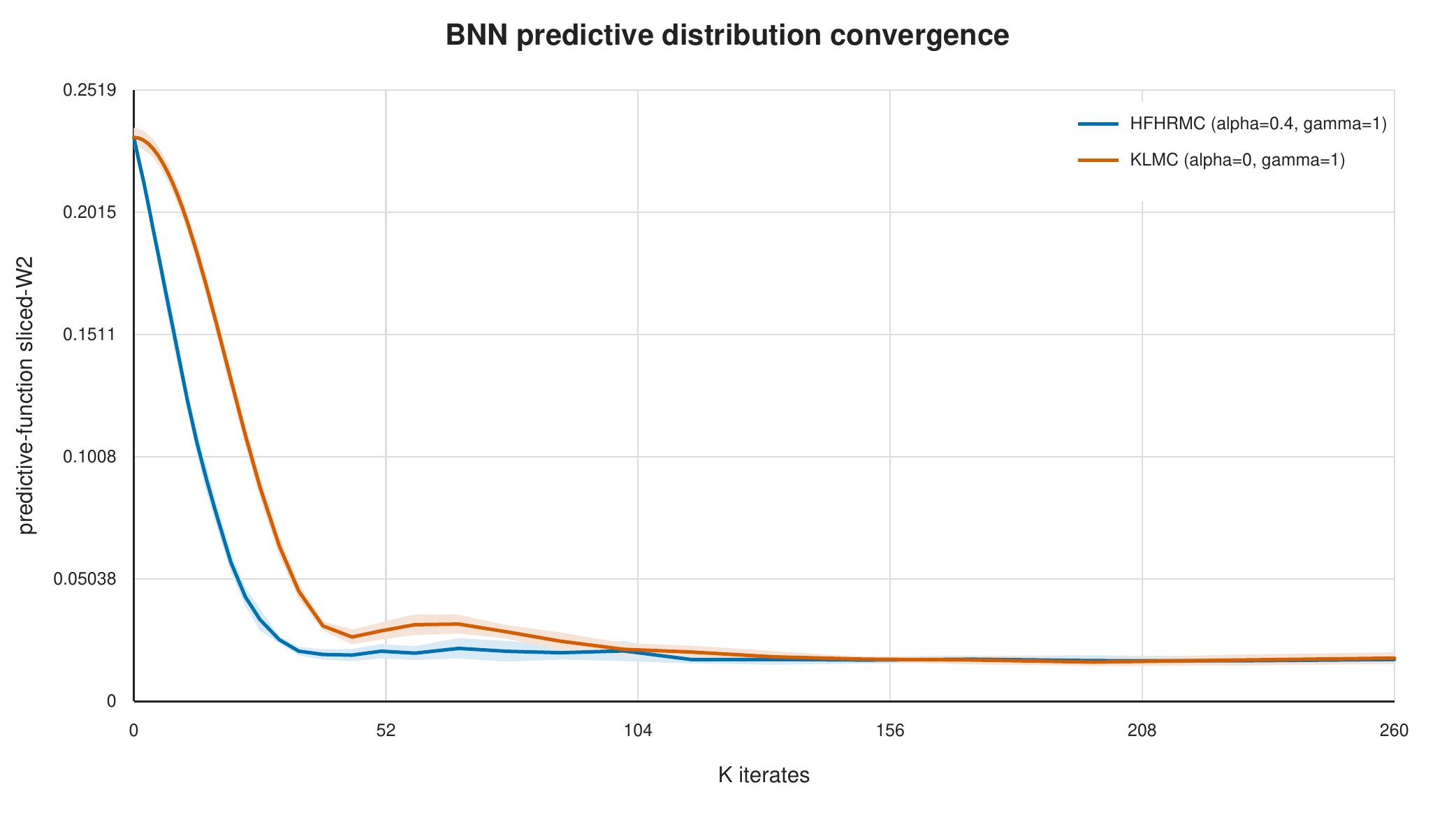}
\caption{Predictive distribution convergence}
\end{subfigure}\hfill
\begin{subfigure}{0.48\textwidth}
\centering
\includegraphics[width=\linewidth]{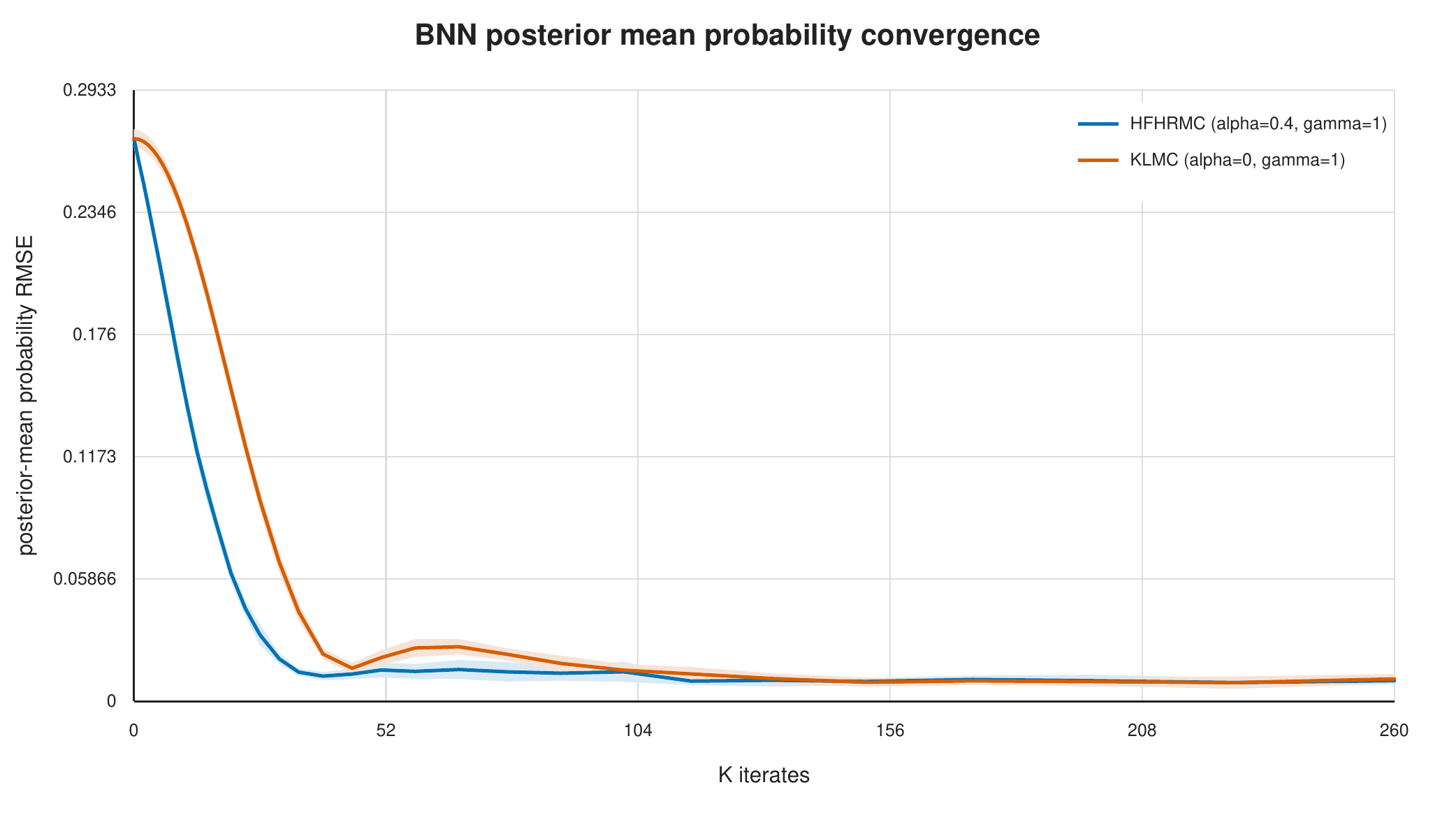}
\caption{Posterior mean probability convergence}
\end{subfigure}
\caption{Convergence for the Bayesian neural network with transformed physical weights and biases in $(-1.5,1.5)$. Solid curves are independent ensemble means and shaded regions are pointwise $95\%$ confidence intervals. The horizontal axis reports $K$ iterates. The state shown at $K$ has used exactly $K$ full data gradient evaluations along each particle trajectory.}
\label{fig:bnn-convergence}
\end{figure}


\section{Conclusion}\label{sec:conclude}

We established an explicit $L^2(\pi)$ convergence rate for HFHR
dynamics under a position Poincar\'e inequality and weighted
regularity assumptions, without requiring convexity of the potential.
The estimate combines the hypocoercive ULD contribution with the
direct position-diffusion contribution and is uniform at the ULD
endpoint $\alpha=0$.  The main analytic ingredient is an
HFHR-adapted time--augmented Poincar\'e inequality, in which the
position diffusion is handled through the position regularity of the
divergence test rather than through a velocity negative Sobolev norm.
As a by-product, we also obtained explicit convergence in $\chi^2$
divergence and total variation for $L^2$ initial densities.  For the
discrete-time HFHRMC algorithm, we obtained a non-asymptotic
convergence bound and iteration complexity by using a path-space KL
estimate without a separate time-uniform moment-stability assumption.
Our results hold throughout the parameter space
$\alpha\geq0$, $\gamma>0$, remain regular at the ULD endpoint, and
therefore allow unconstrained parameter tuning.  The
optimized positive position diffusion improves the finite-accuracy
complexity bound, while its leading high-accuracy order agrees with
that of the separately optimized ULD endpoint.
Under common strong-convexity and warm-start conditions, the
comparison in Remark~\ref{rem:cortild-complexity-comparison}
improves the accuracy and dimension dependence of the existing TV
complexity bound for the same frozen-force discretization.
The numerical experiments complemented these results including the Bayesian learning problems on real data. 
They illustrated the effect and benefit of positive $\alpha$ for the HFHRMC algorithm.

\appendix

\section{Weighted Elliptic Estimates}
\label{app:elliptic}

This appendix supplies the analytic ingredients for the divergence
construction.  They are the time-augmented versions of the estimates
in \cite[Lemmas~2.4 and~2.6]{cao-lu-wang-2023}.  We use the notation of
Section~\ref{sec:setup}
on $I\times\R^d$, with measure $\lambda_T\otimes\mu$; no spatial
boundary is present.

We begin by tensorizing the position Poincar\'e inequality with the
Neumann Poincar\'e inequality in time.  This provides coercivity for
the mixed time--position operator.

\begin{lemma}[Time--position Poincar\'e inequality]
\label{lem:appendix-PI}
For every $u\in H^1(\lambda_T\otimes\mu)$,
\begin{align}
  \norm{u-(u)_{\lambda_T\otimes\mu}}_{L^2}^2
  \leq
  \frac1m\norm{\nabla_qu}_{L^2}^2
  +\frac{T^2}{\pi^2}\norm{\partial_tu}_{L^2}^2
 \leq
  \max\left\{\frac1m,\frac{T^2}{\pi^2}\right\}
  \norm{\widetilde\nabla u}_{L^2}^2.
  \label{eq:appendix-PI}
\end{align}
\end{lemma}

\begin{proof}
Write $u_t(q)=u(t,q)$ and
$\bar u(t)=\int_{\mathbb{R}^{d}} u_t\,\dd\mu$.  Variance decomposition and Assumption
\ref{ass:PI} give
\[
  \int_I\norm{u_t-\bar u(t)}_{L^2(\mu)}^2\,\dd\lambda_T(t)
  \leq\frac1m\norm{\nabla_qu}_{L^2}^2.
\]
The Neumann Poincar\'e inequality on $(0,T)$ gives
\[
  \norm{\bar u-(\bar u)_{\lambda_T}}_{L^2(\lambda_T)}^2
  \leq\frac{T^2}{\pi^2}
  \norm{\partial_t\bar u}_{L^2(\lambda_T)}^2
  \leq\frac{T^2}{\pi^2}\norm{\partial_tu}_{L^2}^2.
\]
The two components are orthogonal, which proves
\eqref{eq:appendix-PI}.
\end{proof}

The unbounded coefficient $\nabla U$ appears when the Hamiltonian
operator acts on the divergence test.  Assumption~\ref{ass:growth}
gives the following weighted multiplier estimate.

\begin{lemma}[Control of multiplication by $\nabla U$]
\label{lem:weighted-multiplier}
Under Assumption~\ref{ass:growth}, every
$\varphi\in H^1(\lambda_T\otimes\mu)$ satisfies
\begin{equation}\label{eq:weighted-multiplier}
  \norm{\varphi\nabla U}_{L^2}^2
  \leq
  16\norm{\nabla_q\varphi}_{L^2}^2
  +4Md\norm\varphi_{L^2}^2.
\end{equation}
\end{lemma}

\begin{proof}
Weighted integration by parts gives
\begin{align*}
  \int_{I\times\mathbb{R}^{d}}\varphi^2\abs{\nabla U}^2\,\dd\lambda_T\,\dd\mu
  &=
  \int_{I\times\mathbb{R}^{d}}\nabla_q\cdot(\varphi^2\nabla U)
   \,\dd\lambda_T\,\dd\mu\\
  &=
  2\int_{I\times\mathbb{R}^{d}}\varphi\nabla_q\varphi\cdot\nabla U
   \,\dd\lambda_T\,\dd\mu
  +\int_{I\times\mathbb{R}^{d}}\varphi^2\Delta U\,\dd\lambda_T\,\dd\mu.
\end{align*}
Young's inequality and \eqref{eq:Laplacian-growth} bound the right-hand
side by
\[
  \left(\frac14+\frac\delta2\right)
  \norm{\varphi\nabla U}_{L^2}^2
  +4\norm{\nabla_q\varphi}_{L^2}^2
  +Md\norm\varphi_{L^2}^2.
\]
Since $\delta<1$, absorption proves the claim.
\end{proof}

The divergence construction also requires second-order regularity.
The next weighted Bochner estimate controls it on compactly supported
smooth functions in terms of $R$; Lemma~\ref{lem:position-domain}
extends this estimate to the position operator domain.

\begin{lemma}[Spatial $H^2$ estimate]
\label{lem:Bochner}
Under Assumption~\ref{ass:growth}, every $u\in C_c^\infty(\R^d)$ satisfies
\begin{equation}\label{eq:Bochner-spatial}
  \norm{\nabla_q^2u}_{L^2(\mu)}^2
  \leq
  C\left(
  \norm{\nabla_q^*\nabla_qu}_{L^2(\mu)}^2
  +R^2\norm{\nabla_qu}_{L^2(\mu)}^2
  \right).
\end{equation}
\end{lemma}

\begin{proof}
Weighted integration by parts gives
\begin{equation}\label{eq:spatial-Bochner-identity}
  \norm{\nabla_q^2u}_{L^2(\mu)}^2
  =
  \norm{\nabla_q^*\nabla_qu}_{L^2(\mu)}^2
  -\int_{\R^d}
  (\nabla_qu)^\top\nabla^2U(\nabla_qu)
   \,\dd\mu.
\end{equation}
All norms below are in $L^2(\mu)$.
For convex $U$, the last term is nonpositive and
\eqref{eq:Bochner-spatial} holds with $R=0$.  If
$\nabla^2U\succeq-K\Id_{d}$, it is bounded by
$K\norm{\nabla_qu}_{L^2}^2$, giving $R=\sqrt K$.

In the general case, \eqref{eq:Hessian-growth} and Cauchy--Schwarz inequality
yield
\begin{align*}
  \abs{\int_{\R^d}
  (\nabla_qu)^\top\nabla^2U(\nabla_qu)\,\dd\mu}
  &\leq
  M\sqrt d\,\norm{\nabla_qu}^2
  +M\norm{\nabla_qu}\norm{\abs{\nabla U}\abs{\nabla_qu}}.
\end{align*}
Apply Lemma~\ref{lem:weighted-multiplier} to each component of
$\nabla_qu$:
\[
  \norm{\abs{\nabla U}\abs{\nabla_qu}}^2
  \leq16\norm{\nabla_q^2u}^2+4Md\norm{\nabla_qu}^2.
\]
Young's inequality therefore gives
\[
  \abs{\int_{\R^d}
  (\nabla_qu)^\top\nabla^2U(\nabla_qu)\,\dd\mu}
  \leq
  \frac12\norm{\nabla_q^2u}^2
  +C\left(M^2+M^{3/2}d^{1/2}\right)\norm{\nabla_qu}^2.
\]
Absorb $\norm{\nabla_q^2u}^2/2$ into the left-hand side of
\eqref{eq:spatial-Bochner-identity} and note that
\[
  M^2+M^{3/2}d^{1/2}
  \leq C\left(M+M^{3/4}d^{1/4}\right)^2.
\]
This proves \eqref{eq:Bochner-spatial}.
\end{proof}

Recall that $A_q$ is the nonnegative self-adjoint operator associated with
the closed form
\[
  z\longmapsto\norm{\nabla_qz}_{L^2(\mu)}^2
  \quad\text{on }H^1(\mu).
\]
Thus $A_q$ is the Friedrichs realization of
$\nabla_q^*\nabla_q=-\Delta_q+\nabla U\cdot\nabla_q$.
The next lemma extends the Bochner estimate \eqref{eq:Bochner-spatial} to $D(A_q)$.

\begin{lemma}[Position operator core and domain regularity]
\label{lem:position-domain}
The operator
\[
  A_q^{\min}z=-\Delta_qz+\nabla U\cdot\nabla_qz,
  \qquad D(A_q^{\min})=C_c^\infty(\R^d),
\]
is essentially self-adjoint in $L^2(\mu)$.  Consequently,
$C_c^\infty(\R^d)$ is an operator core for $A_q$, and
\begin{equation}\label{eq:position-domain}
  D(A_q)\subset H^2(\mu),\qquad
  \norm{\nabla_q^2z}_{L^2(\mu)}^2
  \leq C\left(
  \norm{A_qz}_{L^2(\mu)}^2
  +R^2\norm{\nabla_qz}_{L^2(\mu)}^2
  \right).
\end{equation}
For $z\in D(A_q)$, the identity
\[
  A_qz=-\Delta_qz+\nabla U\cdot\nabla_qz
\]
holds in $L^2(\mu)$.
\end{lemma}

\begin{proof}
The minimal operator is symmetric and nonnegative by weighted
integration by parts.  For a densely defined nonnegative symmetric
operator $B$, the condition $\ker(B^*+1)=\{0\}$ is equivalent to
density of $\operatorname{Ran}(B+1)$.  Let $\overline B$ be its
closure.  The estimate
$\norm{(\overline B+1)z}\geq\norm z$, together with closedness of
$\overline B$, makes $\operatorname{Ran}(\overline B+1)$ closed.
It contains the dense range of $B+1$, so it equals the whole Hilbert
space; a closed symmetric operator with $-1$ in its resolvent is
self-adjoint.  We therefore verify
\[
  \ker\left((A_q^{\min})^*+1\right)=\{0\}.
\]
Let $z$ belong to this kernel.  Since the density of $\mu$ is
positive and locally bounded above and below, $z\in L^2_{\rm loc}$.
The distributional equation
\[
  -\Delta_qz+\nabla U\cdot\nabla_qz+z=0,
\]
and local elliptic regularity give $z\in H^2_{\rm loc}$.  Choose
$\chi\in C_c^\infty(\R^d)$, with $0\leq\chi\leq1$, equal to one on
the unit ball, and put $\chi_R(q)=\chi(q/R)$.  Testing the equation by
$\chi_R^2z$, justified by compactly supported $H^1$ approximation,
gives
\[
  \norm{\chi_Rz}_{L^2(\mu)}^2
  +\int_{\mathbb{R}^{d}}\chi_R^2\abs{\nabla_qz}^2\,\dd\mu
  =-2\int_{\mathbb{R}^{d}}\chi_Rz\,\nabla_q\chi_R\cdot\nabla_qz\,\dd\mu.
\]
Young's inequality therefore yields
\[
  \norm{\chi_Rz}_{L^2(\mu)}^2
  +\frac12\int_{\mathbb{R}^{d}}\chi_R^2\abs{\nabla_qz}^2\,\dd\mu
  \leq\frac{C}{R^2}\norm z_{L^2(\mu)}^2.
\]
Letting $R\to\infty$ proves $z=0$.  Hence
$A_q^{\min}$ is essentially self-adjoint, and its closure is the
Friedrichs operator $A_q$.

Now let $z\in D(A_q)$.  There are $z_n\in C_c^\infty(\R^d)$ such
that $z_n\to z$ and $A_qz_n\to A_qz$ in $L^2(\mu)$.  For
$w=z_n-z_k$,
\[
  \norm{\nabla_qw}_{L^2(\mu)}^2
  =\ip{A_qw}{w}_{L^2(\mu)}
  \leq\norm{A_qw}_{L^2(\mu)}\norm w_{L^2(\mu)}.
\]
Thus $z_n$ is Cauchy in $H^1(\mu)$.  Applying the spatial Bochner
estimate \eqref{eq:Bochner-spatial} to $z_n-z_k$ shows that its
spatial Hessians are also Cauchy in $L^2(\mu)$.  Local equivalence of
weighted and unweighted norms identifies the limits as the weak
derivatives of $z$, proving $z\in H^2(\mu)$; passage to the limit
also proves \eqref{eq:position-domain}.  Finally, Lemma
\ref{lem:weighted-multiplier}, applied to each component of
$\nabla_qz$, gives
\[
  \norm{\abs{\nabla U}\abs{\nabla_qz}}_{L^2(\mu)}^2
  \leq
  16\norm{\nabla_q^2z}_{L^2(\mu)}^2
  +4Md\norm{\nabla_qz}_{L^2(\mu)}^2.
\]
The claimed $L^2$ realization of $A_q$ follows by passing to the
limit in the distributional identity for $z_n$.
\end{proof}

We next apply the preceding coercivity and Bochner estimates to the
mixed Neumann problem.  Its solution supplies the elliptic potentials
used in the divergence construction.

\begin{lemma}[Mixed elliptic regularity]
\label{lem:mixed-Neumann}
Let $h\in L^2(\lambda_T\otimes\mu)$ have zero mean.  There is a
unique mean-zero weak solution $u\in H^2(\lambda_T\otimes\mu)$ of
\begin{equation}\label{eq:mixed-Neumann}
  \mathscr Au=h,\qquad
  \partial_tu(0,\cdot)=\partial_tu(T,\cdot)=0.
\end{equation}
It satisfies
\begin{equation}\label{eq:mixed-energy}
  \norm{\widetilde\nabla u}_{L^2}^2
  \leq
  \max\left\{\frac1m,\frac{T^2}{\pi^2}\right\}
  \norm h_{L^2}^2.
\end{equation}
Moreover,
\begin{equation}\label{eq:mixed-spectral-identity}
  \norm{\partial_{tt}u}_{L^2}^2
  +2\norm{\partial_t\nabla_qu}_{L^2}^2
  +\norm{A_qu}_{L^2}^2
  =\norm h_{L^2}^2,
\end{equation}
and
\begin{equation}\label{eq:mixed-H2-bound}
  \norm{D_{t,q}^2u}_{L^2}^2
  \leq
  C\left[
  1+R^2\max\left\{\frac1m,\frac{T^2}{\pi^2}\right\}
  \right]\norm h_{L^2}^2.
\end{equation}
\end{lemma}

\begin{proof}
On
\[
  V_0=\left\{v\in H^1(\lambda_T\otimes\mu):
  (v)_{\lambda_T\otimes\mu}=0\right\},
\]
consider
\[
  \mathcal B(u,v)
  =\ip{\partial_tu}{\partial_tv}
  +\ip{\nabla_qu}{\nabla_qv}.
\]
Lemma~\ref{lem:appendix-PI} makes $\mathcal B$ coercive on $V_0$.
The Lax--Milgram theorem \cite[Lemma~1.3]{ouhabaz-2005} gives a unique
weak solution.  Testing by $u$, applying
Cauchy--Schwarz inequality, and then using \eqref{eq:appendix-PI} proves
\eqref{eq:mixed-energy}.

For the $L^2$ regularity, let $B_t$ be the nonnegative Neumann
operator $-\partial_{tt}$ in $L^2(\lambda_T)$.
We write
$H_1\mathbin{\widehat\otimes}H_2$ for the Hilbert-space tensor product,
namely the completion of the algebraic tensor product in its canonical
Hilbert norm.  The map
$f\otimes g\mapsto[(t,q)\mapsto f(t)g(q)]$ extends to a canonical
unitary identification
\[
  L^2(\lambda_T\otimes\mu)
  \simeq L^2(\lambda_T)\mathbin{\widehat\otimes}L^2(\mu);
\]
see \cite[Definition~7.2 and Example~7.9]{schmudgen-2012}.
Under this identification, the operators $B_t\otimes\Id$ and
$\Id\otimes A_q$ are nonnegative self-adjoint operators, their
algebraic sum is essentially self-adjoint, and they strongly commute
\cite[Theorem~7.23 and Lemma~7.24]{schmudgen-2012}.  The operator
associated with the form $\mathcal B$ is the closure of that sum.
Therefore the
weak solution above is the joint-spectral solution
\[
  u=(B_t\otimes\Id+\Id\otimes A_q)^{-1}h
\]
on the orthogonal complement of the constants, by the joint functional
calculus \cite[Proposition~5.25]{schmudgen-2012}.  This inverse is
bounded there by Lemma~\ref{lem:appendix-PI}.
In the formulas below, tensoring with the identity operator on the
complementary factor is left implicit.

Let $\sigma_u$ be the scalar joint spectral measure of $u$ for
these two operators, with spectral variables $(b,a)$.  Since
$h=(a+b)u$ in this representation,
\begin{align*}
  \norm h_{L^2}^2
  =
  \int_{[0,\infty)^2}(a+b)^2\,\dd\sigma_u(b,a)
  =
  \norm{B_tu}_{L^2}^2
  +\norm{A_qu}_{L^2}^2
  +2\norm{\left(B_t^{1/2}\otimes A_q^{1/2}\right)u}_{L^2}^2.
\end{align*}
The form domains of $B_t$ and $A_q$ are $H^1(0,T)$ and
$H^1(\mu)$, respectively.  Tensor-product approximation consequently
identifies
\[
  B_tu=-\partial_{tt}u,\qquad
  \norm{\left(B_t^{1/2}\otimes A_q^{1/2}\right)u}_{L^2}
  =\norm{\partial_t\nabla_qu}_{L^2}.
\]
This proves \eqref{eq:mixed-spectral-identity}, including the Neumann
time traces encoded in $D(B_t)$.

Finally, $u\in D(\Id\otimes A_q)$.  Lemma
\ref{lem:position-domain}, applied in the direct integral over time,
and \eqref{eq:mixed-energy} give
\[
  \norm{\nabla_q^2u}_{L^2}^2
  \leq C\left(
  \norm{A_qu}_{L^2}^2
  +R^2\max\left\{\frac1m,\frac{T^2}{\pi^2}\right\}
  \norm h_{L^2}^2
  \right).
\]
Together with \eqref{eq:mixed-spectral-identity}, this proves
$u\in H^2$ and \eqref{eq:mixed-H2-bound}.
\end{proof}

\section{Spectral Construction of the Divergence Test}
\label{app:divergence}

We now prove Lemma~\ref{lem:divergence}.  The proof is included to make
clear both the role of compact embedding and the origin of the
$T,m,R$ dependence.  It follows the construction of
\cite[Lemma~2.6]{cao-lu-wang-2023}.

\paragraph{Position spectrum and harmonic decomposition.}

The closed quadratic form
$v\longmapsto\norm{\nabla_qv}_{L^2(\mu)}^2$
has associated operator $A_q=\nabla_q^*\nabla_q$.  Assumption
\ref{ass:compact} makes the resolvent of $A_q$, restricted to
mean-zero functions, compact.  Hence
\cite[Proposition~5.12]{schmudgen-2012} gives an orthonormal basis
$\{1\}\cup\{w_j:j\geq1\}$
of $L^2(\mu)$
such that
\begin{equation}\label{eq:position-spectrum}
  A_qw_j=\omega_j^2w_j,
  \qquad
  \omega_j\geq\sqrt m.
\end{equation}
The eigenfunctions lie in $H^2(\mu)$ by Lemma
\ref{lem:position-domain}.  Moreover,
\begin{equation}\label{eq:eigen-gradient-orthogonality}
  \ip{\nabla_qw_i}{\nabla_qw_j}_{L^2(\mu)}
  =\omega_j^2\delta_{ij}.
\end{equation}

Let
\[
  \cH
  :=
  \left\{v\in L^2(\lambda_T\otimes\mu):
  \mathscr Av=0\text{ in distributions}\right\}.
\]
This is a closed subspace of $L^2$.  Decompose the centered datum as
\[
  h=h_{\mathrm h}+h_\perp,
  \qquad
  h_{\mathrm h}\in\cH,\quad h_\perp\perp\cH.
\]

\paragraph{The component orthogonal to harmonic functions.}

Let $u$ be the mean-zero solution of
\[
  \mathscr Au=h_\perp,\qquad
  \partial_tu(0,\cdot)=\partial_tu(T,\cdot)=0,
\]
given by Lemma~\ref{lem:mixed-Neumann}.  For every smooth
$v\in\cH$, integration by parts gives
\[
  0=\ip{h_\perp}{v}
  =\ip{u}{\mathscr Av}
  +\frac1T\int_{\R^d}
  \left[u(T)\partial_tv(T)-u(0)\partial_tv(0)\right]\,\dd\mu.
\]
For the $j$-th nonconstant position mode, the functions
$e^{\omega_jt}w_j$ and $e^{-\omega_jt}w_j$ make the pair
$(\partial_tv(0),\partial_tv(T))$ arbitrary in the span of $w_j$.
Thus $u(0,\cdot)$ and $u(T,\cdot)$ are spatially constant, and hence
\[
  \nabla_qu(0,\cdot)=\nabla_qu(T,\cdot)=0.
\]
Together with the Neumann conditions in \eqref{eq:mixed-Neumann},
this gives zero time traces for both $\partial_tu$ and $\nabla_qu$.
Set
\begin{equation}\label{eq:perp-test}
  \phi_0^\perp=\partial_tu,\qquad
  \psi^\perp=u.
\end{equation}
Then
\[
  -\partial_t\phi_0^\perp+A_q\psi^\perp=h_\perp.
\]
Lemma~\ref{lem:mixed-Neumann} gives
\begin{equation}\label{eq:perp-first}
  \norm{\phi_0^\perp}+\norm{\nabla_q\psi^\perp}
  \leq
  C\left(\frac1{\sqrt m}+T\right)\norm{h_\perp}.
\end{equation}
Using \eqref{eq:perp-test} and the mixed elliptic estimate
\eqref{eq:mixed-H2-bound}, we obtain
\begin{equation}\label{eq:perp-second}
  \norm{\nabla_q\phi_0^\perp}
  +\norm{\partial_t\nabla_q\psi^\perp}
  +\norm{\nabla_q^2\psi^\perp}
  \leq
  C\left(1+\frac R{\sqrt m}+RT\right)\norm{h_\perp}.
\end{equation}
The right-hand sides of \eqref{eq:perp-first} and
\eqref{eq:perp-second} are bounded respectively by
$CB_0(T)\norm{h_\perp}$ and $CB_1(T)\norm{h_\perp}$.

\paragraph{Expansion of the harmonic component.}

Expand $h_{\mathrm h}$ in the position eigenbasis.  The equation
$\mathscr Ah_{\mathrm h}=0$ implies
\begin{equation}\label{eq:harmonic-expansion}
  h_{\mathrm h}(t,q)
  =
  c_0\left(t-\frac T2\right)
  +\sum_{j\geq1}
  \left(
  c_j^+e^{-\omega_jt}
  +c_j^-e^{-\omega_j(T-t)}
  \right)w_j(q).
\end{equation}
The additive constant in the zero position mode vanishes because
$h_{\mathrm h}$ is centered.

The two exponentials belonging to the same $j$ are not orthogonal,
but their Gram matrix can be computed explicitly.  With
$x=\omega_jT$, it is
\[
  \begin{pmatrix}
  d_x&e^{-x}\\ e^{-x}&d_x
  \end{pmatrix},
  \qquad
  d_x=\frac{1-e^{-2x}}{2x}.
\]
Its smaller eigenvalue is
\[
  d_x-e^{-x}
  =e^{-x}\left(\frac{\sinh x}{x}-1\right).
\]
For $0<x\leq1$, the series of $\sinh x$ and
$1-e^{-x}\leq x$ give
\[
  d_x-e^{-x}\geq c x^2
  \geq c\frac{(1-e^{-x})^3}{x}.
\]
For $x\geq2$, the bound $xe^{-x}\leq2e^{-2}$ gives
$d_x-e^{-x}\geq c/x$, while the same conclusion on
$[1,2]$ follows by continuity and strict positivity.  Thus, for all
$x>0$,
\begin{equation}\label{eq:Gram-scalar}
  d_x-e^{-x}
  \geq c\frac{(1-e^{-x})^3}{x}.
\end{equation}
Applying \eqref{eq:Gram-scalar} mode by mode gives
\begin{align}
  \norm{h_{\mathrm h}}_{L^2}^2
  \geq
  \frac{c_0^2T^2}{12}
  +
  c\sum_{j\geq1}
  \left((c_j^+)^2+(c_j^-)^2\right)
  \frac{(1-e^{-\omega_jT})^3}{\omega_jT}.
  \label{eq:harmonic-lower-bound}
\end{align}

\paragraph{A scalar mode construction.}

Fix $\omega>0$, set $\vartheta=e^{-\omega T}$, and define
\begin{equation}\label{eq:Gtheta}
  G_\vartheta(s)
  :=
  \frac{6s(s-\vartheta)(1-s)}{(1-\vartheta)^2},
  \qquad \vartheta\leq s\leq1.
\end{equation}
It satisfies
\[
  G_\vartheta(\vartheta)=G_\vartheta(1)=0,
  \qquad
  \int_\vartheta^1\frac{G_\vartheta(s)}s\,\dd s=1-\vartheta.
\]
Set
\begin{equation}
  b_\omega(t):=\frac1{\omega^2}
  G_\vartheta(e^{-\omega t}),
  \qquad
  a_\omega(t)
  :=
  \int_0^t
  \left(\omega^2b_\omega(r)-e^{-\omega r}\right)\,\dd r.
\end{equation}
The integral identity for $G_\vartheta$ shows that
\[
  a_\omega(0)=a_\omega(T)=b_\omega(0)=b_\omega(T)=0,
\]
and the defining equation is
\begin{equation}\label{eq:scalar-divergence}
  -a_\omega'(t)+\omega^2b_\omega(t)=e^{-\omega t}.
\end{equation}

To estimate these modes, use the change of variables
$s=e^{-\omega t}$.  The elementary bounds
\[
  0\leq G_\vartheta(s)\leq\frac32s,
  \qquad
  \abs{G_\vartheta'(s)}
  \leq\frac{C}{1-\vartheta}
\]
and $\dd t=-\dd s/(\omega s)$ give
\begin{align}
  \omega^4\norm{b_\omega}_{L^2(\lambda_T)}^2
  &=
  \frac1{\omega T}\int_\vartheta^1
  \frac{G_\vartheta(s)^2}{s}\,\dd s
  \leq C\frac{1-\vartheta^2}{\omega T},
  \label{eq:b-mode-0}\\
  \omega^2\norm{b_\omega'}_{L^2(\lambda_T)}^2
  &=
  \frac1{\omega T}\int_\vartheta^1
  sG_\vartheta'(s)^2\,\dd s
  \leq\frac{C}{\omega T(1-\vartheta)}.
  \label{eq:b-mode-1}
\end{align}
Writing
\[
  a_\omega(t)=\frac1\omega
  R_\vartheta(e^{-\omega t}),\qquad
  R_\vartheta(s)
  =\int_s^1\left(\frac{G_\vartheta(r)}r-1\right)\,\dd r
  =
  \frac{(s-\vartheta)(1-s)(1+\vartheta-2s)}
  {(1-\vartheta)^2},
\]
we obtain, after setting $s=\vartheta+(1-\vartheta)y$,
\[
  \int_\vartheta^1\frac{R_\vartheta(s)^2}{s}\,\dd s
  \leq C(1-\vartheta)^3.
\]
Consequently,
\begin{equation}\label{eq:a-mode}
  \omega^2\norm{a_\omega}_{L^2(\lambda_T)}^2
  \leq
  C\frac{(1-\vartheta)^3}{\omega T}.
\end{equation}
Since $a_\omega'=G_\vartheta(s)-s$ and
$\abs{G_\vartheta(s)-s}\leq s$, the same change of variables gives
\begin{equation}\label{eq:a-prime-mode}
  \norm{a_\omega'}_{L^2(\lambda_T)}^2
  \leq
  C\frac{1-\vartheta^2}{\omega T}.
\end{equation}
The estimates \eqref{eq:b-mode-0}, \eqref{eq:b-mode-1}, and
\eqref{eq:a-mode} will be used to prove
\eqref{eq:harmonic-first-final}--\eqref{eq:harmonic-second-final}.
The estimate \eqref{eq:a-prime-mode} controls the time derivative of
$\phi_0^{\mathrm h}$ in \eqref{eq:harmonic-time-phi}.

For the reversed exponential, define
\[
  \widetilde a_\omega(t)=-a_\omega(T-t),
  \qquad
  \widetilde b_\omega(t)=b_\omega(T-t).
\]
Then,
\[
  -\widetilde a_\omega'(t)
  +\omega^2\widetilde b_\omega(t)
  =e^{-\omega(T-t)},
\]
and all endpoint traces still vanish.

\paragraph{Assembly and estimates.}

The zero position mode in \eqref{eq:harmonic-expansion} is handled by
\[
  \phi_0^{(0)}(t)
  =-\frac{c_0}{2}(t^2-Tt),
  \qquad
  \psi^{(0)}=0.
\]
Define
\begin{align}
  \phi_0^{\mathrm h}
  :=\;&
  \phi_0^{(0)}
  +\sum_{j\geq1}
  \left[
  c_j^+a_{\omega_j}(t)
  -c_j^-a_{\omega_j}(T-t)
  \right]w_j,
  \label{eq:harmonic-phi}\\
  \psi^{\mathrm h}
  :=\;&
  \sum_{j\geq1}
  \left[
  c_j^+b_{\omega_j}(t)
  +c_j^-b_{\omega_j}(T-t)
  \right]w_j.
  \label{eq:harmonic-psi}
\end{align}
Equations \eqref{eq:scalar-divergence} and
\eqref{eq:position-spectrum} imply that
\[
  -\partial_t\phi_0^{\mathrm h}
  +A_q\psi^{\mathrm h}=h_{\mathrm h}.
\]
Every term in $\phi_0^{\mathrm h}$ and
$\nabla_q\psi^{\mathrm h}$ has zero trace at $t=0,T$.
We first work with finite spectral sums; passage to the infinite sum is
justified after deriving all required Sobolev estimates.

We next sum the estimates \eqref{eq:b-mode-0},
\eqref{eq:b-mode-1}, and \eqref{eq:a-mode} over $j$ in the
expansions \eqref{eq:harmonic-phi}--\eqref{eq:harmonic-psi}.
Orthogonality and \eqref{eq:a-mode} give
\[
  \norm{\phi_0^{\mathrm h}}^2
  \leq
  C\left(T^2+\frac1m\right)
  \norm{h_{\mathrm h}}^2.
\]
By \eqref{eq:b-mode-0}, \eqref{eq:position-spectrum}, and
\eqref{eq:harmonic-lower-bound},
\[
  \norm{\nabla_q\psi^{\mathrm h}}^2
  \leq
  \frac{C}{m(1-e^{-\sqrt mT})^2}
  \norm{h_{\mathrm h}}^2.
\]
Together, these estimates imply
\begin{equation}\label{eq:harmonic-first-final}
  \norm{\phi_0^{\mathrm h}}
  +\norm{\nabla_q\psi^{\mathrm h}}
  \leq C B_0(T)\norm{h_{\mathrm h}}.
\end{equation}

Similarly, \eqref{eq:a-mode} gives
\begin{align}\label{RHS:1}
  \norm{\nabla_q\phi_0^{\mathrm h}}
  \leq C\norm{h_{\mathrm h}},
\end{align}
while \eqref{eq:b-mode-1} gives
\begin{align}\label{RHS:2}
  \norm{\partial_t\nabla_q\psi^{\mathrm h}}
  \leq
  \frac{C}{(1-e^{-\sqrt mT})^2}
  \norm{h_{\mathrm h}}.
\end{align}
Each finite spectral sum defining $\psi^{\mathrm h}(t,\cdot)$ lies
in $D(A_q)$, so \eqref{eq:position-domain} applies at each time.  Since
\[
  A_q\psi^{\mathrm h}
  =
  \sum_{j\geq 1}\omega_j^2
  \left[c_j^+b_{\omega_j}(t)
  +c_j^-b_{\omega_j}(T-t)\right]w_j,
\]
orthogonality, \eqref{eq:b-mode-0}, and
$\omega_j\geq\sqrt m$ yield
\begin{align}\label{RHS:3}
  \norm{\nabla_q^2\psi^{\mathrm h}}
  \leq
  C\left[
  \frac1{1-e^{-\sqrt mT}}
  +\frac{R}{\sqrt m(1-e^{-\sqrt mT})}
  \right]\norm{h_{\mathrm h}}.
\end{align}
The right-hand sides of \eqref{RHS:1}, \eqref{RHS:2} and \eqref{RHS:3} are bounded by $CB_1(T)\norm{h_{\mathrm h}}$;
hence,
\begin{equation}\label{eq:harmonic-second-final}
  \norm{\nabla_q\phi_0^{\mathrm h}}
  +\norm{\partial_t\nabla_q\psi^{\mathrm h}}
  +\norm{\nabla_q^2\psi^{\mathrm h}}
  \leq C B_1(T)\norm{h_{\mathrm h}}.
\end{equation}

It remains to justify the Sobolev regularity and the infinite
spectral sums.  
Let
\[
  \eta_T:=1-e^{-\sqrt mT}>0.
\]
Comparing \eqref{eq:a-prime-mode} with
\eqref{eq:harmonic-lower-bound}, and using
$1-e^{-\omega_jT}\geq\eta_T$, gives
\begin{equation}\label{eq:harmonic-time-phi}
  \norm{\partial_t\phi_0^{\mathrm h}}_{L^2}^2
  \leq C\left(1+\eta_T^{-2}\right)\norm{h_{\mathrm h}}_{L^2}^2.
\end{equation}
Since $\psi^{\mathrm h}$ has zero position mean, Assumption
\ref{ass:PI} and \eqref{eq:harmonic-first-final} give
\[
  \norm{\psi^{\mathrm h}}_{L^2}
  \leq\frac1{\sqrt m}\norm{\nabla_q\psi^{\mathrm h}}_{L^2}
  \leq\frac{CB_0(T)}{\sqrt m}\norm{h_{\mathrm h}}_{L^2}.
\]

Let $h_{\mathrm h}^{>N}$ denote the part of
\eqref{eq:harmonic-expansion} with $j>N$.  Orthogonality of the
position modes gives
$\norm{h_{\mathrm h}^{>N}}_{L^2}\to0$.  Apply
\eqref{eq:harmonic-first-final}, \eqref{eq:harmonic-second-final},
\eqref{eq:harmonic-time-phi}, and the position Poincar\'e inequality
to this tail.  The finite sums for $\phi_0^{\mathrm h}$,
$\psi^{\mathrm h}$, and $\nabla_q\psi^{\mathrm h}$ are Cauchy in
$H^1(\lambda_T\otimes\mu)$, $L^2(\lambda_T;H^2(\mu))$, and
$H^1(\lambda_T\otimes\mu)^d$, respectively.  Continuity of the
trace maps preserves the zero endpoint traces, so their limits satisfy
\[
  \phi_0^{\mathrm h}\in H_0^1(\lambda_T\otimes\mu),\qquad
  \psi^{\mathrm h}\in L^2(\lambda_T;H^2(\mu)),\qquad
  \nabla_q\psi^{\mathrm h}\in H_0^1(\lambda_T\otimes\mu)^d.
\]
The divergence equation passes to the limit in distributions.

Set
\[
  \phi_0:=\phi_0^\perp+\phi_0^{\mathrm h},
  \qquad
  \psi:=\psi^\perp+\psi^{\mathrm h}.
\]
The two components solve the divergence equation, have the required
time traces, and satisfy the desired estimates by
\eqref{eq:perp-first}--\eqref{eq:perp-second} and
\eqref{eq:harmonic-first-final}--\eqref{eq:harmonic-second-final}.
This completes the proof of Lemma~\ref{lem:divergence}.

\section{Weighted Kinetic Density, Traces, and Green's Formula}
\label{app:kinetic-trace}

This appendix proves the weighted whole-space kinetic trace result used
in Proposition~\ref{prop:kinetic-endpoint}.  It is the analogue, in the
Gibbs space and without a spatial boundary, of the density and trace
arguments in
\cite[Proposition~2.2, Proposition~6.1, and
Lemma~6.12]{albritton-armstrong-mourrat-novack-2024}.  The proof is
included because the force $\nabla U$ need not be bounded.

In this appendix, we set
\[
  H:=L^2(\pi),
  \qquad
  V_p:=L^2\left(\mu;H_\kappa^1\right),
  \qquad
  V_p^*:=L^2\left(\mu;H_\kappa^{-1}\right),
\]
so that $V_p\subset H\subset V_p^*$.  We use ordinary Lebesgue
measure in time; replacing it by $\lambda_T$ only rescales the norms.
Recall
\[
  \cK_0=\partial_t-\cL_{\mathrm{ham}},
  \qquad
  \cL_{\mathrm{ham}}
  =p\cdot\nabla_q-\nabla U(q)\cdot\nabla_p.
\]

The trace argument requires approximation by smooth functions in the
kinetic graph norm.  The following lemma handles both the unbounded
Gaussian velocity multiplier and the whole-space cutoffs.

\begin{lemma}[Gaussian multiplier and local graph density]
\label{lem:kinetic-density}
For every $g\in L^2(\kappa)$ and $1\leq i\leq d$,
\begin{equation}\label{eq:gaussian-p-multiplier}
  \norm{p_i g}_{H_\kappa^{-1}}
  \leq C\norm g_{L^2(\kappa)}.
\end{equation}
Let $I\subset\R$ be an open interval and let
\[
  u\in L^2(I;V_p),
  \qquad
  F:=\cK_0u\in L^2(I;V_p^*).
\]
For every $J\Subset I$, there are
$u_n\in C_c^\infty(I\times\R^{2d})$ such that, as $n\to\infty$,
\begin{equation}\label{eq:local-graph-density}
  u_n\longrightarrow u\quad\text{in }L^2(J;V_p),
  \qquad
  \cK_0u_n\longrightarrow F\quad\text{in }L^2(J;V_p^*).
\end{equation}
\end{lemma}

\begin{proof}
For $\varphi\in H_\kappa^1$, Gaussian integration by parts and
Young's inequality give
\[
  \int_{\mathbb{R}^{d}} p_i^2\varphi^2\,\dd\kappa
  =\int_{\mathbb{R}^{d}}\left(\varphi^2+2p_i\varphi\,\partial_{p_i}\varphi\right)\,\dd\kappa
  \leq
  \norm\varphi_{L^2(\kappa)}^2
  +\frac12\norm{p_i\varphi}_{L^2(\kappa)}^2
  +2\norm{\partial_{p_i}\varphi}_{L^2(\kappa)}^2.
\]
Thus multiplication by $p_i$ maps $H_\kappa^1$ continuously into
$L^2(\kappa)$.  Duality proves
\eqref{eq:gaussian-p-multiplier}.

Choose $\eta\in C_c^\infty(I)$ equal to one on a neighborhood of
$J$.  Replacing $u$ by $\eta u$ introduces only the graph term
$\eta' u$, since
\[
  \cK_0(\eta u)=\eta F+\eta'u.
\]
It therefore suffices to approximate a function compactly supported in
time.

Let $\chi\in C_c^\infty(\R^d)$, with $0\leq\chi\leq1$, equal one
on the unit ball, and put $\chi_R(q)=\chi(q/R)$.  Then
\begin{equation}\label{eq:q-cutoff-commutator}
  \cK_0(\chi_Ru)
  =\chi_RF-(p\cdot\nabla_q\chi_R)u.
\end{equation}
The first term converges to $F$ in $L^2(I;V_p^*)$, while
\eqref{eq:gaussian-p-multiplier} gives
\[
  \norm{(p\cdot\nabla_q\chi_R)u}_{L^2(I;V_p^*)}
  \leq\frac CR\norm u_{L^2(I;H)}\longrightarrow0
  \qquad\text{as }R\to\infty.
\]
Also $\chi_Ru\to u$ in $L^2(I;V_p)$.

Next fix $R$, choose an analogous cutoff
$\theta_S(p)=\chi(p/S)$, and set $u_{R,S}=\theta_S\chi_Ru$.
Multiplication by $\theta_S$ is uniformly bounded on
$H_\kappa^1$, hence also on $H_\kappa^{-1}$, and converges strongly
to the identity on both spaces.  For the negative space, the latter
claim follows first on the dense subspace $L^2(\kappa)$ and then on
all of $H_\kappa^{-1}$ by uniform boundedness.  Moreover,
\begin{equation}\label{eq:p-cutoff-commutator}
  \cK_0u_{R,S}
  =\theta_S\cK_0(\chi_Ru)
  +(\nabla U\cdot\nabla_p\theta_S)\chi_Ru.
\end{equation}
Because $q$ is now restricted to a fixed compact set,
$\nabla U$ is bounded there.  The last term in
\eqref{eq:p-cutoff-commutator} is consequently bounded in
$L^2(I;H)$ by $C_R S^{-1}\norm u_{L^2(I;H)}$ and tends to zero.
Taking first $S\to\infty$ and then a diagonal sequence
$R\to\infty$ produces compactly supported functions converging to
$u$ in the kinetic graph norm.

It remains to smooth such a compactly supported function.  On a fixed
compact set the Gibbs density is bounded above and below by positive
constants, so weighted and unweighted local Sobolev norms are
equivalent.  Choose $\rho\in C_c^\infty(\R^{1+2d})$ with
$\rho\geq0$ and $\int\rho=1$, and set
\[
  \rho_\varepsilon(z)
  :=\varepsilon^{-(1+2d)}\rho(z/\varepsilon),
  \qquad z=(t,x),\quad x=(q,p).
\]
Write
$b(q,p)=(p,-\nabla U(q))$, so that
$\cL_{\mathrm{ham}}=b\cdot\nabla_{q,p}$.  On that compact set,
$b$ is Lipschitz and $\operatorname{div}_{q,p}b=0$.  Hence
\begin{align*}
  \cK_0(\rho_\varepsilon*u)
  & =\rho_\varepsilon*F+\mathcal C_\varepsilon u,\\
  \mathcal C_\varepsilon u(t,x)
  & =\int_{\mathbb{R}^{1+2d}}
  \left(b(x-y_x)-b(x)\right)\cdot
  \nabla_{y_x}\rho_\varepsilon(\tau,y_x)
  u(t-\tau,x-y_x)\,\dd\tau\,\dd y_x.
\end{align*}
The Lipschitz bound gives
$\norm{\mathcal C_\varepsilon u}_{L^2}\leq C\norm u_{L^2}$,
uniformly in $\varepsilon$.  For smooth $u$, direct differentiation
shows
$\mathcal C_\varepsilon u\to0$ in $L^2$.  Approximation of an
arbitrary $L^2$ function by smooth ones and the uniform bound give
the same conclusion for the present $u$.  The approximate-identity
property, applied first in $H^1_p$ and then by duality in
$H^{-1}_p$, gives on every $J\Subset I$
\[
  \rho_\varepsilon*u\to u\quad\text{in }L^2(J;V_p),
  \qquad
  \rho_\varepsilon*F\to F\quad\text{in }L^2(J;V_p^*).
\]
Together with $\mathcal C_\varepsilon u\to0$, these convergences
prove \eqref{eq:local-graph-density}.
\end{proof}

The graph-density result permits the classical smooth
integration-by-parts identity to be passed to weak kinetic solutions.
This yields both strong interior time traces and Green's formula.

\begin{proposition}[Kinetic traces and Green's formula]
\label{prop:kinetic-green}
Let
\[
  u,v\in L^2_{\mathrm{loc}}(I;V_p),
  \qquad
  \cK_0u,\cK_0v\in L^2_{\mathrm{loc}}(I;V_p^*).
\]
They admit representatives in
$C(I;L^2(\pi))$, and, for every $s,t\in I$ with $s<t$,
\begin{align}
  \ip{u_t}{v_t}_{L^2(\pi)}-\ip{u_s}{v_s}_{L^2(\pi)}
  =\int_s^t\left[
  \ip{\cK_0u_r}{v_r}_{V_p^*,V_p}
 +\ip{\cK_0v_r}{u_r}_{V_p^*,V_p}
  \right]\,\dd r.
  \label{eq:kinetic-green}
\end{align}
Suppose $0$ is the left endpoint of $I$. For either $w=u$ or $w=v$,
if $w_t\rightharpoonup w_0$ weakly in $L^2(\pi)$ as $t\downarrow0$ and
\begin{equation}\label{eq:endpoint-limsup}
  \limsup_{t\downarrow0}\norm{w_t}_{L^2(\pi)}
  \leq\norm{w_0}_{L^2(\pi)},
\end{equation}
then $w_t\to w_0$ strongly as $t\downarrow0$. If both $u,v$ satisfy
these trace conditions and, for some $t_*\in I$,
$u,v\in L^2((0,t_*);V_p)$ and
$\cK_0u,\cK_0v\in L^2((0,t_*);V_p^*)$, then
\eqref{eq:kinetic-green} extends to $s=0$. The corresponding statement
holds at a finite right endpoint.
\end{proposition}

\begin{proof}
For compactly supported smooth functions, weighted integration by parts
gives
\[
  \ip{\cL_{\mathrm{ham}}u}{v}_{L^2(\pi)}
  +\ip{u}{\cL_{\mathrm{ham}}v}_{L^2(\pi)}=0.
\]
Indeed, the Hamiltonian vector field is divergence-free and annihilates
$U(q)+\abs p^2/2$.  Consequently,
\begin{equation}\label{eq:smooth-kinetic-green}
  \frac{\dd}{\dd t}\ip{u_t}{v_t}_{L^2(\pi)}
  =\ip{\cK_0u_t}{v_t}_{V_p^*,V_p}
  +\ip{\cK_0v_t}{u_t}_{V_p^*,V_p}.
\end{equation}

We first obtain the time trace.  If $w$ is smooth on a compact
interval $J\Subset I$, choose a time $r\in J$ at which
  $\norm{w_r}_H^2\leq\abs{J}^{-1}\norm w_{L^2(J;H)}^2$, and integrate
\eqref{eq:smooth-kinetic-green} with $u=v=w$.  This yields
\begin{equation}\label{eq:kinetic-trace-bound}
  \sup_{t\in J}\norm{w_t}_H^2
  \leq
  \frac1{\abs{J}}\norm w_{L^2(J;H)}^2
  +2\norm{\cK_0w}_{L^2(J;V_p^*)}
  \norm w_{L^2(J;V_p)}.
\end{equation}
Apply Lemma~\ref{lem:kinetic-density} on a slightly larger compact
subinterval.  Estimate \eqref{eq:kinetic-trace-bound}, applied to the
difference of two approximants, shows that the approximants are Cauchy
in $C(J;H)$.  Their limit is the desired representative of $u$.
Representatives obtained on overlapping intervals agree, giving
$u\in C(I;H)$, and similarly for $v$.

Integrating \eqref{eq:smooth-kinetic-green} and passing to the limit in
the graph norm proves \eqref{eq:kinetic-green}. For $w=u$ or $w=v$,
weak lower semicontinuity and \eqref{eq:endpoint-limsup} give
$\lim_{t\downarrow0}\norm{w_t}_H=\norm{w_0}_H$, and hence $w_t\to w_0$ in $H$.
When both traces are strong, $\ip{u_s}{v_s}_H\to\ip{u_0}{v_0}_H$.
For $0<s<t_*$, the Cauchy--Schwarz inequality bounds the two terms near zero by
\begin{align*}
  &\int_0^s\!\left(\abs{\ip{\cK_0u_r}{v_r}_{V_p^*,V_p}}
      +\abs{\ip{\cK_0v_r}{u_r}_{V_p^*,V_p}}\right)\,\dd r\\
  &\qquad\leq
  \norm{\cK_0u}_{L^2((0,s);V_p^*)}\norm v_{L^2((0,s);V_p)}
  +\norm{\cK_0v}_{L^2((0,s);V_p^*)}\norm u_{L^2((0,s);V_p)}.
\end{align*}
The right-hand side tends to zero by the graph integrability on $(0,t_*)$,
so letting $s\downarrow0$ in \eqref{eq:kinetic-green} gives the formula
at $s=0$. The limit at a finite right endpoint follows by integrating
up to that endpoint.
\end{proof}

Taking the same function in both slots of Green's formula gives the
kinetic energy chain rule used in the construction of the endpoint
semigroup.

\begin{corollary}[Kinetic chain rule]
\label{cor:kinetic-chain-rule}
Suppose $u$ satisfies the local graph conditions of Proposition
\ref{prop:kinetic-green} and
\[
  \cK_0u=-\gamma\nabla_p^*\nabla_pu.
\]
Then, for $s,t\in I$ with $s<t$,
\begin{equation}\label{eq:appendix-kinetic-chain-rule}
  \frac12\norm{u_t}_{L^2(\pi)}^2
  +\gamma\int_s^t\norm{\nabla_pu_r}_{L^2(\pi)}^2\,\dd r
  =\frac12\norm{u_s}_{L^2(\pi)}^2.
\end{equation}
If $u$ satisfies the weak trace and norm condition \eqref{eq:endpoint-limsup}
of Proposition~\ref{prop:kinetic-green} at an endpoint, the identity extends
there.
\end{corollary}

\begin{proof}
By taking $v=u$ in \eqref{eq:kinetic-green} and using
\[
  \ip{-\gamma\nabla_p^*\nabla_pu}{u}_{V_p^*,V_p}
  =-\gamma\norm{\nabla_pu}_{L^2(\pi)}^2,
\]
we obtain \eqref{eq:appendix-kinetic-chain-rule} for interior $s,t$.
At the left endpoint, the strong trace of $u$ gives
$\norm{u_s}_{L^2(\pi)}^2\to\norm{u_0}_{L^2(\pi)}^2$, and the monotone
convergence theorem gives
$\int_s^t\norm{\nabla_pu_r}_{L^2(\pi)}^2\,\dd r\uparrow
\int_0^t\norm{\nabla_pu_r}_{L^2(\pi)}^2\,\dd r$ as $s\downarrow0$.
The corresponding limit at a finite right endpoint uses the strong trace
there and monotone convergence as $t$ increases to that endpoint.
\end{proof}

\bibliographystyle{alpha}
\bibliography{bibtex}

\end{document}